\documentclass[aos,noinfoline]{imsart}
\pdfoutput=1
\usepackage[usenames,dvipsnames]{xcolor}
\definecolor{RefColor}{rgb}{0,0,.65}
\RequirePackage[colorlinks,linkcolor=RefColor,citecolor=RefColor,urlcolor=RefColor]{hyperref}
\hypersetup{pdfinfo={Subject={ }}}
\RequirePackage[OT1]{fontenc}

\usepackage{algorithm}
\usepackage{algorithmicx}
\usepackage{algpseudocode}

\usepackage{csquotes}

\usepackage{amsmath,amssymb,amscd,amsfonts,amsthm,mathtools}
\usepackage{natbib}
\usepackage[capitalize]{cleveref}

\crefname{algorithm}{Algorithm}{Algorithms}

\definecolor{darkgreenClj}{rgb}{0.25,.5,0.25}
\definecolor{blueClj}{rgb}{0,0.33,0.66}
\definecolor{redClj}{rgb}{0.66,0.0,0.0}
\definecolor{purpleClj}{rgb}{0.33,0,0.66}
\definecolor{cyanClj}{rgb}{0.0,0.5,0.5}
\definecolor{orangeClj}{rgb}{0.75,0.35,0.0}
\definecolor{grayClj}{rgb}{0.6,0.6,0.6}
\definecolor{darkgrayClj}{rgb}{0.3,0.3,0.3}

\usepackage{thmtools}
\usepackage{booktabs}
\usepackage{tikz}
\usetikzlibrary{arrows, decorations.pathmorphing, matrix}
\usepackage{tikz-network} 
\usepackage{pgfplotstable}
\pgfplotsset{compat=1.12}
\usepackage{calc}

\usepackage{dsfont}
\usepackage{enumitem}
\usepackage{graphicx}

\usepackage{filecontents}

\startlocaldefs

\theoremstyle{plain}
\newtheorem*{theorem*}{Theorem}
\newtheorem{theorem}{Theorem}
\newtheorem{lemma}[theorem]{Lemma}
\newtheorem{proposition}[theorem]{Proposition}
\newtheorem{corollary}[theorem]{Corollary}

\theoremstyle{definition}
\newtheorem{definition}[theorem]{Definition}

\newtheorem*{remark*}{Remark}
\newtheorem{example}{Example}

\declaretheoremstyle[
  postheadhook = {\hspace*{\parindent}},
  mdframed={
    backgroundcolor=gray!10!white, 
    hidealllines=true, 
    innertopmargin=2pt, 
    innerbottommargin=4pt, 
    skipabove=8pt,
    skipbelow=10pt,
    nobreak=true}
]{grayboxed} 
\declaretheoremstyle[
  mdframed={
    backgroundcolor=gray!10!white, 
    hidealllines=true, 
    innertopmargin=2pt, 
    innerbottommargin=4pt, 
    skipabove=8pt,
    skipbelow=10pt,
    nobreak=true},headpunct=\empty
]{grayboxed*}

\crefname{boxalgorithm}{Algorithm}{Algorithms}

\crefname{section}{Section}{Sections}
\crefname{appsec}{Appendix}{Appendices}

\crefname{theorem}{Theorem}{Theorems}
\crefname{proposition}{Proposition}{Propositions}
\crefname{lemma}{Lemma}{Lemmas}
\crefname{example}{Example}{Examples}
\crefname{algorithm}{Algorithm}{Algorithms}

\definecolor{darkgreenClj}{rgb}{0.25,.5,0.25}
\definecolor{blueClj}{rgb}{0,0.33,0.66}
\definecolor{redClj}{rgb}{0.66,0.0,0.0}
\definecolor{purpleClj}{rgb}{0.33,0,0.66}
\definecolor{cyanClj}{rgb}{0.0,0.5,0.5}
\definecolor{orangeClj}{rgb}{0.75,0.35,0.0}
\definecolor{grayClj}{rgb}{0.6,0.6,0.6}
\definecolor{darkgrayClj}{rgb}{0.3,0.3,0.3}

\makeatletter
\newcommand*\rel@kern[1]{\kern#1\dimexpr\macc@kerna}
\newcommand*\widebar[1]{%
  \begingroup
  \def\mathaccent##1##2{%
    \rel@kern{0.8}%
    \overline{\rel@kern{-0.8}\macc@nucleus\rel@kern{0.2}}%
    \rel@kern{-0.2}%
  }%
  \macc@depth\@ne
  \let\math@bgroup\@empty \let\math@egroup\macc@set@skewchar
  \mathsurround\z@ \frozen@everymath{\mathgroup\macc@group\relax}%
  \macc@set@skewchar\relax
  \let\mathaccentV\macc@nested@a
  \macc@nested@a\relax111{#1}%
  \endgroup
}
\makeatother

\makeatletter
  \def\journal@name{ }
  \def\journal@url{}
\makeatother

\newcommand{\empirical}{\mathbb{M}}

\newcommand{\urnlaw}{\mathbb{U}^{\Symn}}

\newcommand{\perm}{\pi}
\newcommand{\permtwo}{\boldsymbol{\perm}_2}
\newcommand{\permd}{\boldsymbol{\perm}_d}
\newcommand{\permdim}{\rho}

\newcommand{\borel}{\calB}
\newcommand{\noise}{\eta}
\newcommand{\noisen}{\boldsymbol{\noise}_{n}}

\newcommand{\normhaar}{{\lambda}_{\grp}}
\newcommand{\orbitlaw}{\mathbb{U}^{\grp}}

\newcommand{\trans}{\Phi}

\newcommand{\grp}{\mathcal{G}}
\newcommand{\model}{\calP}

\newcommand{\activation}{\sigma}

\newcommand{\xn}{\mathbf{x}_n}
\newcommand{\Xn}{\mathbf{X}_n}
\newcommand{\Xnfeati}{\Xn^{(1)}}
\newcommand{\Xnfeatj}{\Xn^{(2)}}
\newcommand{\Ynfeati}{\Yn^{(1)}}
\newcommand{\Ynfeatj}{\Yn^{(2)}}
\newcommand{\Xinf}{\mathbf{X}_{\infty}}

\newcommand{\Xnmi}{\mathbf{X}_{n\setminus i}}

\newcommand{\Xnpm}{\mathbf{X}_{n+m}}
\newcommand{\Xnpmtail}{\Xnpm\setminus\Xn}
\newcommand{\Yn}{\mathbf{Y}_n}
\newcommand{\Ynmi}{\mathbf{Y}_{n\setminus i}}
\newcommand{\Ynmj}{Y_{[n]\setminus j}}
\newcommand{\Zn}{\mathbf{Z}_{n}}

\newcommand{\canon}{\text{\sc{Canon}}}
\newcommand{\Xnd}{\mathbf{X}_{\bfn_d}}
\newcommand{\Csymm}{\mathbf{C}}
\newcommand{\Xsymm}{\widebar{{X}}}
\newcommand{\Ysymm}{\widebar{{Y}}}

\newcommand{\Xndsymm}{\mathbf{\Xsymm}_{n^d}}
\newcommand{\cboard}{\mathbf{C}}

\newcommand{\augcboardsymm}{\widebar{\cboard}}

\newcommand{\Ynd}{\mathbf{Y}_{\bfn_d}}

\newcommand{\Cntwo}{\cboard_{\Xntwo}}
\newcommand{\Xntwo}{\mathbf{X}_{\bfn_2}}
\newcommand{\Xntwoi}{\mathbf{X}_{i,:}}
\newcommand{\Yntwoi}{\mathbf{Y}_{i,:}}
\newcommand{\Cntwoi}{\mathbf{C}_{i,:}}
\newcommand{\Xntwoj}{\mathbf{X}_{:,j}}
\newcommand{\Yntwoj}{\mathbf{Y}_{:,j}}
\newcommand{\Cntwoj}{\mathbf{C}_{:,j}}

\newcommand{\augc}{\cboard_{\Xntwo}^{(i,j)}}
\newcommand{\augcsymm}{\cboard_{\Xntwosymm}^{\{i,j\}}}

\newcommand{\bcast}{\mathbf{B}}
\newcommand{\stak}{\mathbf{S}}

\newcommand{\Cntwosymm}{\mathbf{C}_{\Xntwosymm}}

\newcommand{\Xntwosymm}{\mathbf{\Xsymm}_{\bfn}}
\newcommand{\Xntwosymmi}{\mathbf{\Xsymm}_{i,:}}
\newcommand{\Xntwosymmj}{\mathbf{\Xsymm}_{j,:}}
\newcommand{\Yntwosymm}{\mathbf{\Ysymm}_{\bfn}}

\newcommand{\Yntwo}{\mathbf{Y}_{\bfn_2}}

\newcommand{\Cnfeat}{\mathbf{\Csymm}_{\bcast(\Xn,\Xntwosymm)}}
\newcommand{\Cnfeataug}{\cboard_{\bcast(\Xn,\Xntwosymm)}^{\{i,j\}}}

\newcommand{\subarr}[2]{{#1}_{#2}^{\downarrow}}
\newcommand{\subarrcoll}[3]{{#1}_{#3}^{\downarrow (#2)}}

\newcommand{\superarr}[3]{{#1}_{#3}^{\uparrow (#2)}}

\newcommand{\cspace}{\calC}
\newcommand{\augcspace}{\cspace^{\text{\rm{aug}}}}
\newcommand{\augcspacesymm}{\widebar{\cspace}^{\text{\rm{aug}}}}

\newcommand{\indicator}{\mathds{1}}

\newcommand{\calF}{\mathcal{F}}
\newcommand{\calC}{\mathcal{C}}
\newcommand{\calM}{\mathcal{M}}

\newcommand{\calX}{\mathcal{X}}
\newcommand{\calY}{\mathcal{Y}}

\newcommand{\calB}{\mathcal{B}}
\newcommand{\calS}{\mathcal{S}}
\newcommand{\calA}{\mathcal{A}}

\newcommand{\calZ}{\mathcal{Z}}

\newcommand{\calP}{\mathcal{P}}

\newcommand{\bbR}{\mathbb{R}}
\newcommand{\bbZ}{\mathbb{Z}}
\newcommand{\bbE}{\mathbb{E}}
\newcommand{\bbS}{\mathbb{S}}
\newcommand{\bbN}{\mathbb{N}}
\newcommand{\bbP}{\mathbb{P}}
\newcommand{\bbI}{\mathbb{I}}

\newcommand{\bfi}{\mathbf{i}}
\newcommand{\bfj}{\mathbf{j}}
\newcommand{\bfn}{\mathbf{n}}

\newcommand{\bfOne}{\mathbf{1}}

\newcommand{\bfX}{\mathbf{X}}
\newcommand{\bfY}{\mathbf{Y}}
\newcommand{\bfZ}{\mathbf{Z}}
\newcommand{\bftheta}{\boldsymbol{\theta}}

\newcommand{\Unif}{\text{\rm Unif}}

\def\Symn{{\bbS_n}}
\newcommand{\Symnd}[1]{\bbS_{n_{#1}}}

\newcommand{\argdot}{{\,\vcenter{\hbox{\scalebox{0.75}{$\bullet$}}}\,}}
\def\condind{{\perp\!\!\!\perp}}
\def\equdist{\stackrel{\text{\rm\tiny d}}{=}}
\def\equas{\stackrel{\text{\rm\tiny a.s.}}{=}}
\def\iid{i.i.d.\ }
\newcommand{\simiid}{\overset{\text{\scriptsize{iid}}}{\sim}}

\def\inv{\text{\rm{inv}}}
\def\equ{\text{\rm{eq}}}

\DeclareMathOperator*{\argmin}{arg\,min}

\renewcommand{\paragraph}[1]{\vspace{0.5\baselineskip} \noindent\emph{#1.}\ }

\pgfmathdeclarefunction{gauss}{2}{%
  \pgfmathparse{0.01/(#2*sqrt(2*pi))*exp(-((x-#1)^2)/(2*#2^2))}%
}

\endlocaldefs

\begin{document}

\begin{frontmatter}
  \title{Probabilistic symmetries and invariant neural networks}
  \runtitle{Probabilistic symmetries and invariant neural networks}

\begin{aug}
  \author[A]{\fnms{Benjamin} \snm{Bloem-Reddy}\corref{}\ead[label=e1]{benbr@stat.ubc.ca}}
  \and
  \author[B]{\fnms{Yee Whye} \snm{Teh}\ead[label=e2]{y.w.teh@stats.ox.ac.uk}}

  \runauthor{Bloem-Reddy and Teh}

  \address[A]{Department of Statistics\\
              University of British Columbia\\
              \printead{e1}}

  \address[B]{Department of Statistics\\
            University of Oxford\\
            \printead{e2}}

\end{aug}

\begin{abstract}
  Treating neural network inputs and outputs as random variables, we characterize the structure of neural networks that can be used to model data that are invariant or equivariant under the action of a compact group. Much recent research has been devoted to encoding invariance under symmetry transformations into neural network architectures, in an effort to improve the performance of deep neural networks in data-scarce, non-i.i.d., or unsupervised settings. By considering group invariance from the perspective of probabilistic symmetry, we establish a link between functional and probabilistic symmetry, and obtain generative functional representations of probability distributions that are invariant or equivariant under the action of a compact group. Our representations completely characterize the structure of neural networks that can be used to model such distributions and yield a general program for constructing invariant stochastic or deterministic neural networks. We demonstrate that examples from the recent literature are special cases, and develop the details of the general program for exchangeable sequences and arrays. 
\end{abstract}

\end{frontmatter}

\tableofcontents


\section{Introduction}
\label{sec:intro}

Neural networks and deep learning methods have found success in a wide variety of applications.  Much of the success has been attributed to a confluence of trends, including the increasing availability of data; advances in specialized hardware such as GPUs and TPUs; and open-source software like \textsc{Theano} \citep{theano:2016}, \textsc{Tensorflow} \citep{tensorflow:2015}, and \textsc{PyTorch} \citep{pytorch:2017}, that enable rapid development of neural network models through automatic differentiation and high-level interfaces with specialized hardware. Neural networks have been most successful in settings with massive amounts of i.i.d.\ labeled training data. In recent years, a concerted research effort has aimed to improve the performance of deep learning systems in data-scarce and semi-supervised or unsupervised problems, and for structured, non-i.i.d.\ data. In that effort, there has been a renewed focus on novel neural network architectures: attention mechanisms \citep{Vaswani:etal:2017}, memory networks \citep{Sukhbaatar:etal:2015}, dilated convolutions \citep{Yu:Koltun:2016}, residual networks \citep{He:etal:2016}, and graph neural networks \citep{Scarselli:etal:2009} are a few recent examples from the rapidly expanding literature.

The focus on novel architectures reflects a basic fact of machine learning: in the presence of data scarcity, whether due to small sample size or unobserved variables, or to complicated structure, the model must pick up the slack. 
Amid the flurry of model-focused innovation, there is a growing need for a framework for encoding modeling assumptions and checking their validity, and for assessing the training stability and generalization potential of an architecture. In short, a principled theory of neural network design is needed.

This paper represents one small step in that direction. It concerns the development of neural network architectures motivated by symmetry considerations, typically described by invariance or equivariance with respect to the action of a group. The most well-known examples of such architectures are convolutional neural networks (CNNs) \citep{LeCun:etal:1989}, which ensure invariance of the output $Y$ under translations of an input image $X$. 
Other examples include neural networks that encode rotational invariance \citep{Cohen:etal:2018} or permutation invariance \citep{Zaheer:etal:2017,Hartford:etal:2018,Herzig:etal:2018,Lee:etal:2018:set:transformer}.

Within this growing body of work, symmetry is most often addressed by designing a specific neural network architecture for which the relevant invariance can be verified. A more general approach aims to answer the question:
\begin{displayquote}
\emph{For a particular symmetry property, can {all} invariant neural network architectures be characterized?}
\end{displayquote}
General results have been less common in the literature; important exceptions include characterizations of feed-forward networks (i.e., linear maps composed with pointwise non-linearities) that are invariant under the action of discrete groups \citep{ShaweTaylor:1989,Ravanbakhsh:etal:2017}, finite linear groups \citep{Wood:ShaweTaylor:1996}, or compact groups \citep{Kondor:Trivedi:2018}.

In the probability and statistics literature, there is a long history of probabilistic model specification motivated by symmetry considerations. In particular, if a random variable $X$ is to be modeled as respecting some symmetry property, formalized as invariance under certain transformations, then a model should only contain invariant distributions $P_X$. The relevant theoretical question is:
\begin{displayquote}
\emph{For a particular symmetry property of $X$, can {all} invariant distributions $P_X$ be characterized?}
\end{displayquote}
Work in this area dates at least to the 1930s, and the number of general results reflects the longevity of the field. 
The most famous example is de Finetti's theorem \citep{deFinetti:1930}, which is a cornerstone of Bayesian statistics and machine learning. It shows that all infinitely exchangeable (i.e., permutation-invariant) sequences of random variables have distributional representations that are conditionally i.i.d., conditioned on a random probability measure. Other examples include rotational invariance, translation invariance, and a host of others.

In the present work, we approach the question of invariance in neural network architectures from the perspective of probabilistic symmetry. In particular, we seek to understand the symmetry properties of joint probability distributions of $(X,Y)$ that are necessary and sufficient for the existence of a \emph{noise-outsourced functional representation} $Y = f(\noise,X)$, with generic noise variable $\noise$, such that $f$ obeys the relevant functional symmetries. 
Our approach sheds light on the core statistical issues involved, and provides a broader view of the questions posed above: from considering classes of deterministic functions to stochastic ones; and from invariant marginal distributions to invariant joint and conditional distributions. Taking a probabilistic approach also leads, in many cases, to simpler proofs of results relative to their deterministic counterparts. Furthermore, stochastic networks are desirable for a number of practical reasons; we discuss these in detail in \cref{sec:practical}. 

\paragraph{Outline} 
The remainder of this section provides further background and related work, gives a high-level overview of our main results, and introduces the necessary measure theoretic technicalities and notation. 
\Cref{sec:two:types} defines the relevant functional and probabilistic notions of symmetries; the similarities between the two suggests that there is deeper a mathematical link. 
\Cref{sec:sufficiency:adequacy} provides statistical background, reviewing the ideas of sufficiency and adequacy, and of a technical tool known as noise outsourcing.
\Cref{sec:connecting:symmetries} makes the precise mathematical link alluded to in \cref{sec:two:types} by establishing functional representations of conditional distributions that are invariant or equivariant under the action of a compact group. 
\Cref{sec:practical} examines a number of practical considerations and related ideas in the context of the current machine learning literature, and gives a general program for designing invariant neural networks in \cref{sec:general:program}. 
\Cref{sec:finite:exch:seq,sec:finite:exch:arrays} develop the details of the program for exchangeable sequences, arrays, and graphs. Technical details and proofs not given in the main text are in \cref{sec:proof:noise:out:suff,sec:technical,sec:proof:arrays}.

\subsection{Symmetry in Deep Learning} 

Interest in neural networks that are invariant to discrete groups acting on the network nodes dates back at least to the text of \citet{Minsky:Papert:1988} on single-layer perceptrons (SLPs), who used their results to demonstrate a certain limitation of SLPs. \Citet{ShaweTaylor:1989,ShaweTaylor:1993} extended the theory to multi-layer perceptrons, under the name Symmetry Networks. The main findings of that theory, that invariance is achieved by weight-preserving automorphisms of the neural network, and that the connections between layers must be partitioned into weight-sharing orbits, were rediscovered by \citet{Ravanbakhsh:etal:2017}, who also proposed novel architectures and new applications. 

\Citet{Wood:ShaweTaylor:1996} extended the theory to invariance of feed-forward networks under the action of finite linear groups. Some of their results overlap with results for compact groups found in \citet{Kondor:Trivedi:2018}, including the characterization of equivariance in feed-forward networks in terms of group theoretic convolution.

The most widely applied invariant neural architecture is the CNN for input images. 
Recently, there has been a surge of interest in generalizing the idea of invariant architectures to other data domains such as sets and graphs, with most work belonging to either of two categories:
\begin{enumerate}[label=(\roman*)]
  \item properly defined convolutions \citep{Bruna:etal:2014,Duvenaud:etal:2015,Niepert:etal:2016}; or 
  \item equivariance under the action of groups that lead to weight-tying schemes \citep{Gens:Domingos:2014,Cohen:Welling:2016,Ravanbakhsh:etal:2017}. 
\end{enumerate}
Both of these approaches rely on group theoretic structure in the set of symmetry transformations, and \citet{Kondor:Trivedi:2018} used group theory to show that the two approaches are the same (under homogeneity conditions on the group's action) when the network layers consist of pointwise non-linearities applied to linear maps; \citet{Cohen:etal:2019} extended the theory to more general settings. \Citet{Ravanbakhsh:etal:2017} demonstrated an exception (violating the homogeneity condition) to this correspondence for discrete groups.

Specific instantiations of invariant architectures abound in the literature; they are too numerous to collect here. However, we give a number of examples in \cref{sec:practical}, and in \cref{sec:finite:exch:seq,sec:finite:exch:arrays} in the context of sequence- and graph-valued input data that are invariant under permutations.

\subsection{Symmetry in Probability and Statistics}

The study of probabilistic symmetries has a long history. Laplace's ``rule of succession'' dates to 1774; it is the conceptual precursor to exchangeability \citep[see][for a historical and philosophical account]{Zabell:2005}. Other examples include invariance under rotation and stationarity in time \citep{Freedman:1963}; the former has roots in Maxwell's work in statistical mechanics in 1875 \citep[see, for example, the historical notes in][]{Kallenberg:2005}. The present work relies on the deep connection between sufficiency and symmetry; \citet{Diaconis:1988} gives an accessible overview.

In Bayesian statistics, the canonical probabilistic symmetry is exchangeability. A sequence of random variables, $\Xn=(X_1,\dotsc,X_n)$, is exchangeable if its distribution is invariant under all permutations of its elements. 
If that is true for every $n\in\bbN$ in an infinite sequence $\Xinf$, then the sequence is said to be \emph{infinitely exchangeable}. de Finetti's theorem \citep{deFinetti:1930,Hewitt:Savage:1955} shows that infinitely exchangeable distributions have particularly simple structure. Specifically, $\Xinf$ is infinitely exchangeable if and only if there exists some random distribution $Q$ such that the elements of $\Xinf$ are conditionally \iid with distribution $Q$. Therefore, each infinitely exchangeable distribution $P$ has an integral decomposition: there is a unique (to $P$) distribution $\nu$ on the set $\calM_1(\calX)$ of all probability measures on $\calX$, such that
\begin{align} \label{eq:deFinetti:cond:iid}
  P(\Xinf) = \int_{\calM_1(\calX)} \prod_{i=1}^{\infty} Q(X_i) \nu(dQ) \;. 
\end{align}
The simplicity is useful: by assuming the data are infinitely exchangeable, only models that have a conditionally \iid structure need to be considered. This framework is widely adopted throughout Bayesian statistics and machine learning. 

de Finetti's theorem is a special case of a more general mathematical result, the ergodic decomposition theorem, which puts probabilistic symmetries in correspondence with integral decompositions like \eqref{eq:deFinetti:cond:iid}; see \citet{Orbanz:Roy:2015} for an accessible overview. 
de Finetti's results inspired a large body of work on other symmetries in the probability literature; \citet{Kallenberg:2005} gives a comprehensive treatment and \citet{Kallenberg:2017} contains further results. Applications of group symmetries in statistics include equivariant estimation and testing; see, for example, \citet[][Ch.\ 6]{Lehmann:Romano:2005}; \citet{Eaton:1989,WIjsman:1990,Giri:1996}. 
The connection between symmetry and statistical sufficiency, which we review in \cref{sec:sufficiency:adequacy}, was used extensively by \citet{Diaconis:Freedman:1984} and by \citet{Lauritzen:1984} in their independent developments of ``partial exchangeability'' and ``extremal families'', respectively; see \citet{Diaconis:1988} for an overview. 
\Cref{sec:connecting:symmetries} makes use of maximal invariants to induce conditional independence; related ideas in a different context were developed by \citet{Dawid:1985}.

\subsection{Overview of Main Results}
\label{sec:main:results}

Our main results put functional symmetries in correspondence with probabilistic symmetries. To establish this correspondence, we obtain functional representations of the conditional distribution of an output random variable, $Y\in\calY$, given an input random variable, $X\in\calX$, subject to symmetry constraints. The results provide a general program for constructing deterministic or stochastic functions in terms of statistics that encode a symmetry between input and output, with neural networks as a special case. 

\paragraph{Permutation-invariant output and exchangeable input sequences} 
Consider the example of an exchangeable input sequence $\Xn$. That is, $\perm\cdot\Xn \equdist \Xn$ for all permutations $\perm\in\Symn$, where `$\equdist$' denotes equality in distribution. By definition, the order of the elements has no statistical relevance; only the values contain any information about the distribution of $\Xn$. The \emph{empirical measure}\footnote{The Dirac delta function $\delta_{x}(B) = 1$ if $x\in B$ and is $0$ otherwise, for any measurable set $B$.} (or counting measure),
\begin{align*}
  \empirical_{\Xn}(\argdot) := \sum_{i=1}^n \delta_{X_i}(\argdot) \;,
\end{align*}
acts to separate relevant from irrelevant statistical information in an exchangeable sequence: it retains the values appearing in $\Xn$, but discards their order. In other words, the empirical measure is a \emph{sufficient statistic} for models comprised of exchangeable distributions. 

Under certain conditions on the output $Y$, the empirical measure also contains all relevant  information for predicting $Y$ from $\Xn$, a property known as \emph{adequacy} (see \cref{sec:sufficiency:adequacy}). In particular, we show in \cref{sec:finite:exch:seq} that for an exchangeable input $\Xn$, the conditional distribution of an output $Y$ is invariant to permutations of $\Xn$ if and only if there exists a function $f$ such that\footnote{`$\equas$' denotes equal almost surely.}
\begin{align} \label{eq:exch:rep:intro}
  (\Xn, Y) \equas \big( \Xn, f(\noise, \empirical_{\Xn} ) \big) \quad \text{where} \quad \noise\sim\Unif[0,1] \quad \text{and} \quad \noise \condind \Xn  \;.
\end{align}
This is an example of a noise-outsourced functional representation of samples from a conditional probability distribution. It can be viewed as a general version of the so-called ``reparameterization trick'' \citep{Kingma:Welling:2014,Rezende:Mohamed:Wierstra:2014} for random variables taking values in general measurable spaces (not just $\bbR$). The noise variable $\noise$ acts as a generic source of randomness that is ``outsourced'', a term borrowed from \citet{Austin:2015}. 
The relevant property of $\noise$ is its independence from $X$, and the uniform distribution is not special in this regard; it could be replaced by, for example, a standard normal random variable and the result would still hold, albeit with a different $f$. For modeling purposes, the outer function $f$ may contain ``pooling functions'' that compress $\empirical_{\Xn}$ further, such as sum or max. The only restriction imposed by being a function of $\empirical_{\Xn}$ is that the order of elements in $\Xn$ has been discarded: $f$ must be invariant to permutations of $\Xn$. We give further examples in \cref{sec:finite:exch:seq}.

The representation \eqref{eq:exch:rep:intro} is a characterization of permutation-invariant stochastic functions, and generalizes a deterministic version that \citet{Zaheer:etal:2017} obtained under more restrictive assumptions; \citet{Murphy:etal:2019} and \citet{Wagstaff:2019aa}, among others, provide extensions and further theoretical study of deterministic permutation-invariant functions.

\paragraph{Permutation-equivariant output and exchangeable input sequences} 
For the purposes of constructing deep neural networks, invariance often imposes stronger constraints than are desirable for hidden layers. Instead, the output should transform predictably under transformations of the input, in a manner that preserves the structure of the input; this is a property known as \emph{equivariance}. For simplicity, consider an output, $\Yn = (Y_1,\dotsc,Y_n)$, of the same size as the input $\Xn$. For exchangeable input $\Xn$, the conditional distribution of $\Yn$ is equivariant to permutations of $\Xn$ if and only if there exists a function $f$ such that\footnote{The representation also requires that the elements of $\Yn$ are conditionally independent given $\Xn$, which is trivially satisfied by most neural network architectures; see \cref{sec:equivariant:seq}.}
\begin{align*}
  (\Xn, \Yn) \equas \big( \Xn, (f(\noise_i,X_i,\empirical_{\Xn}))_{1\leq i \leq n} \big) \quad \text{where} \quad \noise_i\simiid\Unif[0,1] \quad \text{and} \quad (\noise_i)_{1\leq i \leq n} \condind \Xn \;.
\end{align*}
Special (deterministic) cases of $f$ have appeared in the literature \citep{Zaheer:etal:2017,Lee:etal:2018:set:transformer}, with equivariance demonstrated on a case-by-case basis; to the best of our knowledge, no version of this general result has previously appeared. 
We give further examples in \cref{sec:finite:exch:seq}.

\paragraph{Invariance and equivariance under compact groups} 
The representation \eqref{eq:exch:rep:intro} explicitly separates the conditional distribution of $Y$ given $\Xn$ into relevant structure (the empirical measure $\empirical_{\Xn}$) and independent random noise. Such separation is a recurring theme throughout the present work. For more general groups of transformations than permutations, the relevant structure is captured by a \emph{maximal invariant} statistic $M(X)$, which can be used to partition the input space into equivalence classes under the action of a group $\grp$. We show in \cref{sec:connecting:symmetries} that for an input $X$ satisfying $g\cdot X \equdist X$ for all $g\in\grp$, the conditional distribution of an output $Y$ is invariant to the action of $\grp$ on $X$ if and only if there exists some function $f$ such that 
\begin{align} \label{eq:invariant:rep:intro}
  (X, Y) \equas \big( X, f(\noise, M(X) ) \big) \quad \text{where} \quad \noise\sim\Unif[0,1] \quad \text{and} \quad \noise \condind X  \;.
\end{align}
Because $M$ is $\grp$-invariant and $f$ depends on $X$ only through $M$, $f$ is also $\grp$-invariant. The empirical measure is an example of a maximal invariant for $\Symn$ acting on $\calX^n$, leading to \eqref{eq:exch:rep:intro} as a special case of \eqref{eq:invariant:rep:intro}. 

We also obtain a functional representation of all $\grp$-equivariant conditional distributions: for input $X$ with $\grp$-invariant distribution, the conditional distribution of $Y$ given $X$ is $\grp$-equivariant if and only if 
\begin{align*}  
  (X, Y) \equas \big( X, f(\noise, X ) \big) \quad \text{where} \quad \noise\sim\Unif[0,1] \quad \text{and} \quad \noise \condind X  \;,
\end{align*}
for $f$ satisfying
\begin{align*}  
  g\cdot Y = g\cdot f(\noise, X) = f(\noise, g\cdot X) \;, \quad \text{a.s., for all } g\in\grp \;.  
\end{align*}
Specifically, distributional equivariance is equivalent to for $(X,Y)\equdist (g\cdot X, g\cdot Y)$, $g \in \grp$. For $(X,Y)$ satisfying that condition, one may construct an equivariant function from a \emph{representative equivariant} $\tau(X)$, which is an element of $\grp$ that maps $X$ to a representative element of its equivalence class. See \cref{sec:equivariant:representations} for details.

\paragraph{Exchangeable arrays and graphs} 
In addition to applying the general theory to exchangeable sequences (\cref{sec:finite:exch:seq}), we develop the details of functional representations of exchangeable arrays and graphs in \cref{sec:finite:exch:arrays}. Consider an input array $\Xntwo$, modeled as invariant under separate permutations of its rows and columns. Such \emph{separately exchangeable} arrays arise in, for example, collaborative filtering problems like recommender systems. The analogue of the empirical measure in \eqref{eq:exch:rep:intro} is any \emph{canonical form} $\Cntwo := \canon(\Xntwo)$.  
The conditional distribution of an output $Y$ is invariant under permutations of the rows and columns of a separately exchangeable input $\Xntwo$ if and only if there exists a function $f$ such that
\begin{align*}  
  (\Xntwo,Y) \equas (\Xntwo,f(\noise ,\Cntwo)) \quad \text{where} \quad \noise\sim\Unif[0,1] \quad \text{and} \quad \noise \condind \Xntwo \;.
\end{align*}
Analogously, the conditional distribution of an output array $\Yntwo$ is equivariant under permutations of the rows and columns of $\Xntwo$ if and only if there exists a function $f$ such that
\begin{align*}  
    \big(\Xntwo, \Yntwo  \big)
    \equas 
    \big(\Xntwo,
      \big(
        f(
        \noise_{i,j}, 
        X_{i,j}, 
        \Cntwo^{(i,j)}
        )
      \big)_{i\leq n_1, j\leq n_2}
    \big) \;,
\end{align*}
where $\Cntwo^{(i,j)}$ is an augmented version of $\Cntwo$, with the $i$th row and $j$th column of $\Xntwo$ broadcast to the entire array.  Deterministic special cases of these representations have appeared in, for example \citet{Hartford:etal:2018,Herzig:etal:2018}. 
See \cref{sec:finite:exch:arrays} for details and more examples, where analogous results are developed for graphs (i.e., symmetric arrays that are exchangeable under the same permutation applied to the rows and columns) and for arrays with features associated with each row and column. Results for $d$-dimensional arrays (i.e., tensors) are given in \cref{sec:proof:arrays}.

\paragraph{Contributions} The probabilistic techniques employed here are generally not new (most are fairly standard), though to our knowledge they have not been applied to the types of problems considered here. Furthermore, the invariant deep learning literature has been in apparent disconnect from important ideas concerning probabilistic symmetry, particularly as related to statistical sufficiency (see \cref{sec:sufficiency}) and foundational ideas in Bayesian statistics; and vice versa. This allows for exchange of ideas between the areas. In addition to certain practical advantages of taking a probabilistic approach (which we detail in \cref{sec:practical:advantages}), we obtain \emph{the most general possible representation results} for stochastic or deterministic functions that are invariant or equivariant with respect to a compact group.  Our representation results are \emph{exact}, as opposed to recent results on universal approximation (see \cref{sec:practical:choosing}). To the best of our knowledge, no results at that level of generality exist in the deep learning literature.

\subsection{Technicalities and Notation}
\label{sec:meas:tech} 

In order to keep the presentation as clear as possible, we aim to minimize measure theoretic technicalities. Throughout, it is assumed that there is a background probability space $(\Omega,\calA,\bbP)$ that is rich enough to support all required random variables. All random variables are assumed to take values in standard Borel spaces (spaces that are Borel isomorphic to a Borel subset of the unit interval \citep[see, e.g.,][]{Kallenberg:2002}). For example, $X$ is a $\calX$-valued random variable in $(\calX,\calB_{\calX})$, where $\calB_{\calX}$ is the Borel $\sigma$-algebra of $\calX$. Alternatively, we may say that $X$ is a random element of $\calX$. 
For notational convenience, for a $\calY$-valued random variable $Y$, we write $P(Y\in\argdot \mid X)$ as shorthand for, ``for all sets $A\in \calB_{\calY}$, $P(Y \in A \mid \sigma(X))$'', where $\sigma(X)$ is the $\sigma$-algebra generated by $X$. We use $\calM(\calX)$ to denote the set of measures on $\calX$, and $\calM_1(\calX)$ to denote the set of probability measures on $\calX$. Many of our results pertain to conditional independence relationships; $Y \condind_{Z} X$ means that $Y$ and $X$ are conditionally independent, given $\sigma(Z)$. Finally, $\equdist$ denotes equality in distribution, and $\equas$ denotes almost sure equality. 

We consider symmetries induced by the action of a group. A group is a set, $\grp$, and a binary composition operator $\cdot$ that together must satisfy four properties: for each $g, g'\in\grp$ the composition $g\cdot g' \in\grp$ is in the group; the group operation is associative, that is, $g\cdot (g' \cdot g'') = (g \cdot g') \cdot g''$ for all $g, g', g'' \in \grp$; there is an identity element $e\in\grp$ such that $e\cdot g = g\cdot e = g$ for all $g\in\grp$; and for each $g\in\grp$ there is an inverse $g^{-1}\in\grp$ such that $g^{-1} \cdot g = g\cdot g^{-1}=e$. See, e.g., \citet{Rotman:1995}. 
Let ${\trans_{\calX} : \grp \times \calX \to \calX}$ be the left-action of $\grp$ on the input space $\calX$, such that $\trans_{\calX}(e,x) = x$ is the identity mapping for all $x\in\calX$, and $\trans_{\calX}(g,\trans_{\calX}(g',x))=\trans_{\calX}(g\cdot g',x)$ for $g,g'\in\grp$, $x\in\calX$. For convenience, we write $g\cdot x = \trans_{\calX}(g,x)$.  Similarly, let $\trans_{\calY}$ be the action of $\grp$ on the output space $\calY$, and $g\cdot y = \trans_{\calY}(g,y)$ for $g\in\grp$, $y\in\calY$. In this paper, we always assume $\grp$ to be compact. 
A group $\grp$, along with a $\sigma$-algebra $\sigma(\grp)$, is said to be measurable if the group operations of inversion $g \mapsto g^{-1}$ and composition $(g,g')\mapsto g\cdot g'$ are $\sigma(\grp)$-measurable. $\grp$ \emph{acts measurably} on $\calX$ if $\trans_{\calX}$ is a measurable function $\sigma(\grp) \otimes \borel_{\calX} \to \borel_{\calX}$ \citep{Kallenberg:2017}.

\section{Functional and Probabilistic Symmetries}
\label{sec:two:types}

We consider the relationship between two random variables $X\in\calX$ and $Y\in\calY$, with $Y$ being a predicted output based on input $X$.  For example, in image classification, $X$ might be an image and $Y$ a class label; in sequence prediction, $X=(X_i)_{i=1}^n$ might be a sequence and $Y$ the next element, $X_{n+1}$, to be predicted; and in a variational autoencoder, $X$ might be an input vector and $Y$ the corresponding latent variable whose posterior is to be inferred using the autoencoder. Throughout, we denote the joint distribution of both variables as $P_{X,Y}$, and the conditional distribution $P_{Y|X}$ is the primary object of interest.

Two basic notions of symmetry are considered, as are the connections between them. The first notion, defined in \cref{sec:background:networks}, is functional, and is most relevant when $Y$ is a deterministic function of $X$, say $Y=f(X)$; the symmetry properties pertain to the function $f$. The second notion, defined in \cref{sec:background:prob:symm}, is probabilistic, and pertains to the conditional distribution of $Y$ given $X$.

\subsection{Functional Symmetry in Neural Networks}
\label{sec:background:networks}

In many machine learning settings, a prediction $y$ based on an input $x$ is modeled as a deterministic function, $y=f(x)$, where $f$ belongs to some function class $\calF = \{ f ; f : \calX \to \calY  \}$, often satisfying some further conditions.  For example, $f$ might belong to a Reproducing Kernel Hilbert Space, or to a subspace of all strongly convex functions. Alternatively, $f$ may be a neural network parameterized by weights and biases collected into a parameter vector $\bftheta$, in which case the function class $\calF$ corresponds to the chosen network architecture, and a particular $f$ corresponds to a particular set of values for $\bftheta$.  We are concerned with implications on the choice of the network architecture (equivalently, the function class $\calF$) due to symmetry properties we impose on the input-output relationship.  In this deterministic setting, the conditional distribution $P_{Y|X}$ is simply a point mass at $f(x)$. 

Two properties, invariance and equivariance, formalize the relevant symmetries. 
A function $f : \calX \to \calY$ is \emph{invariant under $\grp$}, or \emph{$\grp$-invariant}, if the output is unchanged by transformations of the input induced by the group:
\begin{align} \label{eq:invariant:function:def}
  f(g \cdot x) = f(x) \quad \text{for all } g\in\grp, \ x\in\calX \;.
\end{align}
Alternatively, a function $f:\calX \to \calY$ is \emph{equivariant under $\grp$}, or \emph{$\grp$-equivariant}, if
\begin{align} \label{eq:equivariant:function:def}
  f(g \cdot x) = g \cdot f(x) \quad \text{for all } g\in\grp, \ x\in\calX \;.
\end{align}
The action of $\grp$ commutes with the application of an equivariant $f$; transforming the input is the same as transforming the output. 
Note that the action of $\grp$ on $\calX$ and $\calY$ may be different. In particular, invariance is a special case of equivariance, whereby the group action in the output space is trivial: $g\cdot y=y$ for each $g\in\grp$ and $y\in\calY$. Invariance imposes stronger restrictions on the functions satisfying it, as compared to equivariance. 

These properties and their implications for network architectures are illustrated with examples from the literature. For notational convenience, let $[n]=\{1,\dotsc,n\}$ and $\Xn=(X_1,\dotsc,X_n)\in\calX^n$. Finally, denote the finite symmetric group of a set of $n$ elements (i.e., the set of all permutations of $[n]$) by $\Symn$.

\begin{example}[Deep Sets: $\Symn$-invariant functions of sequences] 
  \label{ex:deep:sets}
  \Citet{Zaheer:etal:2017} considered a model $Y=f(\Xn)$, where the input $\Xn$ was treated as a set, i.e., the order among its elements did not matter. Those authors required that the output of $f$ be unchanged under all permutations of the elements of $\Xn$, i.e., that $f$ is $\Symn$-invariant. They found that $f$ is $\Symn$-invariant if and only if it can be represented as $f(\Xn) = \tilde{f}(\sum_{i=1}^n \phi(X_i))$ for some functions $\tilde{f}$ and $\phi$. Clearly, permutations of the elements of $\Xn$ leave such a function invariant: $f(\Xn)=f(\perm\cdot \Xn)$ for $\perm\in\Symn$. The fact that \emph{all} $\Symn$-invariant functions can be expressed in such a form was proved by \citet{Zaheer:etal:2017} in two different settings: (i) for sets of arbitrary size when $\calX$ is countable; and (ii) for sets of fixed size when $\calX$ is uncountable. In both proofs, the existence of such a form is shown by constructing a $\phi$ that uniquely encodes the elements of $\calX$. Essentially, $\phi$ is a generalization of a one-hot encoding, which gets sum-pooled and passed through $\tilde{f}$. The authors call neural architectures satisfying such structure \emph{Deep Sets}. See \Cref{fig:mlp:deep:sets}, left panel, for an example diagram. 
  In \cref{sec:invariant:seq}, we give a short and intuitive proof using basic properties of exchangeable sequences.
\end{example}

\begin{example}[$\Symn$-equivariant neural network layers]
  \label{ex:deep:sets:equivariant}
  Let $\Xn$ and $\Yn$ represent adjacent layers of a standard feed-forward neural network, such that the nodes are indexed by $[n]$, with $X_i\in\bbR$ the $i$th node in layer $\Xn$, and similarly for $Y_i$. In a feed-forward layer, $\Yn=\activation(\boldsymbol{\theta} \Xn)$ where $\activation$ is an element-wise nonlinearity and $\boldsymbol{\theta}$ is a weight matrix (we ignore biases for simplicity). \Citet{ShaweTaylor:1989,Wood:ShaweTaylor:1996,Zaheer:etal:2017}  showed that the only weight matrices that lead to $\Symn$-equivariant layers are the sum of a diagonal matrix, $\theta_0 \bbI_n$, $\theta_0\in\bbR$, and a constant one, $\theta_1 \bfOne_n \bfOne_n^T$, $\theta_1\in\bbR$, such that
  \begin{align} \label{eq:deep:sets:layer}
    [\boldsymbol{\theta} \Xn]_i = \theta_0 X_i + \theta_1 \sum_{j=1}^n X_j \;.
  \end{align}
  \Cref{fig:mlp:deep:sets}, right panel, shows the weight-sharing patterns of the connections.
\end{example}

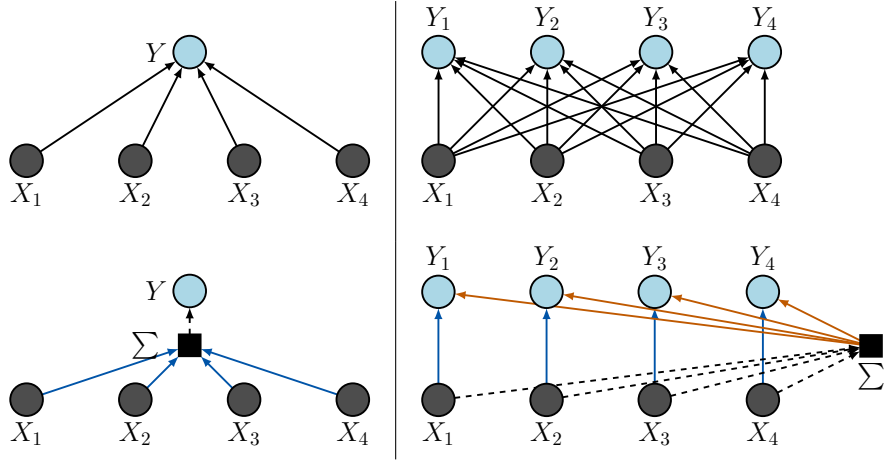
\begin{figure}
  \begin{center}
  \begin{tabular}{c | c}
    \resizebox{0.45\textwidth}{!}{
       
      \begin{tikzpicture}
        \Vertex[x=4,y=-1,Pseudo]{BPs}
        \Vertex[x=0,y=0,color=darkgrayClj,label=$X_1$,fontsize=\Large,position=below]{I1}
        \Vertex[x=2,y=0,color=darkgrayClj,label=$X_2$,fontsize=\Large,position=below]{I2}
        \Vertex[x=4,y=0,color=darkgrayClj,label=$X_3$,fontsize=\Large,position=below]{I3}
        \Vertex[x=6,y=0,color=darkgrayClj,label=$X_4$,fontsize=\Large,position=below]{I4}
        \Vertex[x=3,y=2,label=$Y$,fontsize=\Large,position=left]{O1}
        \Vertex[x=-2,y=1,Pseudo]{Ps}

        \Edge[lw=1,color=black,Direct=true](I1)(O1)
        \Edge[lw=1,color=black,Direct=true](I2)(O1)
        \Edge[lw=1,color=black,Direct=true](I3)(O1)
        \Edge[lw=1,color=black,Direct=true](I4)(O1)
      \end{tikzpicture}
    }
    &
    \resizebox{0.45\textwidth}{!}{
       
      \begin{tikzpicture}
        \Vertex[x=4,y=-1,Pseudo]{BPs}
        \Vertex[x=0,y=0,color=darkgrayClj,label=$X_1$,fontsize=\Large,position=below]{I1}
        \Vertex[x=2,y=0,color=darkgrayClj,label=$X_2$,fontsize=\Large,position=below]{I2}
        \Vertex[x=4,y=0,color=darkgrayClj,label=$X_3$,fontsize=\Large,position=below]{I3}
        \Vertex[x=6,y=0,color=darkgrayClj,label=$X_4$,fontsize=\Large,position=below]{I4}
        \Vertex[x=0,y=2,label=$Y_1$,fontsize=\Large,position=above]{H1}
        \Vertex[x=2,y=2,label=$Y_2$,fontsize=\Large,position=above]{H2}
        \Vertex[x=4,y=2,label=$Y_3$,fontsize=\Large,position=above]{H3}
        \Vertex[x=6,y=2,label=$Y_4$,fontsize=\Large,position=above]{H4}
        \Vertex[x=8,y=1,Pseudo]{Ps}

        \Edge[lw=1,color=black,Direct=true](I1)(H1)
        \Edge[lw=1,color=black,Direct=true](I2)(H1)
        \Edge[lw=1,color=black,Direct=true](I3)(H1)
        \Edge[lw=1,color=black,Direct=true](I4)(H1)

        \Edge[lw=1,color=black,Direct=true](I1)(H2)
        \Edge[lw=1,color=black,Direct=true](I2)(H2)
        \Edge[lw=1,color=black,Direct=true](I3)(H2)
        \Edge[lw=1,color=black,Direct=true](I4)(H2)

        \Edge[lw=1,color=black,Direct=true](I1)(H3)
        \Edge[lw=1,color=black,Direct=true](I2)(H3)
        \Edge[lw=1,color=black,Direct=true](I3)(H3)
        \Edge[lw=1,color=black,Direct=true](I4)(H3)

        \Edge[lw=1,color=black,Direct=true](I1)(H4)
        \Edge[lw=1,color=black,Direct=true](I2)(H4)
        \Edge[lw=1,color=black,Direct=true](I3)(H4)
        \Edge[lw=1,color=black,Direct=true](I4)(H4)
      \end{tikzpicture}
    }
    \\
    \resizebox{0.45\textwidth}{!}{
       
      \begin{tikzpicture}
        \Vertex[x=0,y=0,color=darkgrayClj,label=$X_1$,fontsize=\Large,position=below]{I1}
        \Vertex[x=2,y=0,color=darkgrayClj,label=$X_2$,fontsize=\Large,position=below]{I2}
        \Vertex[x=4,y=0,color=darkgrayClj,label=$X_3$,fontsize=\Large,position=below]{I3}
        \Vertex[x=6,y=0,color=darkgrayClj,label=$X_4$,fontsize=\Large,position=below]{I4}
        \Vertex[x=3,y=1,color=black,shape=rectangle,size=0.4,label=$\sum\ \ $,fontsize=\Large,position=left]{phi}
        \Vertex[x=3,y=2,label=$Y$,fontsize=\Large,position=left]{O1}
        \Vertex[x=-2,y=1,Pseudo]{Ps}

        \Edge[lw=1,color=blueClj,Direct=true](I1)(phi)
        \Edge[lw=1,color=blueClj,Direct=true](I2)(phi)
        \Edge[lw=1,color=blueClj,Direct=true](I3)(phi)
        \Edge[lw=1,color=blueClj,Direct=true](I4)(phi)
        \Edge[lw=1,color=black,Direct=true,style={dashed}](phi)(O1)
      \end{tikzpicture}
    }
    &
    \resizebox{0.45\textwidth}{!}{
       
      \begin{tikzpicture}
        \Vertex[x=0,y=0,color=darkgrayClj,label=$X_1$,fontsize=\Large,position=below]{I1}
        \Vertex[x=2,y=0,color=darkgrayClj,label=$X_2$,fontsize=\Large,position=below]{I2}
        \Vertex[x=4,y=0,color=darkgrayClj,label=$X_3$,fontsize=\Large,position=below]{I3}
        \Vertex[x=6,y=0,color=darkgrayClj,label=$X_4$,fontsize=\Large,position=below]{I4}
        \Vertex[x=0,y=2,label=$Y_1$,fontsize=\Large,position=above]{H1}
        \Vertex[x=2,y=2,label=$Y_2$,fontsize=\Large,position=above]{H2}
        \Vertex[x=4,y=2,label=$Y_3$,fontsize=\Large,position=above]{H3}
        \Vertex[x=6,y=2,label=$Y_4$,fontsize=\Large,position=above]{H4}
        \Vertex[x=8,y=1,color=black,size=0.4,shape=rectangle,label=$\sum$,fontsize=\Large,position=below]{phi}

        \Edge[lw=1,color=blueClj,Direct=true](I1)(H1)
        \Edge[lw=1,color=blueClj,Direct=true](I2)(H2)
        \Edge[lw=1,color=blueClj,Direct=true](I3)(H3)
        \Edge[lw=1,color=blueClj,Direct=true](I4)(H4)

        \Edge[lw=1,color=black,Direct=true,style={dashed}](I1)(phi)
        \Edge[lw=1,color=black,Direct=true,style={dashed}](I2)(phi)
        \Edge[lw=1,color=black,Direct=true,style={dashed}](I3)(phi)
        \Edge[lw=1,color=black,Direct=true,style={dashed}](I4)(phi)

        \Edge[lw=1,color=orangeClj,Direct=true](phi)(H1)
        \Edge[lw=1,color=orangeClj,Direct=true](phi)(H2)
        \Edge[lw=1,color=orangeClj,Direct=true](phi)(H3)
        \Edge[lw=1,color=orangeClj,Direct=true](phi)(H4)
      \end{tikzpicture}
    }
  \end{tabular}
  \end{center}
  \caption{Computation diagrams for the neural networks in \cref{ex:deep:sets,ex:deep:sets:equivariant}. Black solid arrows indicate general, uncoupled weights; black dashed arrows indicate a fixed function, i.e., an activation function; colored arrows indicate shared weights between arrows of the same color. \emph{Left panel}: A general output layer with a different weight from each $X_i$ to $Y$ (top), and a simple weight-sharing $\Symn$-invariant architecture (bottom). \emph{Right panel}: A fully connected MLP with $n^2$ weights (top), and the $\Symn$-equivariant architecture corresponding to \eqref{eq:deep:sets:layer}, with two weights (bottom).} 
  \label{fig:mlp:deep:sets}
\end{figure}

These examples demonstrate that equivariance is less restrictive than invariance, and thus allows for more expressive parameterizations; in \cref{ex:deep:sets:equivariant}, invariance would require that $\theta_0=0$. At the same time, equivariance appears to be strong enough to greatly reduce the dimension of the parameter space: generic fully connected layers contain $n^2$ weights, compared to the two used by the $\Symn$-equivariant architecture. 

At a high level, equivariance in feed-forward neural networks ensures that transformations of the input lead to predictable, symmetry-preserving transformations of higher layers, which allows the network to exploit the symmetry in all layers through weight-sharing \citep{Cohen:Welling:2016}, and to summarize features at multiple scales through pooling \citep{Kondor:Trivedi:2018}. 
In addition to being theoretically interesting, group invariance often indicates simplified architectures through parameter-sharing \citep[e.g.,][]{Ravanbakhsh:etal:2017,Cohen:Welling:2017}. These in turn lead to simpler, more stable training and may lead to better generalization \citep{ShaweTaylor:1991}. Related ideas pertaining to the benefits of invariance in the context of data augmentation and feature averaging are found in \citet{Chen:etal:2019,Lyle:etal:2020}. (See also the discussion in \cref{sec:discussion}.)

\subsection{Symmetry in Conditional Probability Distributions}
\label{sec:background:prob:symm}

An alternative approach to the deterministic models of \cref{sec:background:networks} is to model the relationship between input $X$ and output $Y$ as stochastic, either by directly parameterizing the conditional distribution $P_{Y|X}$, or by defining a procedure for generating samples from $P_{Y|X}$. For example, the encoder network of a variational autoencoder computes an approximate posterior distribution over the latent variable $Y$ given observation $X$; in classification, the use of a soft-max output layer is interpreted as a network which predicts a distribution over labels; in implicit models like Generative Adversarial Networks \citep{Goodfellow:etal:2014} or simulation-based models without a likelihood \citep[e.g.,][]{Gutmann:Corander:2016}, and in many probabilistic programming languages \citep[e.g.,][]{Meent:2018aa}, $P_{Y|X}$ is not explicitly represented, but samples are generated and used to evaluate the quality of the model. 

The relevant properties that encode symmetry in such settings are therefore probabilistic. 
One way to define symmetry properties for conditional distributions is by adding noise to invariant or equivariant functions. For example, if $f$ is $\grp$-invariant and $\noise$ is a standard normal random variable independent of $X$, then $Y = \tilde{f}(X) + \noise$ corresponds to a conditional distribution $P_{Y | X}$ that is $\grp$-invariant. This type of construction, with $Y = f(\noise,X)$ satisfying invariance (equivariance) in its second argument, will lead to invariant (equivariant) conditional distributions. 
While intuitive and constructive, it  does not characterize what probabilistic properties must be satisfied by random variables $X$ and $Y$ in order for such representations to exist. Furthermore, it leaves open the question of whether there are other approaches that may be used. 
To avoid these ambiguities, we define notions of symmetries for probability models directly in terms of the distributions. 

The discussion on exchangeability in \cref{sec:intro} pertains only to invariance of the marginal distribution $P_X$ under $\Symn$. Suitable notions of probabilistic symmetry under more general groups are needed for the conditional distribution of $Y$ given $X$. Let $\grp$ be a group acting measurably on $\calX$ and on $\calY$. 
The conditional distribution $P_{Y|X}$ of $Y$ given $X$ is $\grp$-invariant if $Y|X \equdist Y|g\cdot X$.  More precisely, for all $A \in \borel_{\calY}$ and $B \in \borel_{\calX}$,  
\begin{align} \label{eq:prob:invariance}
  P_{Y|X}(Y\in A \mid X\in B) = P_{Y|X}(Y\in A \mid g\cdot X\in B) \quad \text{for all } g\in\grp \;. 
\end{align}
On the other hand, $P_{Y|X}$ is $\grp$-equivariant if $Y|X \equdist g\cdot Y | g\cdot X$ for all $g\in\grp$. That is, transforming $X$ by $g$ leads to the same conditional distribution of $Y$ except that $Y$ is also transformed by $g$.  More precisely, for all $A \in \borel_{\calY}$ and $B \in \borel_{\calX}$,  
\begin{align} \label{eq:prob:equivariance}
  P_{Y|X}(Y\in A\mid X\in B) = P_{Y|X}(g\cdot Y\in A \mid g\cdot X\in B) \quad \text{for all } g\in\grp \;.
\end{align}
Typically, conditional invariance or equivariance is desired because of symmetries in the marginal distribution $P_X$.  In \cref{ex:deep:sets}, it was assumed that the ordering among the elements of the input sequence is unimportant, and therefore it is reasonable to assume that the marginal distribution of the input sequence is exchangeable.  In general, we assume throughout the present work that $X$ is marginally $\grp$-invariant:
\begin{align}
  P_X(X\in B) = P_X(g\cdot X\in B) \quad \text{for all } g\in\grp, B\in\borel_{\calX} \;.
\end{align}
Clearly, $\grp$-invariance of $P_X$ and of $P_{Y|X}$ will result in $P_{X,Y}(X,Y) = P_{X,Y}(g\cdot X, Y)$; similarly, if $P_X$ is $\grp$-invariant and $P_{Y|X}$ is $\grp$-equivariant, then $P_{X,Y}(X,Y) = P_{X,Y}(g\cdot X, g\cdot Y)$. The converse is also true for sufficiently nice groups and spaces (such that conditioning is well-defined). Therefore, we may work with the joint distribution of $X$ and $Y$, which is often more convenient than working with the marginal and conditional distributions. These ideas are summarized in the following proposition, which is a special case of more general results on invariant measures found in, for example, \citet[][Ch.~7]{Kallenberg:2017}.
\begin{proposition}\label{prop:joint:symmetry}
  For a group $\grp$ acting measurably on Borel spaces $\calX$ and $\calY$, if $P_X$ is marginally $\grp$-invariant then
  \begin{enumerate}[label=(\roman*)]
    \item $P_{Y|X}$ is conditionally $\grp$-invariant if and only if $(X,Y) \equdist (g\cdot X,Y)$ for all $g\in\grp$.
    \item $P_{Y|X}$ is conditionally $\grp$-equivariant if and only if $(X,Y) \equdist g\cdot(X,Y) := (g\cdot X, g\cdot Y)$ for all $g\in\grp$, i.e., $X$ and $Y$ are jointly $\grp$-invariant. 
  \end{enumerate}
\end{proposition}

There is a simple algebra of compositions of equivariant and invariant functions.  Specifically, compositions of equivariant functions are equivariant (equivariance is transitive under function composition), and composing an equivariant function with an invariant one yields an invariant function \citep{Cohen:Welling:2016,Kondor:Trivedi:2018}.  Such compositions generate a grammar with which to construct elaborate functions with the desired symmetry properties. Examples abound in the deep learning literature.  These compositional properties carry over to the probabilistic case as well.

\begin{proposition} 
\label{prop:transitivity} \leavevmode
Let $X,Y,Z$ be random variables such that $X\condind_Y Z$. Suppose $X$ is marginally $\grp$-invariant.
\begin{enumerate}[label=(\roman*)]
    \item If $P_{Y|X}$ is conditionally $\grp$-equivariant, and $P_{Z|Y}$ is conditionally $\grp$-equivariant, then $P_{Z|X}$ is also conditionally $\grp$-equivariant.
    \item If $P_{Y|X}$ is conditionally $\grp$-equivariant, and $P_{Z|Y}$ is conditionally $\grp$-invariant, then $P_{Z|X}$ is conditionally $\grp$-invariant. 
  \end{enumerate}
\end{proposition}
\begin{proof}
(i) \cref{prop:joint:symmetry} shows that $X,Y$ are jointly $\grp$-invariant, implying that $Y$ is marginally $\grp$-invariant. \cref{prop:joint:symmetry} applied to $Y$ and $Z$ shows that $Y,Z$ are jointly $\grp$-invariant as well.  Joint invariance of $X,Y$, along with $X\condind_Y Z$, implies that $X,Y,Z$ are jointly $\grp$-invariant. After marginalizing out $Y$, the joint distribution of $(X,Z)$ is $\grp$-invariant. \cref{prop:joint:symmetry} then shows that $P_{Z|X}$ is conditionally $\grp$-equivariant.

(ii) A similar argument as above shows that $(g\cdot X,g\cdot Y,Z)\equdist (X,Y,Z)$ for each $g\in\grp$. Marginalizing out $Y$, we see that $(g\cdot X,Z)\equdist (X,Z)$, so that $P_{Z|X}$ is conditionally $\grp$-invariant.
\end{proof}

\section{Sufficiency, Adequacy, and Noise Outsourcing}
\label{sec:sufficiency:adequacy}

Functional and probabilistic notions of symmetries  
represent two different approaches to achieving the same goal: a principled framework for constructing models from symmetry considerations. Despite their apparent similarities, the precise mathematical link is not immediately clear. 
The connection consists of two components: \emph{noise outsourcing}, a generic technical tool, allows us to move between a conditional probability distribution $P_{Y|X}$ and a representation of $Y$ as a function of $X$ and random noise; \emph{statistical sufficiency and adequacy} allow us to restrict the noise-outsourced representations to functions of certain statistics of $X$. We review these ideas in this section.

\subsection{Sufficiency and Adequacy}
\label{sec:sufficiency}

Sufficiency is based on the idea that a statistic may contain all information that is needed for an inferential procedure; for example, to completely describe the distribution of a sample or for parameter inference. Adequacy extends that idea to prediction. The ideas go hand-in-hand with notions of symmetry: while invariance describes information that is irrelevant,  
sufficiency and adequacy describe the information that is relevant.  

Sufficiency and adequacy are defined with respect to a probability \emph{model}: a family of distributions indexed by some parameter $\theta\in\Theta$. Throughout, we consider a model for the joint distribution over $X$ and $Y$, $\model_{X,Y}=\{P_{X,Y;\theta} \,:\, \theta\in\Theta\}$, from which there is an induced marginal model $\model_X = \{ P_{X;\theta} \,:\, \theta\in\Theta\}$ and a conditional model $\model_{Y|X} = \{ P_{Y|X;\theta} \,:\, \theta\in\Theta\}$. For convenience, and because the parameter does not play a meaningful role in subsequent developments, we suppress the notational dependence on $\theta$ when there is no chance of confusion.

Sufficiency originates in the work of \citet{Fisher:1922}; it formalizes the notion that a statistic might be used in place of the data for any statistical procedure. It has been generalized in a variety of different directions, including predictive sufficiency \citep{Bahadur:1954,Lauritzen:1974,Fortini:etal:2000} and adequacy \citep{Skibinsky:1967}. 
There are a number of ways to formalize sufficiency, which are equivalent under certain regularity conditions; see \citet{Schervish:1995}. The definition that is most convenient here is due to \citet{Halmos:Savage:1949}: 
there is a \emph{single} Markov kernel that gives the \emph{same} conditional distribution of $X$ conditioned on $S(X)=s$ for \emph{every} distribution $P_X \in \model_X$.
\begin{definition} \label{def:pred:suff}
  Let $\calS$ be a Borel space and $S : \calX \to \calS$ a measurable map. 
  $S$ is a \emph{sufficient statistic for $\model_{X}$} if there is a Markov kernel $q : \borel_{\calX} \times \calS \to \bbR_+$ such that for all $P_X \in\model_X$ and $s\in\calS$, we have $P_X(\argdot \mid S(X)=s) =  q(\argdot, s)$. 
\end{definition}
A canonical example is that of a sequence of $n$ i.i.d.\ coin tosses. If the probability model is the family of Bernoulli distributions with probability of heads equal to $p$, then the number of heads $N_h$ is sufficient: conditioned on $N_h = n_h$, the distribution of the data is uniform on all sequences with $n_h$ heads. Equivalently, the number of heads is also sufficient in estimating $p$. For example, the maximum likelihood estimator is $N_h/n$.  

\Cref{sec:finite:exch:seq} pertains to the more nuanced example of finite exchangeable sequences $\Xn \in\calX^n$. A distribution $P_X$ on $\calX^n$ is \emph{finitely exchangeable} if for all sets $A_1,\dotsc,A_n\in\borel_{\calX}$,
\begin{align} \label{eq:def:finite:exch:seq}
  P(X_1\in A_1,\dotsc,X_n\in A_n) = P(X_{\perm(1)}\in A_1,\dotsc,X_{\perm(n)}\in A_n) \;, \quad \text{for all } \perm\in\Symn \;.
\end{align}
Denote by $\model^{\Symn}_{\Xn}$ the family of all exchangeable distributions on $\calX^n$. 
The de Finetti conditional i.i.d.\ representation \eqref{eq:deFinetti:cond:iid} may fail for a finitely exchangeable sequence \citep[see][for examples]{Diaconis:1977,Diaconis:Freedman:1980}. However, the {empirical measure} plays a central role in both the finitely and infinitely exchangeable cases. 
The empirical measure (or counting measure) of a sequence $\Xn\in\calX^n$ is defined as
\begin{align} \label{eq:empirical:seq}
  \empirical_{\Xn}(\argdot) = \sum_{i=1}^n \delta_{X_i}(\argdot) 
  \;,
\end{align}
where $\delta_{X_i}$ denotes an atom of unit mass at $X_i$. 
The empirical measure discards the information about the order of the elements of $\Xn$, but retains all other information. A standard fact (see \cref{sec:finite:exch:seq}) is that a distribution $P_{\Xn}$ on $\calX^n$ is exchangeable if and only if the conditional distribution $P_{\Xn}(\Xn \mid \empirical_{\Xn}=m)$ is the uniform distribution on all sequences that can be obtained by applying a permutation to $\Xn$. 
This is true for all distributions in $\model^{\Symn}_{\Xn}$; therefore $\empirical_{\Xn}$ is a {sufficient statistic} for $\model^{\Symn}_{\Xn}$ according to \cref{def:pred:suff}, and we may conclude for any probability model $\model_{\Xn}$: 
\begin{center}
  \emph{$\Symn$-invariance of all $P_{\Xn}\in\model_{\Xn}$ is equivalent to the sufficiency of $\empirical_{\Xn}$}.
\end{center}
In this case invariance and sufficiency clearly are two sides of the same coin: a sufficient statistic captures all information that is relevant to a model for $\Xn$; invariance discards the irrelevant information. \Cref{sec:finite:exch:seq} explores this in further detail. 

We note that in this work, there is a clear correspondence between group invariance and a sufficient statistic (see \cref{sec:maximal:invariants:adequate} for details), as in the example of exchangeable sequences. In other situations, there is a sufficient statistic but the set of symmetry transformations may not correspond to a group. See \citet{Freedman:1962,Freedman:1963,Diacnois:Freedman:1987} for examples.

\paragraph{Adequacy} 
The counterpart of sufficiency for modeling the conditional distribution of $Y$ given $X$ is \emph{adequacy} \citep{Skibinsky:1967,Speed:1978}.\footnote{When $Y$ is the prediction of $X_{n+1}$ from $\Xn$, adequacy is also known as \emph{predictive sufficiency} \citep{Fortini:etal:2000} and, under certain conditions, it is equivalent to sufficiency and transitivity \citep{Bahadur:1954} or to total sufficiency \citep{Lauritzen:1974}. See \citet{Lauritzen:1974a} for a precise description of how these concepts are related.} 
The following definition adapts one given by \citet{Lauritzen:1974a}, which is more intuitive than the measure theoretic definition introduced by \citet{Skibinsky:1967}. 
\begin{definition} \label{def:adequate:stat}
  Let $\calS$ be a Borel space, and let $S : \calX \to \calS$ be a measurable map. Then $S$ is an \emph{adequate statistic of $X$ for $Y$ with respect to $\model_{X,Y}$} if
  \begin{enumerate}[label=(\roman*)]
    \item $S$ is sufficient for $\model_X$; and
    \item for all $x\in\calX$ and $P_{X,Y}\in\model_{X,Y}$, 
    \begin{align}
    P_{X,Y}(Y \in \argdot \mid X=x) = P_{X,Y}(Y \in \argdot \mid S = S(x)) \,.
    \label{eq:dseparate}
    \end{align}
  \end{enumerate}
\end{definition}
\Cref{eq:dseparate} amounts to the conditional independence of $Y$ and $X$, given $S(X)$. To see this, note that because $S(X)$ is a measurable function of $X$,
\begin{align*}
  P_{X,Y}(Y \in \argdot \mid X=x) = P_{X,Y}(Y \in \argdot \mid X = x, S = S(x)) = P_{X,Y}(Y \in \argdot \mid S = S(x))\;,
\end{align*}
which is equivalent to $Y \condind_{S(X)} X$. Borrowing terminology from the graphical models literature, we say that $S$ \emph{d-separates} $X$ and $Y$ \citep{Lauritzen:1996}. Therefore, adequacy is equivalent to sufficiency for $\model_X$ and d-separation of $X$ and $Y$, for all distributions in $\model_{X,Y}$.

\subsection{Noise Outsourcing and Conditional Independence}
\label{sec:noise:out:dsep}

Noise outsourcing is a standard technical tool from measure theoretic probability, where it is also known by other names such as \emph{transfer} \citep{Kallenberg:2002}. For any two random variables $X$ and $Y$ taking values in nice spaces (e.g., Borel spaces), noise outsourcing says that there exists a functional representation of samples from the conditional distribution $P_{Y|X}$ in terms of $X$ and independent noise: $Y \equas f(\noise,X)$.  
As noted in \cref{sec:intro}, the relevant property of $\noise$ is its independence from $X$, and the uniform distribution could be replaced by any other random variable taking values in a Borel space, for example a standard normal on $\bbR$, and the result would still hold, albeit with a different $f$. 

Basic noise outsourcing can be refined in the presence of conditional independence. Let $S : \calX \to \calS$ be a statistic such that $Y$ and $X$ are conditionally independent, given $S(X)$: $Y \condind_{S(X)} X$. The following basic result, upon which the results in \cref{sec:connecting:symmetries} rely, says that if there is a statistic $S$ that d-separates $X$ and $Y$, then it is possible to represent $Y$ as a noise-outsourced function of $S$. 

\begin{lemma} 
  \label{lem:noise:out:suff}
  Let $X$ and $Y$ be random variables with joint distribution $P_{X,Y}$. Let $\calS$ be a standard Borel space and $S : \calX \to \calS$ a measurable map. Then $S(X)$ d-separates $X$ and $Y$ if and only if there is a measurable function $f : [0,1] \times \calS \to \calY$ such that
  \begin{align} \label{eq:noise:outsourcing}
    (X,Y) \equas (X,f(\noise ,S(X))) \quad \text{where} \quad \noise\sim\Unif[0,1] \quad \text{and} \quad \noise \condind X \;.
  \end{align}
  In particular, $Y = f(\noise, S(X))$ has distribution $P_{Y|X}$.
\end{lemma}
The proof is a straightforward application of a more general result given in \cref{sec:proof:noise:out:suff}. 
Note that in general, $f$ is measurable but need not be differentiable or otherwise have desirable properties, although for modeling purposes it can be limited to functions belonging to a tractable class (e.g., differentiable, parameterized by a neural network).  
Note also that the identity map $S(X)=X$ trivially d-separates $X$ and $Y$, so that $Y\equas f(\eta,X)$, which is standard noise outsourcing \citep[e.g.,][Lem. 3.1]{Austin:2015}.

To make the connections between ideas clear, we state the following corollary of \cref{lem:noise:out:suff}, which is proved by checking the definition of adequacy. 
\begin{corollary} \label{cor:adeq:rep}
  Let $S: \calX \to \calS$ be a sufficient statistic for the model $\model_X$. Then $S$ is an adequate statistic for $\model_{X,Y}$ if and only if, for each $P_{X,Y}\in\model_{X,Y}$, there exists a corresponding measurable function $f : [0,1] \times \calS \to \calY$ such that \eqref{eq:noise:outsourcing} holds.
\end{corollary}

\section{Functional Representations of Probabilistic Symmetries}
\label{sec:connecting:symmetries}

\Cref{lem:noise:out:suff} shows that if it is possible to find conditions on $P_{X,Y}$ such that there is a $\grp$-invariant d-separating statistic $S$, then there exists a $\grp$-invariant functional representation of $Y$ that depends on $X$ only through $S(X)$. Furthermore, families of probability measures (i.e., statistical models) consisting of such distributions correspond to families of functions with the desired invariance property. 
The main results of this section establish necessary and sufficient conditions on $P_{X,Y}$ to guarantee $\grp$-invariant (\cref{thm:group:invariant:rep}) and $\grp$-equivariant (\cref{thm:kall:equiv}) functional representations of $Y$; they also establish generic properties of the associated d-separating statistics. Together, they form a general program for identifying function classes and neural network architectures corresponding to symmetry conditions.

The results of this section imply that the two approaches in \cref{sec:two:types} are equivalent up to stochasticity (which is not trivial; see \cref{sec:practical:advantages}): symmetry can be incorporated into an explicit probabilistic model by considering invariant and equivariant distributions, or it can be incorporated into an implicit model by considering invariant and equivariant stochastic functions. The deterministic functional models in \cref{sec:background:networks} are special cases of the latter: for every invariant noise-outsourced function $f$ such that $Y = f(\noise,X)$, we have that  
\begin{align*}
  \bbE[ Y \mid M(X)] = \int_{[0,1]} f(\noise,X)\; d\noise = f'(X) \;,
\end{align*} 
such that $f'(X)$ is also invariant (and similarly for equivariant functions).

\subsection{Invariant Conditional Distributions}
 \label{sec:maximal:invariants}

To state the first main result, on the function representation of invariant conditional distributions, some definitions are required. 
For a group $\grp$ acting on a set $\calX$, the \emph{orbit} of any $x\in\calX$ is the set of elements in $\calX$ that can be generated by applying the elements of $\grp$. It is denoted $\grp\cdot x = \{ g\cdot x ; g\in\grp \}$. The \emph{stabilizer}, or isotropy subgroup, of $x\in\calX$ is the subgroup of $\grp$ that leaves $x$ unchanged: $\grp_x = \{ g\in \grp ; g\cdot x = x  \}$. An \emph{invariant statistic} $S : \calX \to \calS$ is a measurable map that satisfies $S(x) = S(g\cdot x)$ for all $g\in\grp$ and $x\in\calX$. A \emph{maximal invariant statistic}, or maximal invariant, is an invariant statistic $M : \calX \to \calS$ such that $M(x_1) = M(x_2)$ implies $x_2 = g\cdot x_1$ for some $g\in\grp$; equivalently, $M$ takes a different constant value on each orbit. The orbits partition $\calX$ into equivalence classes, on each of which $M$ takes a different value.

By definition, an invariant distribution $P_X$ is constant on any particular orbit. Consider the conditional distribution $P_X(X \mid M(X)=m):= P_{X|m}(X)$. Conditioning on the maximal invariant taking a particular value is equivalent to conditioning on $X$ being in a particular orbit; for invariant $P_X$, $P_{X|m}$ is zero outside the orbit on which $M(X)=m$, and ``uniform'' on the orbit, modulo fixed points (see \cref{fn:orbit:uniform} in \cref{sec:technical}). Furthermore, for any $Y$ such that $(g\cdot X,Y)\equdist (X,Y)$ for all $g\in\grp$, $M(X)$ contains all relevant information for predicting $Y$ from $X$; that is, $Y \condind_{M(X)} X$. \Cref{fig:thm:group:invariant} illustrates the structure. 
These high-level ideas, which are made rigorous in \cref{sec:technical}, lead to the following functional representation of invariant conditional distributions.

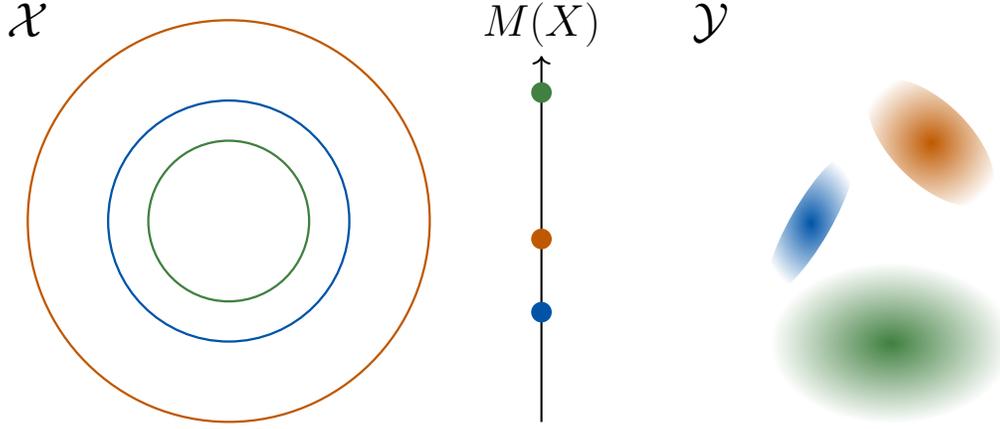
\begin{figure}[t]
  \begin{center}
    \begin{tabular}{c c c c}
        \resizebox{!}{0.25\textheight}{
          \begin{tikzpicture}
            \draw[thick,color=orangeClj] (0,0) circle (2.5cm);
            \draw[thick,color=blueClj] (0,0) circle (1.5cm);
            \draw[thick,color=darkgreenClj] (0,0) circle (1cm);
            \node at (-2.5,2.5) {\LARGE{$\calX$}};
          \end{tikzpicture}
        } & 
        \resizebox{!}{0.25\textheight}{
          \begin{tikzpicture}
            \draw[thick,->] (0,-2.5) -- (0,2.5) node[anchor=south] {\LARGE{$M(X)$}};
            \fill[color=orangeClj] (0,0) circle (4pt);
            \fill[color=blueClj] (0,-1) circle (4pt);
            \fill[color=darkgreenClj] (0,2) circle (4pt);
          \end{tikzpicture}
        }
        & &
        \resizebox{!}{0.25\textheight}{
          \begin{tikzpicture}
            \shade[inner color=orangeClj, outer color=white, rotate around={45:(0.5,1)}] (0.5,1) ellipse (0.5cm and 1cm);
            \shade[inner color=blueClj, outer color=white, rotate around={-120:(-1,0)}] (-1,0) ellipse (1cm and 0.25cm);
            \shade[inner color=darkgreenClj, outer color=white] (0,-1.5) ellipse (1.5cm and 1cm);
            \node at (-2.25,2.5) {\LARGE{$\calY$}};
          \end{tikzpicture}
        }
    \end{tabular}
  \end{center}
  \caption{
    An illustration of structure of \cref{thm:group:invariant:rep}. The maximal invariant $M$ (middle) acts as an index of the orbits of $\calX$ under $\grp$ (left), and mediates all dependence of $Y$ on $X$, i.e., $P_{Y|X} = P_{Y|M(X)}$ (right). (Best viewed in color.)
  }
  \label{fig:thm:group:invariant}
\end{figure}

\begin{restatable}{theorem}{groupinvariantrep} 
  \label{thm:group:invariant:rep}
  Let $X$ and $Y$ be random elements of Borel spaces $\calX$ and $\calY$, respectively, and $\grp$ a compact group acting measurably on $\calX$. Assume that $P_X$ is $\grp$-invariant, and pick a maximal invariant $M : \calX \to \calS$, with $\calS$ another Borel space. Then $P_{Y|X}$ is $\grp$-invariant if and only if there exists a measurable function $f : [0,1] \times \calS \to \calY$ such that
  \begin{align} \label{eq:max:invariant:rep}
    (X, Y) \equas \big(X, f(\noise, M(X)) \big) \quad \text{with } \noise\sim \Unif[0,1] \text{ and } \noise\condind X \;.
  \end{align}
\end{restatable}

The proof is given in \cref{sec:technical} after the necessary intermediate technical results have been developed. Note that because $f$ is a function of a maximal invariant, it is $\grp$-invariant: $f(\noise, M(g\cdot X)) = f(\noise,M(X))$, almost surely. 
A maximal invariant always exists for sufficiently nice $\grp$ and $\calX$ \citep{Hall:Wijsman:Ghosh}; for example, define $M : \calX \to \bbR$ to be a function that takes a unique value on each orbit. Therefore, $P_{X|M}$ can always be defined in such cases. The assumptions made in the present work, that $\grp$ acts measurably on the Borel space $\calX$, allow for the existence of maximal invariants, and in many settings of interest a maximal invariant is straightforward to construct. For example, \cref{sec:finite:exch:seq,sec:finite:exch:arrays} rely on constructing maximal invariants for applications of the results of this section to specific exchangeable structures. 

Versions of \cref{thm:group:invariant:rep} may hold for non-compact groups and more general topological spaces, but require considerably more technical details. At a high level, $\grp$-invariant measures on $\calX$ may be disintegrated into product measures on suitable spaces $\calS \times \calZ$ under fairly general conditions, though extra care is needed in order to ensure that such disintegrations exist. See, for example, \citet{Andersson:1982,Eaton:1989,WIjsman:1990,Schindler:2003}; and especially \citet{Kallenberg:2017}.

\subsection{Equivariant Conditional Distributions} \label{sec:equivariant:representations}

According to \cref{thm:group:invariant:rep}, \emph{any} maximal invariant can be used to establish a functional representation of $\grp$-invariant conditional distributions. 
If a particular type of equivariant function exists and is measurable, then it can be used to establish a functional representation of equivariant conditional distributions. 
Let $\tau : \calX \to \grp$ be an equivariant function, 
\begin{align} \label{eq:tau:def}
  \tau(g\cdot x) = g\cdot \tau(x) \;, \quad \text{for all } g\in\grp, \; x\in\calX \;.
\end{align}
For ease of notation, let $\tau_x := \tau(x)$, and denote by $\tau_x^{-1}$ the inverse of $\tau_x$ in $\grp$, such that $\tau_x^{-1} \cdot \tau_x = e$.  
$\tau$ has some remarkable properties that make it suitable for constructing equivariant functions.

\begin{restatable}{lemma}{tauproperties}
  \label{prop:tau:properties}
  For a group $\grp$ acting measurably on Borel spaces $\calX$ and $\calY$, a representative equivariant $\tau : \calX \to \grp$, as defined in \eqref{eq:tau:def}, has the following properties:
  \begin{enumerate}[label=(\roman*)]
    \item The function $M_{\tau} : \calX \to \calX$ defined by $M_{\tau}(x) = \tau_x^{-1} \cdot x$ is a maximal invariant.
    \item For any mapping $b : \calX \to \calY$, the function  
    \begin{align} \label{eq:tau:equiv}
      f(x) = \tau_x \cdot b(\tau_x^{-1} \cdot x), \quad x\in\calX,
    \end{align}  
    is $\grp$-equivariant: $f(g\cdot x) = g\cdot f(x)$, $g \in \grp$.
  \end{enumerate}
\end{restatable}

We call $\tau$ a \emph{representative equivariant} of its use in constructing the maximal invariant $M_{\tau}(x):= \tau_x^{-1}\cdot x$, which is a representative element from each orbit in $\calX$. (We give examples in \cref{sec:finite:exch:seq,sec:finite:exch:arrays}.)  
The properties in \cref{prop:tau:properties} are used to establish the equivariant counterpart of \cref{thm:group:invariant:rep}. In essence, for $\grp$-invariant $P_X$, $P_{Y|X}$ is conditionally $\grp$-equivariant if and only if there exists a noise-outsourced function such that
\begin{align} \label{eq:kall:equiv}
  g\cdot Y \equas f(\noise,g\cdot X) \;.  
\end{align}
Observe that this is equivalent to $f$ being $\grp$-equivariant, as defined in \eqref{eq:equivariant:function:def}, in the second argument: $Y = e\cdot Y = f(\noise,e\cdot X) = f(\noise,X)$ and therefore
\begin{align*} 
  f(\noise,g \cdot X) = g \cdot Y = g \cdot f(\noise,X) \;, \quad \text{a.s., for each } g\in \grp \;.
\end{align*}
It is straightforward to show that for $X$ and $Y$ constructed in this way, the resulting distributions have the desired equivariance property. The existence of an equivariant functional representation is harder to show, and the representative equivariant $\tau$ plays a key role. We elaborate after stating the result.

\begin{restatable}{theorem}{groupequivariantrep}
  \label{thm:kall:equiv}
  Let $\grp$ be a compact group acting measurably on Borel spaces $\calX$ and $\calY$, such that there exists a measurable representative equivariant $\tau : \calX \to \grp$ satisfying \eqref{eq:tau:def}.  Suppose $P_X$ is $\grp$-invariant. Then $P_{Y|X}$ is $\grp$-equivariant if and only if there exists a measurable $\grp$-equivariant function $f : [0,1] \times \calX \to \calY$ satisfying \eqref{eq:kall:equiv} such that
  \begin{align*}
    (X,Y) \equas (X, f(\noise,X)) \quad \text{for each } g\in \grp , \text{ with } \noise\sim\Unif[0,1] \text{ and } \noise\condind X \;.
  \end{align*}
\end{restatable}

\begin{figure}
  \begin{center}
    \begin{tabular}{c c c}
      \resizebox{!}{0.25\textheight}{
        \begin{tikzpicture}
          \draw[thick,color=orangeClj] (0,0) circle (2.5cm);
          \draw[thick,color=blueClj] (0,0) circle (1.5cm);
          \draw[thick,color=darkgreenClj] (0,0) circle (1cm);

          \fill[color=orangeClj] (-2.5,0) circle (4pt);
          \node[color=orangeClj] at (0,2.5) {\Huge $*$};

          \fill[color=blueClj] (1.225,0.866) circle (4pt);
          \node[color=blueClj] at (0,1.5) {\Huge $*$};

          \fill[color=darkgreenClj] (0.548,-0.837) circle (4pt);
          \node[color=darkgreenClj] at (0,1) {\Huge $*$};

          \draw[thick,color=orangeClj,->] (0,0) ++(180:2.22) arc (180:90:2.2);
          \draw[thick,color=blueClj,->] (0,0) ++(35.25:1.8) arc (35.25:90:1.8);
          \draw[thick,color=darkgreenClj,->] (0,0) ++(-56.77:0.7) arc (-56.77:90:0.7);

          \node at (-2.75,2.75) {\LARGE{$\calX$}};
        \end{tikzpicture}
      } &
      \resizebox{!}{0.25\textheight}{
        \begin{tikzpicture}[rotate=90, transform shape]
          \begin{axis}[
          no markers, axis y line=none, axis x line*=bottom, 
          xtick={-5,-0.25,3.5}, ytick=\empty, x tick style={draw=none},
          xticklabels={\textcolor{cyanClj}{\Huge $*$},\textcolor{purpleClj}{\Huge $*$},\textcolor{redClj}{\Huge $*$}},
          domain=-10:10, samples=100,
          height=3cm,width=7cm
          ]
            \addplot[very thick,blueClj] {gauss(3,1.5)};
            \addplot[very thick,orangeClj] {gauss(-3,1)};
            \addplot[very thick,darkgreenClj] {gauss(1,0.85)};
          \end{axis}
        \end{tikzpicture}
      } &
      \resizebox{!}{0.25\textheight}{
        \begin{tikzpicture}
          \draw[thick,color=cyanClj] (1,1.5) circle (0.5cm);
          \draw[thick,color=purpleClj] (-2,1) circle (1.414cm);
          \draw[thick,color=redClj] (0.5,-0.5) circle (1.118cm);

          \node[color=cyanClj] at (1,2) {\Huge $*$};
          \fill[color=cyanClj] (0.5,1.5) circle (4pt);

          \node[color=purpleClj] at (-2,2.414) {\Huge $*$};
          \fill[color=purpleClj] (-0.845,1.816) circle (4pt);

          \fill[color=redClj] (1.112,-1.435) circle (4pt);
          \node[color=redClj] at (0.5,0.618) {\Huge $*$};

          \draw[thick,color=cyanClj,<-] (1,1.5) ++(180:0.8) arc (180:90:0.8);
          \draw[thick,color=purpleClj,<-] (-2,1) ++(35.25:1.714) arc (35.25:90:1.714);
          \draw[thick,color=redClj,<-] (0.5,-0.5) ++(-56.77:0.818) arc (-56.77:90:0.818);

          \node at (-2.75,2.75) {\Large{$\calY$}};
        \end{tikzpicture}
      }
    \end{tabular}
  \end{center}
  \caption{
    Illustration of the structure of \cref{thm:kall:equiv}. Each orbit of $\calX$ (left) induces a distribution over the orbits of $\calY$ (middle). Given $X$, a sample $Y$ is generated by: 1) obtaining the orbit representative $\tau_X^{-1}\cdot X$; 2) sampling an orbit representative in $\calY$, conditioned on $\tau_X^{-1}\cdot X$; and 3) applying $\tau_X$ to the orbit representative in $\calY$. (Best viewed in color.)
  }
  \label{fig:thm:kall:equiv}
\end{figure}
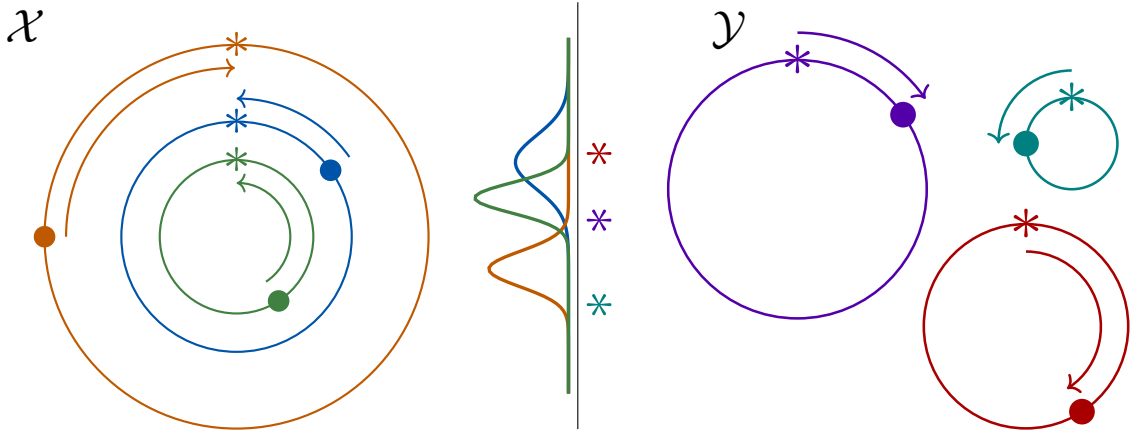

The proof of \cref{thm:kall:equiv} adapts an argument due to \citet[][Lem.~7.11]{Kallenberg:2005}. 
The proof, given in full in \cref{sec:proof:kall:equiv}, is rather technical but we sketch the basic ideas here. $\grp$-invariance of $P_X$ and $\grp$-equivariance of $P_{Y|X}$ is equivalent to the joint invariance $(g\cdot X, g\cdot Y) \equdist (X,Y)$, for all $g\in\grp$. As such, we require a frame of reference relative to which $X$ and $Y$ will equivary. The representative equivariant $\tau$ plays this role. Recall that $\tau$ is used to define a maximal invariant, $M_{\tau}(x):= \tau_x^{-1}\cdot x$, which maps $x$ to the representative element of its orbit. The proof of \cref{thm:kall:equiv} relies on establishing the conditional independence relationship $\tau_X^{-1} \cdot Y \condind_{M_{\tau}(X)} X$. Noise outsourcing (\cref{lem:noise:out:suff}) then implies that $\tau_X^{-1}\cdot Y = b(\noise, M_{\tau}(X))$ for some $b : [0,1] \times \calX \to \calY$, and hence
\begin{align} \label{eq:equivariant:construction}
  Y = f(\noise,X) := \tau_X \cdot b(\noise, M_{\tau}(X)) = \tau_X \cdot b(\noise, \tau_X^{-1}\cdot X) \;.
\end{align} 
Finally, \cref{prop:tau:properties} (see \cref{sec:technical}) establishes that $f$ so defined is equivariant in the second argument: $f(\noise, g\cdot X) = g\cdot f(\noise,X)$ for all $g\in\grp$. 

Observe that if the action of $\grp$ on $\calY$ is trivial, $\tau_X^{-1}\cdot Y \condind_{M_{\tau}(X)} X$ reduces to $Y \condind_{M_{\tau}(X)} X$, which is precisely the relationship needed to establish the invariant representation in \cref{thm:group:invariant:rep}. However, in the invariant case, \emph{any} maximal invariant will d-separate $X$ and $Y$, and therefore \cref{thm:group:invariant:rep} is more general than simply applying \cref{thm:kall:equiv} to the case with trivial group action on $\calY$. 
The additional assumptions in \cref{thm:kall:equiv} account for the non-trivial action of $\grp$ on $\calY$, which requires setting up a fixed frame of reference through $\tau$ and the orbit representatives $M_{\tau}(X)$. 

\Cref{thm:kall:equiv} may hold for non-compact groups, but our proof for compact groups makes use of the normalized Haar measure; extending to non-compact groups requires additional technical overhead.

\subsection{Symmetry-induced Adequate Statistics}
\label{sec:maximal:invariants:adequate}

The statistics upon which the functional representations in \cref{thm:group:invariant:rep,thm:kall:equiv} rely capture all of the relevant structural information from $X$, and discard the rest. 
The utility of this orbit-resolving property can be characterized statistically. For suitably defined models---families of distributions satisfying the required invariance conditions---a maximal invariant is a sufficient statistic for the marginal model $\model_X$, and an adequate statistic for the joint model $\model_{X,Y}$. Similarly, for models satisfying the equivariance conditions, a representative equivariant is an adequate statistic.  

To make this precise, denote by $\model_{X,Y}^{\inv}$ the family of probability measures on $\calX\times\calY$ that satisfy $(g\cdot X, Y) \equdist (X,Y)$, for all $g\in\grp$.  
Similarly, denote by $\model_{X,Y}^{\equ}$ the family of probability measures on $\calX\times\calY$ that satisfy $(g\cdot X, g\cdot Y) \equdist (X,Y)$, for all $g \in \grp$. Let $\model_{X,\tau_X^{-1}\cdot Y}$ denote the family of probability measures obtained from $\model_{X,Y}^{\equ}$ via $P_{X,Y} \mapsto P_X P_{Y|X}\circ \tau_X^{-1}$.

\begin{restatable}{theorem}{sufficientmaximalinvariant} 
  \label{thm:suff:maximal:invariant}
  Let $\grp$ be a compact group acting measurably on standard Borel spaces $\calX$ and $\calY$, and let $\calS$ be another Borel space. Then:  
  \begin{enumerate}[label=(\roman*)]
    \item Any maximal invariant $M : \calX \to \calS$ on $\calX$ under $\grp$ is an adequate statistic of $X$ for $Y$ with respect to any model $\model_{X,Y} \subseteq \model_{X,Y}^{\inv}$. In particular, to each $P_{X,Y}\in \model_{X,Y}^{\inv}$ corresponds a measurable $\grp$-invariant function $f : [0,1] \times \calS \to \calY$ as in \eqref{eq:max:invariant:rep}. 
    \item $M_{\tau}$ as in \cref{prop:tau:properties} is an adequate statistic of $X$ for $\tau_X^{-1}\cdot Y$ with respect to any model $\model_{X,Y} \subseteq \model_{X,\tau_X^{-1}\cdot Y}$, and to each $P_{X,Y} \in \model_{X,\tau_X^{-1}\cdot Y}$ there corresponds a measurable function $b : [0,1] \times \calX \to \calY$ as in \eqref{eq:equivariant:construction} such that $\tau_X\cdot b(\eta,\tau_X^{-1}\cdot X)$ is $\grp$-equivariant.
  \end{enumerate} 
\end{restatable}

\Cref{thm:suff:maximal:invariant} puts explicit probability models in correspondence with implicit functional models. 
It is a consequence of \cref{lem:maximal:invariant} in \cref{sec:technical}, which establishes the necessary structural properties of all invariant distributions. 
To illustrate, consider again the example of an exchangeable $\calX$-valued sequence.

\begin{example}[Adequate statistics of an exchangeable sequence] \label{ex:max:invariant:exch:seq}
 Let $\Xn$ be an \\exchangeable $\calX$-valued sequence. $\bbS_n$ acts on any point $\xn\in\calX^n$ by permuting its indices, which defines an action on $\calX^n$; the orbit of $\xn$ is the set of sequences that can be obtained from $\xn$ by applying a permutation. The empirical measure is a well-known sufficient statistic for $\model^{\bbS_n}_{\Xn}$; it is easy to see that the empirical measure is also a maximal invariant (see \cref{sec:invariant:seq}). If $\calX=\bbR$ (or any other set with a total order), then so too is the vector of order statistics, $\xn^{\uparrow}=(x_{(1)},\dotsc,x_{(n)})$, with $x_{(1)} \leq \dotsb \leq x_{(n)}$. A representative equivariant may be obtained by defining $\tau_{\xn}^{-1}$ as any permutation $\perm$ (not necessarily unique) that satisfies $\perm \cdot \xn = \xn^{\uparrow}$. 
\end{example}

\paragraph{Exchangeable random structures}  
\Cref{sec:finite:exch:seq,sec:finite:exch:arrays} work out the details of the program outlined in \cref{sec:general:program} to various exchangeable random structures, and relates the resulting representations to some recent \citep{Zaheer:etal:2017,Gilmer:etal:2017,Hartford:etal:2018} and less recent \citep{ShaweTaylor:1989} work in the neural networks literature. Exchangeability is distributional invariance under the action of $\Symn$ (or other groups defined by composing $\Symn$ in different ways for other data structures). \Cref{ex:max:invariant:exch:seq} states that the empirical measure is a maximal invariant of $\Symn$ acting on $\calX^n$; suitably defined generalizations of the empirical measure are maximal invariants for other exchangeable structures. With these maximal invariants, obtaining a functional representation of an invariant conditional distribution is straightforward.

With an additional conditional independence assumption that is satisfied by most neural network architectures, \cref{thm:kall:equiv} can be refined to obtain a detailed functional representation of the relationship between $\Symn$-equivariant random variables (\cref{thm:finite:exch:seq:equivariant}). That refinement relies on the particular subgroup structure of $\Symn$, and it raises the question, not pursued here, of whether, and under what conditions, other groups may yield similar refinements.

\section{Practical Implications and Considerations}
\label{sec:practical}

In order to make the previous sections' theory more useful to practitioners, we formulate a general program for model specification, briefly discuss computational considerations, and interpret some aspects of the theory in the context of the recent literature.

\subsection{A Program for Obtaining Symmetric Functional Representations}
\label{sec:general:program}

The previous section suggests that for a model consisting of $\grp$-invariant distributions, if an adequate statistic can be found then all distributions in the conditional model, $P_{Y|X} \in \model_{Y|X}$ have a noise-outsourced functional representation in terms of the adequate statistic, as in \eqref{eq:noise:outsourcing}. Furthermore, \cref{thm:suff:maximal:invariant} says that any maximal invariant is an adequate statistic under such a model. Therefore, for a specified compact group $\grp$ acting measurably on $\calX$ and $\calY$, a program for functional model specification through distributional symmetry is as follows:
\begin{mdframed}
\textsc{Program for $\grp$-equivariant and -invariant model specification.}
  \begin{enumerate}[label=(\roman*)]
    \item Determine a maximal $\grp$-invariant statistic $M : \calX \to \calS$. 
    \item Fix a representative equivariant under $\grp$, $\tau : \calX \to\grp$, that can be used to map any element of $\calX$ to its orbit's representative.
    \item Specify the model with a function class $\calF_{\tau,M}$, consisting of compositions of noise-outsourced equivariant functions and a final noise-outsourced invariant function. 
  \end{enumerate}
\end{mdframed}

\subsection{Computing Maximal Invariants and Representative Equivariants}
\label{sec:computing:max:inv}

In theory, a maximal invariant can be computed by specifying a representative element for each orbit. In practice, this can always be done when $\grp$ is discrete because the relevant elements can be enumerated and reduced to permutation operations. Several systems for computational discrete mathematics, such as Mathematica \citep{Mathematica} and GAP \citep{GAP4}, have built-in functions to do so. For continuous groups it may not be clear how to compute a maximal invariant, and it may be impractical. 
Furthermore, maximal invariants are not unique, and some may be better suited to a particular application than others. As such, they are best handled on a case-by-case basis, depending on the problem at hand. Some examples from the classical statistics literature are reviewed in \citet[][Ch.~6]{Lehmann:Romano:2005} and \citet[][Ch.~2]{Eaton:1989}; \citet[][Ch.~7]{Kallenberg:2017} presents a generic method based on so-called projection maps. \Cref{sec:finite:exch:seq,sec:finite:exch:arrays} apply the theory of this section to exchangeable structures by explicitly constructing maximal invariants.

Likewise, a representative equivariant $\tau$, as used in \cref{thm:kall:equiv}, always exists for a discrete group acting on a Borel space $\calX$, a proof of which is given in \citet[][Lem. 7.10]{Kallenberg:2005}. If $\grp$ is compact and acts measurably on Borel spaces $\calX$ and $\calY$, then a representative equivariant exists; see \citet[][Remark 2.46]{Schindler:2003}. In more general settings, $\tau$ with the desired properties may not be well-defined or it may not be measurable, in which case a $\grp$-invariant probability kernel (called an \emph{inversion kernel}) may by constructed from a maximal invariant to obtain similar representations; see \citet[][Ch.~7]{Kallenberg:2017}. Whether a representative equivariant is readily computed is a different matter. If, as in \cref{ex:max:invariant:exch:seq}, a maximal invariant corresponding to a representative element of each orbit is computed, it corresponds to $M_{\tau}$ and the representative equivariant is computed as a by-product. However, even this can be difficult. For example, in the case of graphs (\cref{sec:finite:exch:arrays}), finding a maximal invariant statistic is equivalent to finding a complete graph invariant \citep[e.g.,][]{Lovasz:2012}; finding a representative equivariant is equivalent to graph canonization, which is at least as computationally hard as the graph isomorphism problem.

\subsection{On the Advantages of Stochasticity} \label{sec:practical:advantages}

Much of deep learning (and machine learning in general) is conducted in the framework of learning a deterministic function that does not have stochastic component apart from the data. Although that approach has been successful in a wide variety of settings, and the methods are often more accessible to users without training in probability, there are some fundamental limitations to the approach. One body of work has argued that injecting noise into a neural network during training has beneficial effects on training \citep[e.g.,][]{Srivastava:etal:2014,Wager:etal:2014}.

We focus here on a different aspect. Let $X$ and $Y\in\bbR$ be the input and output data, respectively, and let $Y'=f(X)$, for some $f : \calX \to \bbR$. A basic result about conditional expectation \citep[see, e.g.,][Thm.\ IV.1.17]{Cinlar:2011} says that an $\bbR$-valued random variable $Y'$ is a deterministic (measurable) function of $X$ if and only if $Y'$ corresponds to the conditional expectation of some (typically not unique) random variable (say $Z$) given $X$. That is,
\begin{align*}
  Y' = f(X) \quad \iff \quad Y' = \bbE[Z \mid X] \;.
\end{align*}
Much of machine learning relies on (approximately) minimizing a risk $\bbE[\ell(f(X),Y)]$ with respect to $f$. 
The best such a procedure can do is to learn the conditional expectation $\bbE[Y \mid X]$; the model does not allow $Y$ to exhibit any randomness apart from that in $X$. Except for certain cases (see below), important information may be lost.  

\paragraph{Energy-based models} The extension from conditional expectation $\bbE[Y \mid X]$ to conditional distribution $P_{Y|X}$ is typically not unique. One principled way of bridging the gap is via maximum entropy, which stipulates that $P_{Y|X}$ should be the distribution with the maximum entropy, $H(P_{Y|X})$, subject to a constraint on the conditional expectation, e.g., $\bbE[Y \mid X] = s$. The result is a distribution with density
\begin{align*}
  p_{Y|X}(y) \propto \exp(-\beta y) = \argmin_{p_{Y|X}} \mathcal{L}(p_{Y|X},\beta) = -H(p_{Y|X}) + \beta (\bbE_{p_{Y|X}}[Y \mid X] - s) \;,
\end{align*}
that minimizes the free energy $\mathcal{L}(p_{Y|X},\beta)$. This approach, which has a long history in statistical physics (macrocanonical models) and classical statistics (exponential families) \citep{Jaynes:1957}, was the starting point for work on invariant wavelet-based models \citep{Bruna:Mallat:2013,Bruna:Mallat:2018,Gao:etal:2019}, which can be described in terms of invariant sufficient statistics. More broadly, this approach fits in the framework of energy-based learning with Gibbs distributions \citep{LeCun:etal:2006,Grathwohl:etal:2020}, which subsumes much of the standard machine learning practice of training deterministic functions with (regularized) empirical risk minimization. Despite the long history and easy of interpretation of energy-based models, more sophisticated approaches that incorporate randomness in the network itself seem attractive, for the reasons detailed above.

\paragraph{Simpler proofs and clearer structure} From a theoretical perspective, the probabilistic approach often yields simpler proofs of general results than their deterministic counterparts. Although this assessment can be somewhat subjective, this paper is an example: most of our results rely on establishing conditional independence and appealing to \cref{lem:noise:out:suff}. This includes \cref{thm:finite:exch:seq:invariant}, which characterizes the representation of permutation-invariant distributions. Its the proof of its deterministic counterpart by \citet{Zaheer:etal:2017}, relies on rather lengthy arguments about symmetric polynomials. Furthermore, the structure driving the representation becomes, in our opinion, much clearer from the perspective of sufficiency.

\subsection{Choosing a Function Class}
\label{sec:practical:choosing}

\Cref{thm:group:invariant:rep,thm:kall:equiv} are very general existence results about the representations of symmetric distributions, and in that sense they are related to the recent literature on \emph{universal approximations} \citep[as in][]{Cybenko:1989,Hornik:1991} of invariant and equivariant functions \citep[e.g.,][]{Yarotsky:2018,Maron:etal:2019,Keriven:Payre:2019,Segol:Lipman:2020,Ravanbakhsh:2020}. Those results rely on approximating to arbitrary precision $\grp$-symmetric polynomials. The most practically useful work considers the universal approximation problem in terms of composing linear maps with pointwise activations \citep{Maron:etal:2019,Segol:Lipman:2020,Ravanbakhsh:2020}. For finite groups, that approach also gives upper and lower bounds on the degree of the polynomial required. Despite being conceptually related to that work, \cref{thm:group:invariant:rep,thm:kall:equiv} are \emph{exact} representation results for $\grp$-invariant and -equivariant distributions. Universal distributional approximation is an open question.  

However, like the universal approximation literature, \cref{thm:group:invariant:rep,thm:kall:equiv} are mostly silent about the practical matter of \emph{which} class of functions should be used for a model. Detailed analysis of this type of problem, which is beyond the scope of this paper, is an active area of research requiring detailed analysis of the trade-off between model complexity, computation, and problem-specific requirements. \Citet{Bietti:Mairal:2019} takes some initial steps in this direction.

\paragraph{Convolutions and beyond} The most widely used equivariant layers correspond to convolutions (with respect to $\grp$) with an equivariant convolution kernel. Recent work shows that they are the \emph{only} equivariant linear maps \citep{Kondor:Trivedi:2018,Cohen:etal:2019}. The details of this program have been implemented for a growing number of groups; see \citet[][Appendix D]{Cohen:etal:2019} for a comprehensive list. Modifications of the basic convolutional filters have appeared in the form of convolutional capsules \citep[e.g.,][]{Sabour:etal:2017,Hinton:etal:2018,Lenssen:etal:2018}, and convolution was recently combined with attention mechanisms \citep{Bello:etal:2019,Romero:etal:2020}. 
The representation result in \cref{thm:kall:equiv} says nothing about linearity; non-convolutional equivariant architectures based on self-attention \citep{Parmar:etal:2019} have recently appeared.

\section{Learning from Finitely Exchangeable Sequences}
\label{sec:finite:exch:seq}

In this section,\footnote{The main results of this section appeared in an extended abstract \citep{Bloem-Reddy:Teh:2018:short}.} the program described in \cref{sec:general:program} is fully developed for the case where the conditional distribution of $Y$ given $X$ has invariance or equivariance properties with respect to $\Symn$. Deterministic examples have appeared in the neural networks literature \citep{ShaweTaylor:1989,Zaheer:etal:2017,Ravanbakhsh:etal:2017,Murphy:etal:2019}, and the theory developed in here establishes the necessary and sufficient functional forms for permutation invariance and equivariance, in both the stochastic and deterministic cases.

Throughout this section, the input $X=\Xn$ is a sequence of length $n$, and $Y$ is an output variable, whose conditional distribution given $\Xn$ is to be modeled.\footnote{Note that $\Xn$ may represent a set (i.e., there are no repeated values) or a multi-set (there may be repeated values). It depends entirely on $P_{\Xn}$: if $P_{\Xn}$ places all of its probability mass on sequences $x_n\in\calX^n$ that do not have repeated values, then $\Xn$ represents a set almost surely. Otherwise, $\Xn$ represents a multi-set. The results of this section hold in either case.} 
Recall from \cref{sec:background:prob:symm} that $P_{Y|\Xn}$ is $\Symn$-invariant if $Y|\Xn \equdist Y|\perm \cdot \Xn$ for each permutation $\perm\in\Symn$ of the input sequence.  Alternatively, if $Y=\Yn$ is also a sequence of length $n$,\footnote{In general, $\bfY$ need not be of length $n$, but the results are much simpler when it is; see \cref{sec:seq:partial:exch}.} then we say that $\Yn$ given $\Xn$ is $\Symn$-equivariant if $\Yn|\Xn \equdist \perm\cdot \Yn|\perm\cdot \Xn$ for all $\perm\in\Symn$.  In both cases these symmetry properties stem from the assumption that the ordering in the input sequence $\Xn$ does not matter; that is, the distribution of $\Xn$ is finitely exchangeable: $\Xn\equdist \perm\cdot \Xn$ for each $\perm\in\Symn$. 
Recall that $P_{\Xn}$ and $P_{Y|\Xn}$ denote the marginal and conditional distributions respectively, $\model^{\Symn}_{\Xn}$ is the family of distributions on $\calX^n$ that are $\Symn$-invariant, and $\model^{\Symn}_{Y|\Xn}$ is the family of conditional distributions on $\calY$ given $\Xn$ that are $\Symn$-invariant.

The primary conceptual matter is the central role of the empirical measure $\empirical_{\Xn}$. 
Exchangeable sequences have been studied in great detail, and the sufficiency of the empirical measure for $\model_{\Xn}^{\Symn}$ is well-known (e.g., \citet{Diaconis:Freedman:1980}; \citet[][Prop.\ 1.8]{Kallenberg:2005}). It is also straightforward to show the adequacy of the empirical measure for $\model_{\Xn,Y}^{\Symn}$ using methods that are not explicitly group theoretic. Alternatively, it is enough to show that the empirical measure is a maximal invariant of $\calX^n$ under $\Symn$ and then apply \cref{thm:suff:maximal:invariant}. In either case, the results of the previous sections imply a noise-outsourced functional representation of $\Symn$-invariant conditional distributions (\cref{sec:invariant:seq}). The previous sections also imply a representation for $\Symn$-equivariant conditional distributions, but under an additional conditional independence assumption a more detailed representation can be obtained due to the structure of $\Symn$ (\cref{sec:equivariant:seq}).

\subsection{\texorpdfstring{$\Symn$}{Sn}-Invariant Conditional Distributions}
\label{sec:invariant:seq}

The following special case of \cref{thm:group:invariant:rep,thm:suff:maximal:invariant} establishes a functional representation for all $\Symn$-invariant conditional distributions.

\begin{theorem} 
  \label{thm:finite:exch:seq:invariant} 
  Let $\Xn\in\calX^n$, for some $n\in\bbN$, be an exchangeable sequence, and $Y\in\calY$ some other random variable. Then $P_{Y|\Xn}$ $\Symn$-invariant if and only if there is a measurable function $f : [0,1] \times \calM(\calX) \to \calY$ such that
  \begin{align} \label{eq:finite:seq:rep:invariant}
    (\Xn,Y) \equas (\Xn,f(\noise ,\empirical_{\Xn})) \quad \text{where} \quad \noise\sim\Unif[0,1] \quad \text{and} \quad \noise \condind \Xn \;.
  \end{align}
  Furthermore, $\empirical_{\Xn}$ is sufficient for the family $\model^\Symn_{\Xn}$, 
  and adequate for the family $\model^\Symn_{\Xn,Y}$. 
\end{theorem}

\begin{proof}
  With \cref{thm:group:invariant:rep}, we need only prove that $\empirical_{\Xn}$ is a maximal invariant of $\calX^n$ under $\Symn$. Clearly, $\empirical_{\Xn}$ is $\Symn$-invariant. Now, let $\Xn'$ be another sequence such that $\empirical_{\Xn'} = \empirical_{\Xn}$. Then $\Xn'$ and $\Xn$ contain the same elements of $\calX$, and therefore $\Xn' = \perm\cdot \Xn$ for some $\perm\in\Symn$, so $\empirical_{\Xn}$ is a maximal invariant.  

  The second claim follow from \cref{thm:suff:maximal:invariant}.
\end{proof}

\paragraph{Modeling $\Symn$-invariance with neural networks} 
\Cref{thm:finite:exch:seq:invariant} is a general characterization of $\Symn$-invariant conditional distributions. It says that all such conditional distributions must have a noise-outsourced functional representation given by $Y=f(\noise,\empirical_{\Xn})$.  Recall that $\empirical_{\Xn}=\sum_{i=1}^n \delta_{X_i}$.  An atom $\delta_X$ can be thought of as a measure-valued generalization of a one-hot encoding to arbitrary measurable spaces, so that $\empirical_{\Xn}$ is a sum-pooling of encodings of the inputs (which removes information about the ordering of $\Xn$), and the output $Y$ is obtained by passing that, along with independent outsourced noise $\noise$, through a function $f$. In case the conditional distribution is deterministic, the outsourced noise is unnecessary, and we simply have $Y=f(\empirical_{\Xn})$.  

From a modeling perspective, one choice for (stochastic) neural network architectures that are $\Symn$-invariant is 
\begin{align}
Y=f(\noise,\sum_{i=1}^n \phi(X_i))\,,
\label{eq:invariant:sum:pool}
\end{align}
where $f$ and $\phi$ are arbitrary neural network modules, with $\phi$ interpreted as an embedding function of input elements into a high-dimensional space (see first panel of \cref{fig:sym:modules}).  These embeddings are sum-pooled, and passed through a second neural network module $f$. This architecture can be made to approximate arbitrarily well any $\Symn$-invariant conditional distribution \citep[c.f.,][]{Hornik:etal:1989}. Roughly, $\phi(X)$ can be made arbitrarily close to a one-hot encoding of $X$, which can in turn be made arbitrarily close to an atom $\delta_X$ by increasing its dimensionality, and similarly the neural module $f$ can be made arbitrarily close to any desired function. Below, we revisit an earlier example and give some new ones.

\begin{example}[\nameref{ex:deep:sets}, revisited]
  The architecture derived above is exactly the one described in \cref{ex:deep:sets}. \Cref{thm:finite:exch:seq:invariant} generalizes the result in \citet{Zaheer:etal:2017}, from deterministic functions to conditional distributions. The proof technique is also significantly simpler and sheds light on the core concepts underlying the functional representations of permutation invariance.\footnote{The result in \citet{Zaheer:etal:2017} holds for sets of arbitrary size when $\calX$ is countable and for fixed size when $\calX$ is uncountable. We note that the same is true for \cref{thm:finite:exch:seq:invariant}, for measure-theoretic reasons: in the countable $\calX$ case, the power sets of $\bbN$ form a valid Borel $\sigma$-algebra; for uncountable $\calX$, e.g., $\calX=\bbR$, there may be non-Borel sets and therefore the power sets do not form a Borel $\sigma$-algebra on which to define a probability distribution using standard techniques.} 
\end{example}

In general, a function of $\empirical_{\Xn}$ is a function of $\Xn$ that discards the order of its elements. That is, functions of $\empirical_{\Xn}$ are permutation-invariant functions of $\Xn$. The sum-pooling in \eqref{eq:invariant:sum:pool} gives rise to one such class of functions. Other permutation invariant pooling operations can be used. For example: product, maximum, minimum, log-sum-exp, mean, median, and percentiles have been used in various neural network architectures.  Any such function can be written in the form $f(\eta,\empirical_{\Xn})$, by absorbing the pooling operation into $f$ itself. The following examples illustrate.   

\begin{example}[Pooling using Abelian groups or semigroups]
 A group $\grp$, with binary operator $\oplus$, is Abelian if its elements commute: $g\oplus h = h\oplus g$ for all $g,h\in \grp$. Examples are $(\bbR_+,\times)$ and $(\bbZ,+)$. 
A semigroup is a group without the requirements for inverse and identity elements. Examples are $(\bbR_+,\vee)$ ($\vee$ denotes maximum: $x\vee y=\max\{x,y\}$) and $(\bbR_+,\wedge)$ ($\wedge$ denotes minimum).
For an Abelian group or semigroup $\grp$ and a map $\phi : \calX \to \grp$, a $\Symn$-invariant conditional distribution of $Y$ given a sequence $\Xn$ can be constructed with $f : [0,1] \times \grp \to \calY$ as
  \begin{align} \label{eq:symm:n}
    Y = f(\noise, \phi(X_1)\oplus \dotsb \oplus \phi(X_n)) \;.
  \end{align}
\end{example}

\begin{example}[Pooling using U-statistics] \label{ex:u:statistics}
  Given a permutation invariant function of $k\leq n$ elements, $\phi_k : \calX^k \to \calS$, a permutation invariant conditional distribution can be constructed with $f : [0,1] \times \calS \to \calY$ as
  \begin{align} \label{eq:u:statistic}
    f\bigg(\noise,\binom{n}{k}^{-1} \sum_{\{i_1,\dotsc,i_k\}\in [n]} \phi_k(X_{i_1},\dotsc,X_{i_k})\bigg)
  \end{align}
  where the pooling involves averaging over all $k$-element subsets of $[n]$. The average is a U-statistic \citep[e.g.,][]{Cox:Hinkley:1974}, and examples include the sample mean ($k=1$ and $\phi_k(x)=x$), the sample variance ($k=2$ and $\phi_k(x,y)=\frac{1}{2}(x-y)^2$), and estimators of higher-order moments, mixed moments, and cumulants. 

  \Citet{Murphy:etal:2019} developed a host of generalizations to the basic first-order pooling functions from \cref{ex:deep:sets} (\nameref{ex:deep:sets}), many of them corresponding to $k$-order U-statistics, and developed tractable computational techniques that approximate the average over $k$-element subsets by random sampling of permutations.
\end{example}

Exchangeability plays a central role in a growing body of work in the deep learning literature, particularly when deep learning methods are combined with Bayesian ideas. Examples include \citet{Edwards:Storkey:2017,Garnelo:etal:2018,Korshunova:etal:2018}, and the following.

\begin{example}[Neural networks for exchangeable genetic data]
In work by \citet{Chan:etal:2018}, $X_i\in \{ 0,1 \}^d$ is a binary $d$-dimensional vector indicating the presence or absence of $d$ single nucleotide polymorphisms in individual $i$. The individuals are treated as being exchangeable, forming an exchangeable sequence $\Xn$ of $\{ 0,1 \}^d$-valued random variables. \Citet{Chan:etal:2018} analyze the data using a neural network where each vector $X_i$ is embedded into $\bbR^{d'}$ using a convolutional network, the pooling operation is the element-wise mean of the top decile, and the final function is parameterized by a fully-connected network. They demonstrated empirically that encoding permutation invariance into the  network architecture led to  faster training and higher test set accuracy.
\end{example}

The final example, from the statistical relational artificial intelligence literature, uses the sufficiency of the empirical measure to characterize the complexity of inference algorithms for exchangeable sequences.

\begin{example}[Lifted inference]
  \Citet{Niepert:vdb:2014} studied the tractability of exact inference procedures for exchangeable models, through so-called lifted inference. 
  One of their main results shows that if $\Xn$ is a finitely exchangeable sequence of $\calX^d$-valued random variables on a discrete domain (i.e., each element of $\Xn$ is a $d$-dimensional vector of discrete random variables) then there is a sufficient statistic $S$, and probabilistic inference (defined as computing marginal and conditional distributions) based on $S$ has computational complexity that is polynomial in $d\times n$. In the simplest case, where $\calX=\{0,1\}$, $S$ is constructed as follows: encode all possible $d$-length binary vectors with unique bit strings $b_k\in\{0,1\}^d$, $k\in[2^d]$, and let $S(\Xn) = (c_1,\dotsc,c_{2^{d}})$ where $c_k = \sum_{i=1}^n \delta_{X_i}(b_k)$. Although not called the empirical measure by the authors, $S$ is precisely that.
\end{example}

\subsection{\texorpdfstring{$\Symn$}{Sn}-Equivariant Conditional Distributions} 
\label{sec:equivariant:seq}

Let $\Xn$ be an input sequence of length $n$, and $\Yn$ an output sequence. \Cref{thm:kall:equiv} shows that if the conditional distribution of $\Yn$ given $\Xn$ is $\Symn$-equivariant, then $\Yn$ can be expressed in terms of a noisy $\Symn$-equivariant function of $\Xn$. If the elements of $\Yn$ are assumed to be conditionally independent given $\Xn$, then by using properties of the finite symmetric group, we obtain a more detailed representation of $\Yn$ conditioned on $\Xn$. (The implications of the conditional independence assumption are discussed below.)

The resulting theorem is a one-dimensional special case of a more general representation theorem for exchangeable $d$-dimensional arrays (\cref{thm:equivariant:d:arrays} in \cref{sec:proof:arrays}).  The following simplified proof for sequences shows how the necessary conditional independence relationships are established and provides a template for proving the more general result. The $d=2$ case, which corresponds to graphs and networks, is taken up in \cref{sec:finite:exch:arrays}.

\begin{theorem} 
  \label{thm:finite:exch:seq:equivariant}
  Let $\Xn\in\calX^n$ be an exchangeable sequence and $\Yn\in\calY^n$ another random sequence, and assume that $Y_i \condind_{\Xn} (\Yn \setminus Y_i)$, for each $i\in [n]$. Then $P_{\Yn| \Xn}$ is $\Symn$-equivariant if and only if there is a measurable function $f : [0,1] \times \calX \times \calM(\calX) \to \calY$ such that
  \begin{align} \label{eq:finite:seq:rep:equivariant}
    \big( \Xn,\Yn \big) 
    \equas 
    \big(\Xn,
      \big(
        f(\noise_i,X_i,\empirical_{\Xn})
      \big)_{i\in [n]}
    \big) 
    \quad \text{where} \quad \noise_i\simiid\Unif[0,1]  \text{ and }  \noise_i \condind \Xn \;.
  \end{align}
\end{theorem}

\begin{proof}
  For the forward direction, suppose $\Yn$ is conditionally $\Symn$-equivariant given $\Xn$.  
  For a fixed $i\in[n]$, let $\Xnmi := \Xn \setminus X_i$ be $\Xn$ with its $i$th element removed, and likewise for $\Ynmi$.  The proof in this direction requires that we establish the conditional independence relationship 
    \begin{align} \label{eq:key:ci}
      Y_i \condind_{(X_i,\empirical_{\Xn})} (\Xn, \Ynmi) \;,
    \end{align}
   and then apply \cref{lem:noise:out:suff}. 

  To that end, let $\bbS_{n \setminus i}$ be the stabilizer of $i$, i.e., the subgroup of $\Symn$ that fixes element $i$. $\bbS_{n\setminus i}$ consists of permutations $\perm_{\setminus i}\in\bbS_{n}$ for which $\perm_{\setminus i}(i) = i$. The action of $\perm_{\setminus i}$ on $\Xn$ fixes $X_i$; likewise it fixes $Y_i$ in $\Yn$. By \cref{prop:joint:symmetry}, $(\Xn,\Yn)\equdist (\perm_{\setminus i}\cdot\Xn,\perm_{\setminus i}\cdot\Yn)$, so that, marginalizing out $\Ynmi$ yields
  $(\Xn,Y_i) \equdist (\perm_{\setminus i}\cdot \Xn,Y_i)$ for each $\perm_{\setminus i}\in\bbS_{n\setminus i}$. Moreover, $\bbS_{n\setminus i}$ forms a subgroup and is homomorphic to $\bbS_{n-1}$, so that the previous distributional equality is equivalent to
  \begin{align*}
    \big(  \Xnmi ,(X_i,Y_i)  \big) \equdist \big( \perm' \cdot\Xnmi, (X_i,Y_i)  \big) \quad \text{for each} \quad \perm'\in\bbS_{n-1} \;.
  \end{align*}
  \Cref{thm:finite:exch:seq:invariant}, with input $\Xnmi$ and output $(X_i,Y_i)$ then implies $(X_i,Y_i) \condind_{\empirical_{\Xnmi}} \Xnmi$. Conditioning on $X_i$ as well gives $Y_i \condind_{(X_i, \empirical_{\Xnmi})} \Xn$, marginally for each $Y_i$. With the assumption of mutual conditional independence among $\Yn$ conditioned on $\Xn$, the marginal conditional independence also holds jointly, and by the chain rule for conditional independence \citep[][Prop.~6.8]{Kallenberg:2002},
  \begin{align} \label{eq:yi:cond:ind}
      Y_i \condind_{(X_i,\empirical_{\Xnmi})} (\Xn,\Ynmi) \;.
  \end{align}
  Because conditioning on $(X_i,\empirical_{\Xnmi})$ is the same as conditioning on $(X_i,\empirical_{\Xn})$, \eqref{eq:yi:cond:ind} is equivalent to the key conditional independence relationship \eqref{eq:key:ci}.

  By \cref{lem:noise:out:suff}, there exists a measurable $f_i : [0,1] \times \calX \times \calM(\calX) \to \calX$ such that
    \begin{align*}
      (\Xn, \Ynmi, Y_i) \equas \big(\Xn, \Ynmi, f_i(\noise_i,X_i,\empirical_{\Xn})\big) \;,
    \end{align*}
  for $\noise_i\sim\Unif[0,1]$ and $\noise_i\condind (\Xn,\Ynmi)$. 
  This is true for each $i\in [n]$, and $\Symn$-equivariance implies that $(\Xn, \Ynmi, Y_i)\equdist (\Xn, \Ynmj, Y_j)$ for all $i,j\in[n]$.  Thus it is possible to choose the same function $f_i=f$ for all $i$. This yields \eqref{eq:finite:seq:rep:equivariant}. 

  The reverse direction is easy to verify, since the noise variables are i.i.d., $\Xn$ is exchangeable, and $\empirical_{\Xn}$ is permutation-invariant.  

\end{proof}

\paragraph{The impact of the conditional independence assumption $Y_i \condind_{\Xn}(\Yn \setminus Y_i)$} 
In the deterministic case, the assumed conditional independence relationships among the outputs $\Yn$ are trivially satisfied, so that \eqref{eq:finite:seq:rep:equivariant} (without outsourced noise) is the most general form for a permutation-equivariant function.
However, in the stochastic case, the assumed conditional independence significantly simplifies the structure of the conditional distribution and the corresponding functional representation. 
While the assumed conditional independence is key in the simplicity of the representation \eqref{eq:finite:seq:rep:equivariant}, it may limit the expressiveness: there are permutation-equivariant conditional distributions which do not satisfy the conditional independence assumption (examples are given below). 
On the other hand, without conditional independence between the elements of $\Yn$, it is possible to show that $Y_i = f(\noise_i, X_i, \empirical_{(X_j,Y_j)_{j\in [n]\setminus i}})$. Such dependence between elements of $\Yn$, although more expressive than \eqref{eq:finite:seq:rep:equivariant}, induces cycles in the computation graph, similar to Restricted Boltzmann Machines \citep{Smolensky:1987} and other Exponential Family Harmoniums \citep{Welling:etal:2005}. Furthermore, they may be limited in practice by the computational requirements of approximate inference algorithms. Striking a balance between flexibility and tractability via some simplifying assumption seems desirable.

Two examples illustrate the existence of permutation-equivariant conditional distributions which do not satisfy the conditional independence assumption made in \cref{thm:finite:exch:seq:equivariant}, and suggest another assumption. Both examples have a similar structure: there exists some random variable, say $W$, such that the conditional independence ${Y_i \condind_{(W,\Xn)} (\Yn \setminus Y_i)}$ holds. Assuming the existence of such a $W$ would lead to the representation $Y_i = f(\noise_i,W,X_i,\empirical_{\Xn})$, and potentially allow for more expressive models, as $W$ could be included in the neural network architecture and learned.

For the first example, let $\Yn$ be given as in \eqref{eq:finite:seq:rep:equivariant}, but with a vector of finitely exchangeable (but not i.i.d.) noise $\noisen\condind \Xn$. Then $\Yn$ would still be conditionally $\Symn$-equivariant, but it would not satisfy the conditional independence assumption ${Y_i \condind_{\Xn}(\Yn\setminus Y_i)}$. However, it would satisfy ${Y_i \condind_{(\Xn,\noisen)} (\Yn \setminus Y_i)}$, which by similar arguments as in the proof of \cref{thm:finite:exch:seq:equivariant}, implies the existence of a representation
\begin{align} \label{eq:alt:example}
  Y_i = f'(\noise'_i, \noise_i, X_i, \empirical_{(X_i,\noise_i)_{i\in [n]}}) \;,
\end{align}
for some other function $f' : [0,1]^2 \times \calX \times \calM(\calX \times [0,1])$ and i.i.d.\ noise $\noise'_i \condind (\Xn,\noisen)$, in which case \eqref{eq:finite:seq:rep:equivariant} would be a special case. 

As a second example, in practice it is possible to construct more elaborate conditionally $\Symn$-equivariant distributions by composing multiple ones as in \cref{prop:transitivity}(ii).  Suppose $\Yn$ is conditionally $\Symn$-equivariant and mutually independent given $\Xn$, and $\Zn$ is conditionally $\Symn$-equivariant and mutually independent given $\Yn$.  \Cref{prop:transitivity} shows that with $\Yn$ marginalized out, $\Zn$ is conditionally $\Symn$-equivariant given $\Xn$, while \cref{thm:finite:exch:seq:equivariant} guarantees the existence of the functional representations for each $i \in [n]$:
\begin{align*}
  (Y_i)_{i\in[n]} = (f(\noise_i,X_i,\empirical_{\Xn}))_{i\in[n]} 
    \quad \text{and} \quad 
  (Z_i)_{i\in[n]} = (f'(\noise'_i,Y_i,\empirical_{\Yn}))_{i\in[n]} \;.
\end{align*}
Substituting the functional representation for $\Yn$ into that for $\Zn$ yields, for each $i\in [n]$,
\begin{align*}
  Z_i = f'(\noise'_i,f(\noise_i,X_i,\empirical_{\Xn}),\empirical_{(f(\noise_i,X_i,\empirical_{\Xn}))_{i\in[n]}}) 
  = f''(\noise'_i, \noise_i, X_i ,\empirical_{(X_i,\noise_i)_{i\in [n]}}) \;,
\end{align*}
for some $f''$, which is a special case of \eqref{eq:alt:example} and which implies ${Z_i \condind_{(\Xn,\noisen)} (\Zn \setminus Z_i)}$.

\paragraph{Modeling $\Symn$-equivariance with neural networks} 
One choice for $\Symn$-equivariant neural networks is
\begin{align}
  Y_i = f(\eta_i, X_i, g(\Xn))
  \label{eq:equivariant:form}
\end{align}
where $f$ is an arbitrary (stochastic) neural network module, and $g$ an arbitrary permutation-invariant module (say one of the examples in \cref{sec:invariant:seq}).  
An example of a permutation-equivariant module using a permutation-invariant submodule is shown in the middle panel of \cref{fig:sym:modules}.

\Cref{prop:transitivity} enables the use of an architectural algebra for constructing complex permutation-invariant and -equivariant stochastic neural network modules.  Specifically, permutation-equivariant modules may be constructed by composing simpler permutation-equivariant modules (see the middle panel of \cref{fig:sym:modules}), while permutation-invariant modules can be constructed by composing permutation-equivariant modules with a final permutation-invariant module (right panel of \cref{fig:sym:modules}). Some examples from the literature illustrate.

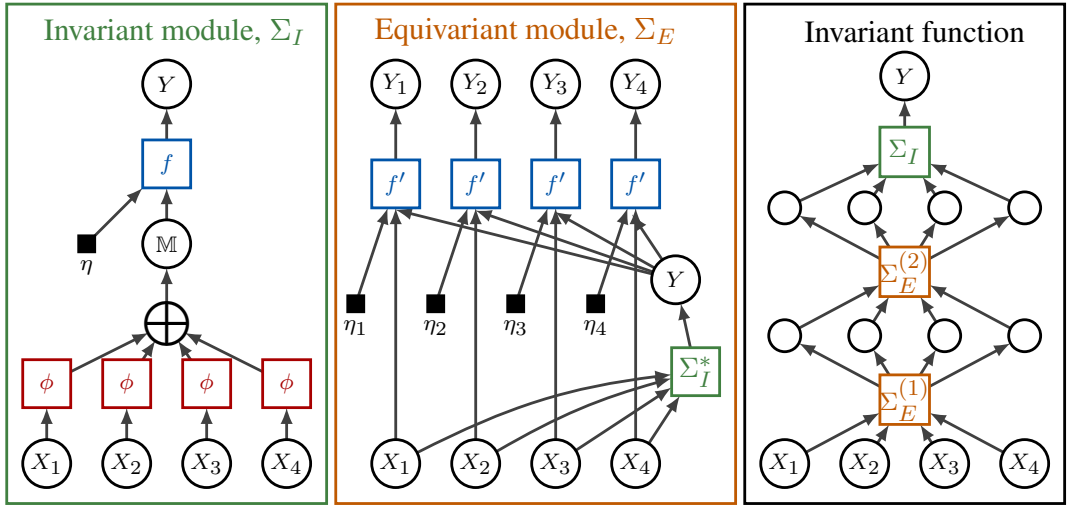
\begin{figure}[!t]
\begin{center}
  \resizebox{\textwidth}{!}{
     
    \begin{tikzpicture}
       
      \Vertex[x=0,y=0,opacity=0,label=$X_1$]{I1}
      \Vertex[x=1,y=0,opacity=0,label=$X_2$]{I2}
      \Vertex[x=2,y=0,opacity=0,label=$X_3$]{I3}
      \Vertex[x=3,y=0,opacity=0,label=$X_4$]{I4}
       
      \Vertex[x=0,y=1,label=$\phi$,shape=rectangle,opacity=0,fontcolor=redClj,style={draw=redClj}]{g1}
      \Vertex[x=1,y=1,label=$\phi$,shape=rectangle,opacity=0,fontcolor=redClj,style={draw=redClj}]{g2}
      \Vertex[x=2,y=1,label=$\phi$,shape=rectangle,opacity=0,fontcolor=redClj,style={draw=redClj}]{g3}
      \Vertex[x=3,y=1,label=$\phi$,shape=rectangle,opacity=0,fontcolor=redClj,style={draw=redClj}]{g4}

      \Vertex[x=1.5,y=1.75,opacity=0,size=0.5,style={color=white},fontsize=\LARGE,label={$\bigoplus$}]{spool}

      \Vertex[x=1.5,y=2.75,opacity=0,label={$\empirical$}]{m}
      \Edge[lw=1,Direct=true](spool)(m)

      \Vertex[x=1.5,y=3.75,opacity=0,shape=rectangle,fontcolor=blueClj,style={draw=blueClj},label={$f$}]{f}
      \Edge[lw=1,Direct=true](m)(f)

      \Vertex[x=0.5,y=2.75,color=black,shape=rectangle,size=0.2,label=$\noise$,position=below]{eta}
      \Edge[lw=1,Direct=true](eta)(f)

      \Vertex[x=1.5,y=4.75,opacity=0,label=$Y$]{Y}
      \Edge[lw=1,Direct=true](f)(Y)

      \Edge[lw=1,Direct=true](I1)(g1)
      \Edge[lw=1,Direct=true](I2)(g2)
      \Edge[lw=1,Direct=true](I3)(g3)
      \Edge[lw=1,Direct=true](I4)(g4)

      \Edge[lw=1,Direct=true](g1)(spool)
      \Edge[lw=1,Direct=true](g2)(spool)
      \Edge[lw=1,Direct=true](g3)(spool)
      \Edge[lw=1,Direct=true](g4)(spool)

      \Text[x=1.6,y=5.4,color=darkgreenClj]{Invariant module, $\Sigma_I$}

      \draw[draw=darkgreenClj,line width=1.0pt] (-0.5,-0.5) rectangle (3.5,5.75);
    \end{tikzpicture}
    \begin{tikzpicture}
      \Vertex[x=0,y=0,opacity=0,label=$X_1$]{I1}
      \Vertex[x=1,y=0,opacity=0,label=$X_2$]{I2}
      \Vertex[x=2,y=0,opacity=0,label=$X_3$]{I3}
      \Vertex[x=3,y=0,opacity=0,label=$X_4$]{I4}
      \Vertex[x=0,y=4.75,opacity=0,label=$Y_1$]{H1}
      \Vertex[x=1,y=4.75,opacity=0,label=$Y_2$]{H2}
      \Vertex[x=2,y=4.75,opacity=0,label=$Y_3$]{H3}
      \Vertex[x=3,y=4.75,opacity=0,label=$Y_4$]{H4}
      \Vertex[x=0,y=3.5,label=$f'$,shape=rectangle,opacity=0,fontcolor=blueClj,style={draw=blueClj}]{g1}
      \Vertex[x=1,y=3.5,label=$f'$,shape=rectangle,opacity=0,fontcolor=blueClj,style={draw=blueClj}]{g2}
      \Vertex[x=2,y=3.5,label=$f'$,shape=rectangle,opacity=0,fontcolor=blueClj,style={draw=blueClj}]{g3}
      \Vertex[x=3,y=3.5,label=$f'$,shape=rectangle,opacity=0,fontcolor=blueClj,style={draw=blueClj}]{g4}

      \Vertex[x=-0.1,y=3.2,Pseudo,size=0]{H1b}
      \Vertex[x=0.9,y=3.2,Pseudo,size=0]{H2b}
      \Vertex[x=1.9,y=3.2,Pseudo,size=0]{H3b}
      \Vertex[x=2.9,y=3.2,Pseudo,size=0]{H4b}

      \Vertex[x=3.75,y=1.15,label=$\Sigma_I^*$,shape=rectangle,style={draw=darkgreenClj},fontsize=\small,,opacity=0,fontcolor=darkgreenClj]{phit}
      \Vertex[x=3.5,y=2.3,opacity=0,label={$Y$}]{Z}

      \Vertex[x=-0.5,y=2,color=black,shape=rectangle,size=0.2,label=$\noise_1$,position=below]{eta1}
      \Vertex[x=0.5,y=2,color=black,shape=rectangle,size=0.2,label=$\noise_2$,position=below]{eta2}
      \Vertex[x=1.5,y=2,color=black,shape=rectangle,size=0.2,label=$\noise_3$,position=below]{eta3}
      \Vertex[x=2.5,y=2,color=black,shape=rectangle,size=0.2,label=$\noise_4$,position=below]{eta4}

      \Edge[lw=1,Direct=true](I1)(g1)
      \Edge[lw=1,Direct=true](I2)(g2)
      \Edge[lw=1,Direct=true](I3)(g3)
      \Edge[lw=1,Direct=true](I4)(g4)

      \Edge[lw=1,Direct=true](g1)(H1)
      \Edge[lw=1,Direct=true](g2)(H2)
      \Edge[lw=1,Direct=true](g3)(H3)
      \Edge[lw=1,Direct=true](g4)(H4)

      \Edge[lw=1,Direct=true,bend=10](I1)(phit)
      \Edge[lw=1,Direct=true,bend=7.5](I2)(phit)
      \Edge[lw=1,Direct=true](I3)(phit)
      \Edge[lw=1,Direct=true](I4)(phit)

      \Edge[lw=1,Direct=true](phit)(Z)

      \Edge[lw=1,Direct=true](Z)(H1b)
      \Edge[lw=1,Direct=true](Z)(H2b)
      \Edge[lw=1,Direct=true](Z)(H3b)
      \Edge[lw=1,Direct=true](Z)(H4b)

      \Edge[lw=1,Direct=true](eta1)(g1)
      \Edge[lw=1,Direct=true](eta2)(g2)
      \Edge[lw=1,Direct=true](eta3)(g3)
      \Edge[lw=1,Direct=true](eta4)(g4)

      \Text[x=1.625,y=5.4,color=orangeClj]{Equivariant module, $\Sigma_E$}

      \draw[draw=orangeClj,line width=1.0pt] (-0.75,-0.5) rectangle (4.25,5.75);
    \end{tikzpicture}
    \begin{tikzpicture}
      \Vertex[x=0,y=0,opacity=0,label=$X_1$]{I1}
      \Vertex[x=1,y=0,opacity=0,label=$X_2$]{I2}
      \Vertex[x=2,y=0,opacity=0,label=$X_3$]{I3}
      \Vertex[x=3,y=0,opacity=0,label=$X_4$]{I4}

      \Vertex[x=1.5,y=0.8,label=$\Sigma_E^{(1)}$,shape=rectangle,style={draw=orangeClj},fontsize=\footnotesize,opacity=0,fontcolor=orangeClj,size=0.62]{em1}

      \Vertex[x=0,y=1.6,opacity=0,size=0.4]{H1}
      \Vertex[x=1,y=1.6,opacity=0,size=0.4]{H2}
      \Vertex[x=2,y=1.6,opacity=0,size=0.4]{H3}
      \Vertex[x=3,y=1.6,opacity=0,size=0.4]{H4}

      \Vertex[x=1.5,y=2.4,label=$\Sigma_E^{(2)}$,shape=rectangle,style={draw=orangeClj},fontsize=\footnotesize,opacity=0,fontcolor=orangeClj,size=0.62]{em2}

      \Vertex[x=0,y=3.2,opacity=0,size=0.4]{H21}
      \Vertex[x=1,y=3.2,opacity=0,size=0.4]{H22}
      \Vertex[x=2,y=3.2,opacity=0,size=0.4]{H23}
      \Vertex[x=3,y=3.2,opacity=0,size=0.4]{H24}

      \Vertex[x=1.5,y=3.9,label=$\Sigma_I$,shape=rectangle,style={draw=darkgreenClj},fontsize=\footnotesize,opacity=0,fontcolor=darkgreenClj,size=0.62]{im}

      \Vertex[x=1.5,y=4.85,opacity=0,label=$Y$]{Y}
      \Edge[lw=1,Direct=true](im)(Y)

      \Edge[lw=1,Direct=true](I1)(em1)
      \Edge[lw=1,Direct=true](I2)(em1)
      \Edge[lw=1,Direct=true](I3)(em1)
      \Edge[lw=1,Direct=true](I4)(em1)

      \Edge[lw=1,Direct=true](em1)(H1)
      \Edge[lw=1,Direct=true](em1)(H2)
      \Edge[lw=1,Direct=true](em1)(H3)
      \Edge[lw=1,Direct=true](em1)(H4)

      \Edge[lw=1,Direct=true](H1)(em2)
      \Edge[lw=1,Direct=true](H2)(em2)
      \Edge[lw=1,Direct=true](H3)(em2)
      \Edge[lw=1,Direct=true](H4)(em2)

      \Edge[lw=1,Direct=true](em2)(H21)
      \Edge[lw=1,Direct=true](em2)(H22)
      \Edge[lw=1,Direct=true](em2)(H23)
      \Edge[lw=1,Direct=true](em2)(H24)

      \Edge[lw=1,Direct=true](H21)(im)
      \Edge[lw=1,Direct=true](H22)(im)
      \Edge[lw=1,Direct=true](H23)(im)
      \Edge[lw=1,Direct=true](H24)(im)

      \Text[x=1.625,y=5.4,color=black]{Invariant function}

      \draw[draw=black,line width=1.0pt] (-0.5,-0.5) rectangle (3.5,5.75);
    \end{tikzpicture}
  }
\end{center}
\caption{
  \emph{Left}: An invariant module depicting \eqref{eq:symm:n}. \emph{Middle}: An equivariant module depicting \eqref{eq:equivariant:form}; note that the invariant sub-module, $\Sigma_{I}^*$, must be deterministic unless there are alternative conditional independence assumptions, such as \eqref{eq:alt:example}. \emph{Right}: An invariant stochastic function composed of equivariant modules. 
  Functional representations of $\Symn$-invariant and -equivariant conditional distributions. Circles denote random variables, with a row denoting an exchangeable sequence. The blue squares denote arbitrary functions, possibly with outsourced noise $\noise$ which are mutually independent and independent of everything else. Same labels mean that the functions are the same. Red squares denote arbitrary embedding functions, possibly parameterized by a neural network, and $\oplus$ denotes a symmetric pooling operation. Orange rectangles denote a module which gives a functional representation of a $\Symn$-equivariant conditional distribution. Likewise green rectangles for permutation-invariant conditional distributions.
}
\label{fig:sym:modules}
\end{figure}

\begin{example}[\nameref{ex:deep:sets:equivariant}, revisited] \label{ex:deep:sets:revisited}
  It is straightforward to see that \cref{ex:deep:sets:equivariant} is a deterministic special case of \cref{thm:finite:exch:seq:equivariant}: 
  \begin{align*}
    Y_i & = \activation(\theta_0 X_i  + \theta_1 \textstyle\sum_{j=1}^n X_j)
  \end{align*}
where $\sum_{j=1}^n X_j = \int_{\calX} \empirical_{\Xn}(dx)$ is a function of the empirical measure. 
While this example's nonlinear activation of linear combinations encodes a typical feed-forward neural network structure, \cref{thm:finite:exch:seq:equivariant} shows that more general functional relationships are allowed, for example \eqref{eq:equivariant:form}. 
\end{example}

\begin{example}[Equivariance and convolution]
  \Citet{Kondor:Trivedi:2018} characterized the properties of deterministic feed-forward neural networks that are equivariant under the action of a compact group, $\grp$. Roughly speaking, their results show that each layer $\ell$ of the network must be a convolution of the output of the previous layer with some filter $\chi_{\ell}$. 
  The general form of the convolution is defined in group theoretic terms that are beyond the scope of this paper. However, in the case of an exchangeable sequence, the situation is particularly simple. As an alternative to \eqref{eq:finite:seq:rep:equivariant}, $Y_i$ may be represented by a function $f'(\eta_i,X_i,\empirical_{\Xnmi})$ (see \eqref{eq:yi:cond:ind} in the proof of \cref{thm:finite:exch:seq:equivariant}), which makes clear the structure of the relationship between $\Xn$ and $\Yn$: element $Y_i$ has a ``receptive field'' that focuses on $X_i$ and treats the elements $\Xnmi$ as a structureless background field via $\empirical_{\Xnmi}$. The dependence on the latter is invariant under permutations $\perm'\in\bbS_{n-1}$; in group theoretic language, $\bbS_{n-1}$ {stabilizes} $i$ in $[n]$. That is, all permutations that fix $i$ form an equivalence class. As such, for each $i$ the index set $[n]$ is in one-to-one correspondence with the set of equivalent permutations that either move the $i$th element to some other element (there are $n-1$ of these), or fix $i$. This is the quotient space $\Symn/\bbS_{n-1}$; by the main result of \cite{Kondor:Trivedi:2018}, any $\Symn$-equivariant feed-forward network with hidden layers all of size $n$ must be composed of connections between layers $X\mapsto Y$ defined by the convolution 
  \begin{align*}
    Y_i = \activation((X * \chi)_i) = \activation\bigg(\sum_{j=1}^n  X_{(i+j) \text{ mod }  n} \chi(j) \bigg) \;, \; \chi(j) = \delta_n(j) (\theta_0 + \theta_1) + \sum_{k=1}^{n-1} \delta_{k}(j) \theta_1 \;.
  \end{align*}
  The representation may be interpreted as a convolution of the previous layer's output with a filter ``centered'' on element $X_i$, and is equivalent to that of \cref{ex:deep:sets:revisited}.
\end{example}

\begin{example}[Self-attention] \label{ex:self:attention}
  \Citet{Lee:etal:2018:set:transformer} proposed a $\Symn$-invariant architecture based on \emph{self-attention} \citep{Vaswani:etal:2017}. For an input set of $n$ $d$-dimensional observations, the so-called \emph{Set Transformer} combines attention over the $d$ input dimensions with nonlinear functions of pairwise interactions between the input observations. In the simplest implementation, with no attention components, the Set Transformer computes in each network layer a nonlinear activation of the Gramian matrix $\Xn \Xn^T$; the full architecture with attention is somewhat complicated, and we refer the reader to \citet{Lee:etal:2018:set:transformer}. Furthermore, to combat prohibitive computational cost, a method inspired by inducing point techniques from the Gaussian Process literature \citep{Snelson:Ghahramani:2006} was introduced. In a range of experiments focusing on tasks that benefit from modeling dependence between elements of the set, such as clustering, the Set Transformers architecture out-performed architectures that did not include pairwise or higher-order interactions, like Deep Sets (\cref{ex:deep:sets}).
\end{example}

\subsection{Input Sets of Variable Size}
\label{sec:var:size}

In applications, the data set may consist of finitely exchangeable sequences of varying length. \Cref{thm:finite:exch:seq:invariant,thm:finite:exch:seq:equivariant} are statements about input sequences of fixed length $n$. In general, they suggest that a separate function $f_n$ needs to be learned for each $n$ for which there is a corresponding observation $\Xn$ in the data set. As a practical matter, clearly this is undesirable. In practice, the most common approach is to compute an independent embedding $\phi : \calX \to \bbR$ of each element of an input set, and then combine the embeddings with a symmetric pooling function like sum or max. For example, $f(\Xn) = \max\{ \phi(X_1),\dotsc,\phi(X_n) \}$, or \eqref{eq:deep:sets:layer} from \cref{ex:deep:sets,ex:deep:sets:equivariant}. 
Clearly, such a function is invariant under permutations. However, recent work has explored the use of pairwise and higher-order interactions in the pooling function \citep[the work by][mentioned in \cref{ex:u:statistics} is an example]{Murphy:etal:2019}; empirical evidence indicates that the increased functional complexity results in higher model capacity and better performance for tasks that rely on modeling the dependence between elements in the input set.

Intuitively, more complex functions, such as those composed of pairwise (or higher-order) functions of an input sequence, give rise to higher-capacity models able to model more complicated forms of dependence between elements of an input sequence. In the context of exchangeability, this can be made more precise, as a difference between finitely and infinitely exchangeable sequences. 
In particular, let $\Xn$ be the length $n$ prefix of an infinitely exchangeable sequence $\Xinf$. If a sequence of sufficient statistics $S_n : \calX^n \to \calS_n$ exists, the conditionally \iid representation \eqref{eq:deFinetti:cond:iid} of the distribution of $\Xinf$ requires that they have the following properties \citep{Freedman:1962,Lauritzen:1984,Lauritzen:1988}:
\begin{enumerate}[label=(\roman*)] 
  \item \emph{symmetry under permutation}:
    \begin{align} \label{eq:symm:stat}
      S_n(\perm\cdot \Xn) = S_n(\Xn) \quad \text{for all } \perm\in\Symn, \quad n\in\bbN \;;
    \end{align}
  \item \emph{recursive computability}: for all $n,m\in\bbN$ there are functions $\psi_{n,m} : \calS_n \times \calS_m \to \calS_{n+m}$ such that
    \begin{align} \label{eq:recursive:stat}
      S_{n+m}(\Xnpm) = \psi_{n,m}(S_n(\Xn),S_m({\Xnpmtail})) \;.
    \end{align}
\end{enumerate}
A statistic that satisfies these properties must be of the form \citep{Lauritzen:1988}
\begin{align} \label{eq:exch:extend}
  S_n({\Xn}) = S_1(X_1)\oplus\dotsb\oplus S_1(X_n) \;,
\end{align}
where $(\calS_1,\oplus)$ is an Abelian group or semigroup. Equivalently, because of symmetry under permutation, we write $S_n(\empirical_{\Xn})$. 
Examples with $\calX = \bbR_+$ include: 
\begin{enumerate}[label=(\roman*)]
  \item $S_1(X_i) = \log X_i$ with $S_1(X_i)\oplus S_1(X_j) = \log X_i + \log X_j$; 
  \item $S_1(X_i)=X_i$ with $S_1(X_i)\oplus S_1(X_j) = X_i \vee X_j$; 
  \item $S_1(X_i) = \delta_{X_i}(\argdot)$ with $S_1(X_i)\oplus S_1(X_j) = \delta_{X_i}(\argdot) + \delta_{X_j}(\argdot)$. 
\end{enumerate}
Observe that pairwise (i.e., second-order) and higher-order statistics that do not decompose into first-order functions are precluded by the recursive computability property; an example is $S_n(\empirical_{\Xn}) = \sum_{i,j\in [n]} X_i X_j$. 

Infinite exchangeability restricts the dependence between elements in $\Xn$: using a model based on only first-order functions, so that the properties of permutation symmetry and recursive computability are satisfied, limits the types of dependence that the model can capture. In practice, this can lead to shortcomings. For example, an infinitely exchangeable sequence cannot have negative correlation $\rho = \text{Corr}(X_i,X_j)$, but a finitely exchangeable sequence that cannot be extended to a longer exchangeable sequence can have $\rho < 0$ \citep[e.g.,][pp.\ 7-8]{Aldous:1983}. One way to interpret this fact is that the type of dependence that gives rise to negative covariance cannot be captured by first-order functions. 
When using $\Xn$ to predict another random variable, the situation becomes more complex, but to the extent that adequate (sufficient and d-separating) statistics are used, the same concepts are relevant. 
Ultimately, a balance between flexibility and computational efficiency must be found; the exact point of balance will depend on the details of the problem and may require novel computational methods (i.e., the inducing point-like methods in \cref{ex:self:attention}, or the sampling approximations in \citet{Murphy:etal:2019}) so that more flexible function classes can be used.

\subsection{Partially Exchangeable Sequences and Layers of Different Sizes}
\label{sec:seq:partial:exch}

\Citet{ShaweTaylor:1989} and \citet{Ravanbakhsh:etal:2017} consider the problem of input-output invariance under a general discrete group $\grp$ acting on a standard feed-forward network, which consists of layers, potentially with different numbers of nodes, connected by weights. Those papers each found that $\grp$ and the neural network architecture must form a compatible pair: a pair of layers, treated as a weighted bipartite graph, forms a $\grp$-equivariant function if the action of $\grp$ on that graph is an automorphism. 
Essentially, $\grp$ must partition the nodes of each layer into weight-preserving orbits; \citet{Ravanbakhsh:etal:2017} provide some illuminating examples.

In the language of exchangeability, such a partition of elements corresponds to \emph{partial exchangeability}:\footnote{\Citet{Diaconis:Freedman:1984} use partial exchangeability to mean any number of probabilistic symmetries; we use it only to mean partial permutation-invariance.} the distributional invariance of $\bfX_{n_x}$ under the action of a subgroup of the symmetric group, $\grp \subset \Symn$. Partially exchangeable analogues of \cref{thm:finite:exch:seq:invariant,thm:finite:exch:seq:equivariant} are possible using the same types of conditional independence arguments. The resulting functional representations would express elements of $\bfY_{n_y}$ in terms of the empirical measures of blocks in the partition of $\bfX_{n_x}$; the basic structure is already present in \cref{thm:kall:equiv}, but the details would depend on conditional independence assumptions among the elements of $\bfY_{n_y}$. We omit a specific statement of the result, which would require substantial notational development, for brevity.

\section{Learning From Finitely Exchangeable Matrices and Graphs}
\label{sec:finite:exch:arrays}

Neural networks that operate on graph-valued input data have been useful for a range of tasks, from molecular design \citep{Duvenaud:etal:2015} and quantum chemistry \citep{Gilmer:etal:2017}, to knowledge-base completion \citep{Hamaguchi:etal:2017}.  

In this section, we consider random matrices\footnote{The results here are special cases of general results for $d$-dimensional arrays. For simplicity, we present the two-dimensional case and consider the general case in \cref{sec:proof:arrays}.} whose distribution is invariant to permutations applied to the index set. In particular, let $\Xntwo$ be a two-dimensional $\calX$-valued array with index set $[\bfn_2]:= [n_1] \times [n_2]$, such that $X_{i,j}$ is the element of $\Xntwo$ at position $(i,j)$. Let $\perm_k\in\Symnd{k}$ be a permutation of the set $[n_k]$ for $k\in\{1,2\}$. Denote by $\bbS_{\bfn_2}$ the direct product $\bbS_{n_1}\times \bbS_{n_2}$. A {collection} of permutations $\permtwo:=(\perm_1,\perm_2) \in \bbS_{\bfn_2}$ acts on $\Xntwo$ in the natural way, separately on the corresponding dimension:
\begin{align} \label{eq:perm:action:sep:matrix}
  [\permtwo \cdot \Xntwo ]_{i,j} = X_{\perm_1(i),\perm_2(j)} \;.
\end{align}
The distribution of $\Xntwo$ is \emph{separately exchangeable} if
\begin{align} \label{eq:sep:exch}
  \permtwo \cdot \Xntwo = (X_{\perm_1(i),\perm_2(j)})_{i\in [n_1],j\in [n_2]} \equdist (X_{i,j})_{i\in [n_1],j\in [n_2]} = \Xntwo  \;,
\end{align}
for every collection of permutations $\permtwo\in \bbS_{\bfn_2}$. 
We say that $\Xntwo$ is separately exchangeable if its distribution is. 

For symmetric arrays, such that $n_1=n_2=n$ and $X_{i,j}=X_{j,i}$, a different notion of exchangeability is needed. The distribution of a symmetric $\calX$-valued array $\Xntwosymm$ is \emph{jointly exchangeable} if, for all $\perm\in\Symn$,
\begin{align} \label{eq:joint:exch}
  \perm\cdot\Xntwosymm = (\Xsymm_{\perm(i),\perm(j)})_{i,j\in [n]} \equdist (\Xsymm_{i,j})_{i,j\in [n]} = \Xntwosymm  \;.
\end{align}

\subsection{Sufficient Representations of Exchangeable Matrices} 
\label{sec:suff:rep:arrays}

In order to obtain functional representation results for matrices, a suitable analogue to the empirical measure is required. 
In contrast to the completely unstructured empirical measure of a sequence defined in \eqref{eq:empirical:seq}, a sufficient representation of a matrix must retain the structural information encoded by the rows and columns of the matrix, but discard any ordering information. 
For matrices, such an object corresponds to a step function, also called the empirical graphon or a {checkerboard function} \citep{Orbanz:Roy:2015,Borgs:Chayes:2017}, which has played a central role in the theory of large exchangeable graphs and their representations as graphons \citep{Lovasz:2012}. 

In the context of the current work, working with the checkerboard function is unnecessarily complicated; any maximal invariant of the symmetric group acting on $\calX$-valued arrays will suffice. Specifically, any \emph{canonical form} $\Cntwo := \canon(\Xntwo)$ (or $\Cntwosymm=\canon(\Xntwosymm)$ for symmetric matrices), which serves as an orbit representative, can be used. The computational problem of finding a canonical representative of a graph is known as \emph{canonical labeling}, and is at least as hard as the graph isomorphism problem (which is in NP, though it is unknown whether it is in P or NP-complete). In practice, some graph automorphism tools rely on canonical labeling algorithms, of which \textsc{Nauty} and \textsc{Traces} \citep{McKay:Piperno:2014} have demonstrated remarkably good performance. 
Canonical labeling results in a permutation $\permtwo$ applied to the input to obtain an orbit representative, and therefore the process also can be used to define a representative equivariant $\tau : \calX^{n_1\times n_2} \to \bbS_{\bfn_2}$. Different input domains $\calX$ may admit different canonicalization procedures, though for most practical purposes, $\calX\subseteq \bbR$.

\subsection{\texorpdfstring{$\bbS_{\bfn_2}$}{Sn2}-Invariant Conditional Distributions}
\label{sec:invariant:output}

As in \cref{sec:invariant:seq}, consider modeling $Y$ as the output of a function whose input is a separately exchangeable matrix $\Xntwo$. 
In analogy to sequences in \cref{thm:finite:exch:seq:invariant}, sufficiency of $\Cntwo$ characterizes the class of all finitely exchangeable distributions on $\calX^{n_1\times n_2}$. 
The proof simply requires showing that $\Cntwo$ is a maximal invariant of $\bbS_{\bfn_2}$ acting on $\calX^{n_1\times n_2}$. To state the result, for any fixed canonicalization procedure, let $\cspace_{\bfn_2}(\calX)$ denote the space of canonical forms of matrices in $\calX^{n_1\times n_2}$.

\begin{theorem}
  \label{thm:invariant:two:arrays}
  Let $\Xntwo$ be a separately exchangeable $\calX$-valued matrix indexed by $[n_1]\times [n_2]$, and $Y\in\calY$ another random variable. Then $P_{Y|\Xntwo}$ is $\bbS_{\bfn_2}$-invariant if and only if there is a measurable function $f : [0,1] \times \cspace_{\bfn_2}(\calX) \to \calY$ such that
  \begin{align} \label{eq:invariant:two:arrays}
    (\Xntwo,Y) \equas (\Xntwo,f(\noise ,\Cntwo)) \quad \text{where} \quad \noise\sim\Unif[0,1] \quad \text{and} \quad \noise \condind \Xntwo \;.
  \end{align}
  Furthermore, $\Cntwo$ is sufficient for the family $\model_{\Xntwo}^{\bbS_{\bfn_2}}$ and adequate for $\model_{\Xntwo,Y}^{\bbS_{\bfn_2}}$.
\end{theorem}
\begin{proof}
  $\Cntwo$ is a maximal invariant by construction, and therefore \cref{thm:group:invariant:rep,thm:suff:maximal:invariant} yield the result.
\end{proof}

An identical result holds for jointly exchangeable matrices $\Xntwosymm$, with symmetric canonical form $\Cntwosymm$.

As was the case for sequences, any function of a canonical form is invariant under permutations of the index set of $\Xntwo$. Most of the recent examples from the deep learning literature incorporate vertex features; that composite case is addressed in \cref{sec:vertex:features}. A simple example without vertex features is the read-out layer of a neural network that operates on undirected, symmetric graphs. 

\begin{example}[Read-out for message-passing neural networks] \Citet{Gilmer:etal:2017} reviewed recent work on neural networks whose input is an undirected graph (i.e., a symmetric matrix) on vertex set $[n]$, and whose hidden layers act as message-passing operations between vertices of the graph. These message-passing hidden layers are equivariant (see \cref{sec:equivariant:arrays,sec:vertex:features}); adding a final invariant layer makes the whole network invariant. A particularly simple architecture involves a single message-passing layer and no input features on the vertices, and a typical permutation-invariant read-out layer of the form
\begin{align*}
  R =  \sum_{i\in [n]} f\bigg( \sum_{j \in [n]} h(X_{i,j} )  \bigg) = f'(\Cntwosymm) \;.
\end{align*}
\end{example}

\subsection{\texorpdfstring{$\bbS_{\bfn_2}$}{Sn2}-Equivariant Conditional Distributions}
\label{sec:equivariant:arrays}

In analogy to $\Xn$ and $\Yn$ in \cref{sec:equivariant:seq}, $\Xntwo$ and $\Yntwo$ might represent adjacent neural network layers; in such cases the goal is to transfer the symmetry of $\Xntwo$ to $\Yntwo$, and permutation-equivariance is the property of interest. With a collection of permutations acting on $\Xntwo$ and $\Yntwo$ as in \eqref{eq:perm:action:sep:matrix}, permutation-equivariance is defined in the same way as for sequences. In particular, if $\Xntwo$ is exchangeable then $\Yntwo$ is conditionally $\bbS_{\bfn_2}$-equivariant if and only if
\begin{align}
  (\permtwo \cdot \Xntwo, \permtwo \cdot \Yntwo) \equdist (\Xntwo,\Yntwo) \quad \text{for all } \permtwo \in \bbS_{\bfn_2} \;.
\end{align}
The main result in this section is a functional representation of $\Yntwo$ in terms of a separately exchangeable array $\Xntwo$, when the elements of $\Yntwo$ are also conditionally independent given $\Xntwo$.\footnote{This is satisfied by neural networks without intra-layer or skip connections. Weaker conditional independence assumptions may be considered, as in \cref{sec:equivariant:seq}.} In particular, each element $Y_{i,j}$ is expressed in terms of $X_{i,j}$, outsourced noise $\noise_{i,j}$, and an augmented canonical form defined as follows.

Let $\Cntwoi$ denote the $i$th row of $\Cntwo$, and $\Cntwoj$ the $j$th column; define the \emph{separately augmented canonical form} as
\begin{align} \label{eq:sep:aug:cboard}
  [\augc]_{k,\ell} = ([\Cntwo]_{k,\ell},[\Cntwo]_{i,\ell},[\Cntwo]_{k,j}) \;.
\end{align} 
$\augc$ augments the canonical form $\Cntwo$ with the $i$th row and $j$th column, which are broadcast over the appropriate dimensions; one encodes $\Cntwoi$ and one encodes $\Cntwoj$. These are analogues of the empirical measures $\empirical_{\Cntwoi}$ and $\empirical_{\Cntwoj}$, but with their structure coupled to that of $\Cntwo$. Denote by $\augcspace_{\bfn_2}(\calX)$ the space of all such functions augmented canonical forms of matrices with dimension $\bfn_2$.

A general version of the following theorem, for $d$-dimensional arrays, is given in \cref{sec:proof:arrays}. 
The proof has the same basic structure as the one for sequences in \cref{thm:finite:exch:seq:equivariant}, but with substantially more notation. Below, the proof for $d=2$ is given in order to highlight the important structure.

\begin{theorem}
  \label{thm:equivariant:two:arrays}

  Suppose $\Xntwo$ and $\Yntwo$ are $\calX$- and $\calY$-valued arrays, respectively, each indexed by $[n_1]\times [n_2]$, and that $\Xntwo$ is separately exchangeable. Assume that the elements of $\Yntwo$ are mutually conditionally independent given $\Xntwo$. Then $\Yntwo$ is conditionally $\bbS_{\bfn_2}$-equivariant given $\Xntwo$ if and only if there is a measurable function $f : [0,1] \times \calX \times \augcspace_{\bfn_2}(\calX) \to \calY$ such that 
  \begin{align} \label{eq:rep:exch:two:array}
    \bigg(\Xntwo, \Yntwo  \bigg)
    \equas 
    \bigg(\Xntwo,
      \big(
        f(
        \noise_{i,j}, 
        X_{i,j}, 
        \augc
        )
      \big)_{i\in [n_1], j\in [n_2]}
    \bigg) \;,
  \end{align}
  for \iid uniform random variables $(\noise_{i,j})_{i\in [n_1],j\in[n_2]} \condind \Xntwo$. 
\end{theorem}

\begin{proof}
  First, assume that $\Yntwo$ is conditionally $\bbS_{\bfn_2}$-equivariant given $\Xntwo$. Then $\permtwo\cdot(\Xntwo,\Yntwo)\equdist (\Xntwo,\Yntwo)$ for all $\permtwo\in\bbS_{\bfn_2}$. 
  Let $\bbS_{\bfn_2}^{(i,j)}\subset \bbS_{\bfn_2}$ be the stabilizer subgroup of $(i,j)$, i.e., the set of permutations that fixes element $(i,j)$ in $\Xntwo$. Note that each $\permtwo^{(i,j)} \in\bbS^{(i,j)}_{\bfn_2}$ fixes both $X_{i,j}$ and $Y_{i,j}$, and that $\bbS_{\bfn_2}^{(i,j)}$ is homomorphic to $\bbS_{\bfn_2 - \bfOne}$. Observe that any $\permtwo^{(i,j)} \in \bbS_{\bfn_2}^{(i,j)}$ may rearrange the elements within the $i$th row of $\Xntwo$, but it remains the $i$th row in $\permtwo^{(i,j)}\cdot \Xntwo$. Similarly, the elements in the $j$th column, $\Xntwoj$ may be rearranged but remains the $j$th column. As a result, the $j$th element of every row is fixed (though it moves with its row), as is the $i$th element of every column. 

  That fixed-element structure will be used to establish the necessary conditional independence relationships. To that end, let $r_i : [n_1]\setminus i \to [n_1-1]$ map the row indices of $\Xntwo$ to the row indices of the matrix obtained by removing the $i$th row from $\Xntwo$: 
  \begin{align*}
    r_i^{-1}(k) = 
    \begin{cases}
      k & k < i \\
      k - 1 & k > i \;.
    \end{cases}
  \end{align*}
  Analogously, let $c_j^{-1} : [n_2]\setminus j \to [n_2 - 1]$ map the column indices of $\Xntwo$ to those of $\Xntwo$ with the $j$th column removed. 
  Define the $\calX^3$-valued array $\bfZ^{(i,j)}$ as
  \begin{align} \label{eq:z:array}
    [\bfZ^{(i,j)}]_{k,\ell} = (X_{r_i(k),c_j(\ell)},X_{i,c_j(\ell)},X_{r_i(k),j}) \;, \quad k \in [n_1 - 1], \ \ell \in [n_2 -1] \;.
  \end{align}
  That is, $\bfZ^{(i,j)}$ is formed by removing $\Xntwoi$ and $\Xntwoj$ from $\Xntwo$ (and $X_{i,j}$ from $\Xntwoi$ and $\Xntwoj$), and broadcasting the removed row and column entries over the corresponding rows and columns of the matrix that remains. $\bfZ^{(i,j)}$ inherits the exchangeability of $\Xntwo$ in the first element of each entry, and the fixed-elements structure in the second two elements, and therefore overall it is separately exchangeable:
  \begin{align*}
    \permtwo' \cdot \bfZ^{(i,j)} \equdist \bfZ^{(i,j)} \;, 
      \quad \text{for all} \quad \permtwo' \in \bbS_{\bfn_2 - \bfOne} \;.
  \end{align*}
  Now, marginally (for $Y_{i,j}$),
  \begin{align*}
    (\permtwo' \cdot \bfZ^{(i,j)}, (X_{i,j},Y_{i,j}) ) \equdist (\bfZ^{(i,j)}, (X_{i,j}, Y_{i,j})) \;, \quad \text{for all} \quad \permtwo' \in \bbS_{\bfn_2 - \bfOne} \;.
  \end{align*}
  Therefore, by \cref{thm:invariant:two:arrays}, $(X_{i,j},Y_{i,j}) \condind_{\cboard_{\bfZ^{(i,j)}}} \bfZ^{(i,j)}$. Conditioning on $X_{i,j}$ and $\cboard_{\bfZ^{(i,j)}}$ is equivalent to conditioning on $X_{i,j}$ and $\augc$, yielding 
  \begin{align}
    Y_{i,j} \condind_{(X_{i,j},\augc)} \Xntwo \;.
  \end{align}
  By \cref{lem:noise:out:suff}, there is a measurable function $f_{i,j} : [0,1] \times \calX \times \augcspace_{\bfn_2} \to \calY$ such that
  \begin{align*}
    Y_{i,j} = f_{i,j}(\noise_{i,j},X_{i,j},\augc) \;,
  \end{align*} 
  for a uniform random variable $\noise_{i,j} \condind \Xntwo$. 
  This is true for all $i\in[n_1]$ and $j\in[n_2]$; by equivariance the same $f_{i,j}$ must work for every $(i,j)$. Furthermore, by assumption the elements of $\Yntwo$ are mutually conditionally independent given $\Xntwo$, and therefore by the chain rule for conditional independence \citep[][Prop.~6.8]{Kallenberg:2002}, the joint identity \eqref{eq:rep:exch:two:array} holds. 

  The reverse direction is straightforward to verify.
\end{proof}

A particularly simple version of \eqref{eq:rep:exch:two:array} is
\begin{align} \label{eq:exch:two:array:simple}
  Y_{i,j} = f\big(\noise_{i,j},
    {X}_{i,j},
    \sum_{k=1}^{n_2} h_1({X}_{i,k}),
    \sum_{\ell=1}^{n_1} h_2({X}_{\ell,j}),
    \sum_{k,\ell} h_3({X}_{\ell,k})
    \big) \;,
\end{align}
for some functions $h_m : \calX \to \bbR$, $m\in\{1,2,3\}$. 
Clearly, this is conditionally equivariant. The following example from the deep learning literature is an even simpler version.

\begin{example}[Array-based MLPs]
   \Citet{Hartford:etal:2018} determined the parameter-sharing schemes that result from deterministic permutation-equivariant MLP layers for matrices. They call such layers ``exchangeable matrix layers''.\footnote{This is a misnomer: exchangeability is a distributional property and there is nothing random.} The equivariant weight-sharing scheme yields a simple expression:
   \begin{align*}
    Y_{i,j} = \activation\bigg( 
      \theta_0 +  
      \theta_1 X_{i,j} + 
      \theta_2 \sum_{k=1}^{n_2} X_{i,k} +  
      \theta_3 \sum_{\ell=1}^{n_1} X_{\ell,j} +  
      \theta_4 \sum_{k,\ell} X_{\ell,k}  
      \bigg) \;.
   \end{align*}
   That is, $Y_{i,j}$ is a nonlinear activation of a linear combination of $X_{i,j}$, the sums of the $j$th column and $i$th row of $\Xntwo$, and the sum of the entire matrix $\Xntwo$. It is straightforward to see that this is a special deterministic case of \eqref{eq:exch:two:array:simple}. \Citet{Hartford:etal:2018} also derive analogous weight-sharing schemes  for MLPs for $d$-dimensional arrays; those correspond with the $d$-dimensional version of \cref{thm:equivariant:two:arrays} (see \cref{sec:proof:arrays}).
\end{example}

\paragraph{Jointly exchangeable arrays} The case of jointly exchangeable symmetric arrays is of particular interest because it applies to graph-valued data. Importantly, the edge variables are not restricted to be $\{0,1\}$-valued; in practice they often take values in $\bbR_+$. 

Let $\Xntwosymm$ be a jointly exchangeable matrix in $\calX^{n\times n}$, and $\Cntwosymm$ its canonical form. Similarly to \eqref{eq:sep:aug:cboard}, define the \emph{jointly augmented canonical form} for a symmetric array as
\begin{align} \label{eq:joint:aug:cboard}
  [\augcsymm]_{k,\ell} = 
    ([\Cntwosymm]_{k,\ell},\{ ([\Cntwosymm]_{i,k},[\Cntwosymm]_{j,k}), ([\Cntwosymm]_{i,\ell}, [\Cntwosymm]_{j,\ell}) \}) \;.
\end{align}
The symmetry of $\Xntwosymm$ requires that the row and column entries be paired, which results in the second, set-valued, element on the right-hand side of \eqref{eq:joint:aug:cboard}. 
(The curly braces indicate a set that is insensitive to the order of its elements, as opposed to parentheses, which indicate a sequence that is sensitive to order.) 
Denote by $\augcspacesymm_{\bfn}(\calX)$ the space of all such jointly augmented canonical forms on $\calX^{n \times n}$. 
The following counterpart of \cref{thm:equivariant:two:arrays} applies to jointly exchangeable arrays such as undirected graphs.

\begin{theorem} \label{thm:joint:equivariant:two:arrays}
  Suppose $\Xntwosymm$ and $\Yntwosymm$ are symmetric $\calX$- and $\calY$-valued arrays, respectively, each indexed by $[n]\times [n]$, and that $\Xntwosymm$ is jointly exchangeable. Assume that the elements of $\Yntwosymm$ are mutually conditionally independent given $\Xntwosymm$. Then $\Yntwosymm$ is conditionally $\bbS_{n}$-equivariant given $\Xntwosymm$ if and only if there is a measurable function $f : [0,1] \times \calX \times \augcspacesymm_{\bfn}(\calX) \to \calY$ such that 
  \begin{align} \label{eq:rep:joint:exch:two:array}
  \big(\Xntwosymm,\Yntwosymm \big)
    \equas
    \bigg(\Xntwosymm,
      \big(
      f(
        \noise_{i,j},
        \widebar{X}_{i,j},
        \augcboardsymm_{\Xntwosymm}^{\{i,j\}}
      ) 
      \big)_{i\in [n],j\in [n]}
    \bigg) \;,
  \end{align}
  for \iid uniform random variables $(\noise_{i,j})_{i\in [n],j\leq i} \condind \Xntwosymm$ with $\noise_{i,j}=\noise_{j,i}$. 
\end{theorem}
\begin{proof}
  The proof is essentially the same as that of \cref{thm:equivariant:two:arrays}. The main difference stems from the symmetry of $\Xntwosymm$: fixing row $i$ also fixes column $i$, and fixing column $j$ also fixes row $j$. Hence, the augmentations to the matrix consist of the paired sequences $(\Xntwosymmi,\Xntwosymmj)=((\widebar{X}_{i,k},\widebar{X}_{j,k}))_{k\in [n]}$. The counterpart to $\bfZ^{(i,j)}$ in \eqref{eq:z:array} is
  \begin{align*}
    [\widebar{\bfZ}^{\{i,j\}}]_{k,\ell} = 
      (\widebar{X}_{k,\ell},\{ (\widebar{X}_{i,\ell}, \widebar{X}_{j,\ell}), (\widebar{X}_{k,i},\widebar{X}_{k,j}) \}) \;, \quad k,\ell \in [n-1] \;.
  \end{align*}
  It is straightforward to show that $\widebar{\bfZ}^{\{i,j\}}$ is jointly exchangeable; the rest of the argument is the same as in the proof of \cref{thm:equivariant:two:arrays}.
\end{proof}

\subsection{Vertex Features: Equivariant Functions on the Vertex Set}
\label{sec:vertex:features}

In many applications of neural networks to graph-structured data, the desired output is a function defined on the vertex set, which can be thought of as a collection of vertex features. Additionally, the input graph may include vertex features. Formally, these features are encoded as a sequence $\Xn$ whose elements correspond to the rows and columns of a symmetric matrix $\Xntwosymm$. We assume for simplicity that both $\Xn$ and $\Xntwosymm$ are $\calX$-valued, but it is not strictly necessary that they take values in the same domain. In the separately exchangeable case, two distinct feature sequences $\Xnfeati$ and $\Xnfeatj$ correspond to rows and columns, respectively. For simplicity, we focus on the symmetric case. A permutation $\perm\in\bbS_n$ acts simultaneously on $\Xn$ and $\Xntwosymm$, and the feature-augmented graph is jointly exchangeable if
\begin{align} \label{eq:vertex:aug:exch}
  (\perm \cdot \Xn, \perm \cdot \Xntwosymm) \equdist (\Xn, \Xntwosymm) \;, \quad \text{ for all } \perm \in \bbS_n \;.
\end{align}
Likewise, we may define a vertex feature sequence $\Yn$, whose elements correspond to the rows and columns of $\Yntwosymm$. 

Consider the array $\bcast(\Xn,\Xntwosymm)$ obtained by broadcasting the vertex features over the appropriate dimensions of $\Xntwosymm$: $[\bcast(\Xn,\Xntwosymm)]_{i,j} = (\widebar{X}_{i,j}, \{X_i, X_j \})$. In this case, we denote the canonical form by $\Cnfeat$. 
Observe that $\bcast(\Xn,\Xntwosymm)$ is a jointly exchangeable array because $(\Xn,\Xntwosymm)$ is exchangeable as in \eqref{eq:vertex:aug:exch}. Furthermore, if $\Yntwosymm$ is conditionally $\bbS_n$-equivariant given $\Xntwosymm$, then $\bcast(\Yn,\Yntwosymm)$, with entries $[\bcast(\Yn,\Yntwosymm)]_{i,j} = (\widebar{Y}_{i,j},\{Y_i,Y_j\})$, is conditionally $\bbS_n$-equivariant given $\bcast(\Xn,\Xntwosymm)$. 
However, \cref{thm:joint:equivariant:two:arrays} does not apply: $\bcast(\Yn,\Yntwosymm)$ does not satisfy the mutually conditionally independent elements assumption because, for example, $Y_i$ appears in every element in the $i$th row.

Further assumptions are required to obtain a useful functional representation of $\Yntwosymm$. In addition to the mutual conditional independence of the elements of $\Yntwosymm$ given $(\Xn,\Xntwosymm)$, we assume further that
\begin{align} \label{eq:vtx:feature:cond:ind}
  Y_i \condind_{(\Xn,\Xntwosymm,\Yntwosymm)} (\Yn \setminus Y_i) \;, \quad i \in [n] \;.
\end{align}
In words, the elements of the output vertex feature sequence are conditionally independent, given the input data $\Xn,\Xntwosymm$ and the output edge features $\Yntwosymm$. 
This is consistent with the implicit assumptions used in practice in the deep learning literature, and leads to a representation with simple structure.

To state the result, let $\Cnfeat$ be the augmented canonical form that includes the vertex features, 
\begin{align*}
  [\Cnfeataug]_{k,\ell}
  =
  (\widebar{X}_{k,\ell}, \{ X_k, X_{\ell} \},\{ (\widebar{X}_{i,\ell}, \widebar{X}_{j,\ell}), (\widebar{X}_{k,i},\widebar{X}_{k,j}) \}) \;,
\end{align*}
which belongs to the space denoted $\augcspace_{\bcast,\bfn}(\calX)$. 
Let $\stak(\Xntwosymm,\Yntwosymm)$ be the symmetric (stacked) array with entries $(\widebar{X}_{i,j},\widebar{Y}_{i,j})$, and $\bcast(\Xn,\stak(\Xntwosymm,\Yntwosymm))$ be the same, with the entries of $\Xn$ broadcast. 

\begin{theorem} \label{thm:joint:equivariant:aug:two:arrays}
  Suppose $(\Xn,\Xntwosymm)$ and $(\Yn,\Yntwosymm)$ are $\calX$- and $\calY$-valued vertex feature-augmented arrays, and that $(\Xn,\Xntwosymm)$ is jointly exchangeable as in \eqref{eq:vertex:aug:exch}. Assume that the elements of $\Yntwosymm$ are mutually conditionally independent given $(\Xn,\Xntwosymm)$, and that $\Yn$ satisfies \eqref{eq:vtx:feature:cond:ind}. Then $(\Yn,\Yntwosymm)$ is conditionally $\bbS_n$-equivariant given $(\Xn,\Xntwosymm)$ if and only if there are measurable functions $f_e : [0,1] \times \calX \times \widebar{\calX}^2 \times \augcspacesymm_{\bcast,\bfn}(\calX) \to \calY$ and $f_v : [0,1] \times \calX \times \calX \times \calY \times \augcspacesymm_{\bcast,\bfn}(\calX\times\calY) \to \calY$ such that
  \begin{align}
    \widebar{Y}_{i,j} & \equas f_e \big( \noise_{i,j}, \widebar{X}_{i,j}, \{ X_i, X_j  \}, \Cnfeataug \big) \;, \quad i,j \in [n] \label{eq:aug:edge:rep}  \\ 
    Y_i & \equas f_v \big( \noise_i,X_i,\widebar{X}_{i,i},\widebar{Y}_{i,i}, \cboard_{\bcast(\Xn,\stak(\Xntwosymm,\Yntwosymm))}^{\{i\}} \big) \;, \quad i \in [n] \label{eq:aug:vtx:feature:rep} \;,
  \end{align}
  for \iid uniform random variables $(\noise_{i,j})_{i\in [n], j \leq i} \condind (\Xn,\Xntwosymm)$ with $\noise_{i,j}=\noise_{j,i}$, and \\ $(\noise_i)_{i\in [n]} \condind (\Xn,\Xntwosymm,\Yntwosymm)$. 
\end{theorem}
\begin{proof}
  The proof, like that of \cref{thm:joint:equivariant:two:arrays}, is essentially the same as for \cref{thm:equivariant:two:arrays}. Incorporating vertex features requires that for any permutation $\perm\in\Symn$, the fixed elements of $\Xn$ be collected along with the fixed rows and columns of $\Xntwosymm$; the structure of the argument is identical.
\end{proof}

Equations \eqref{eq:aug:edge:rep} and \eqref{eq:aug:vtx:feature:rep} indicate that given an input $(\Xn,\Xntwosymm)$ and functional forms for $f_e$ and $f_v$, computation of $\Yntwosymm$ and $\Yn$ proceeds in two steps: first, compute the elements $\widetilde{Y}_{i,j}$ of $\Yntwosymm$; second, compute the vertex features $\Yn$ from $\Xn$, $\Xntwosymm$, and $\Yntwosymm$. Note that within each step, computations can be parallelized due to the conditional independence assumptions. 

The following examples from the literature are special cases of \cref{thm:joint:equivariant:aug:two:arrays}. 

\begin{example}[Graph-based structured prediction]
  \Citet{Herzig:etal:2018} considered the problem of deterministic permutation-equivariant structured prediction in the context of mapping images to scene graphs. In particular, for a weighted graph with edge features $\Xntwosymm$ and vertex features $\Xn$, those authors define a graph labeling function $f$ to be ``graph-permutation invariant'' (GPI) if $f(\perm\cdot(\Xn,\Xntwosymm)) = \perm\cdot f(\Xn,\Xntwosymm)$ for all $\perm\in\bbS_n$.\footnote{Although \citet{Herzig:etal:2018} call this invariance, we note that it is actually equivariance.} Furthermore, they implicitly set $\Yntwosymm = \Xntwosymm$, and assume that $\widetilde{X}_{i,i}=0$ for all $i\in [n]$ and that $\Xn$ is included in $\Xntwosymm$ (e.g., on the diagonal). The main theoretical result is that a graph labeling function is GPI if and only if 
  \begin{align*}
    [f(\Xn,\Xntwosymm)]_i = \rho(X_i, \textstyle\sum_j \alpha(X_j, \textstyle\sum_{k}\phi(X_j,\widetilde{X}_{j,k},X_k))) \;,
  \end{align*}
  for appropriately defined functions $\rho$, $\alpha$, and $\phi$. Inspection of the proof reveals that the second argument of $\rho$ (the sum of $\alpha$ functions) 
  is equivalent to a maximal invariant. In experiments, a particular GPI neural network architecture showed better sample efficiency for training, as compared with an LSTM with the inputs in random order and with a fully connected feed-forward network, each network with the same number of parameters.
\end{example}

\begin{example}[Message passing graph neural networks]
  \Citet{Gilmer:etal:2017} reviewed the recent literature on graph-based neural network architectures, and found that many of them fit in the framework of so-called message passing neural networks (MPNNs). MPNNs take as input a graph with vertex features and edge features (the features may be vector-valued, but for simplicity of notation we assume they are scalar-valued real numbers). Each neural network layer $\ell$ acts as a round of message passing between adjacent vertices, with the typically edge-features held fixed from layer to layer. In particular, denote the (fixed) input edge features by $\Xntwosymm$ and the computed vertex features of layer $\ell-1$ by $\Xn$; then the vertex features for layer $\ell$, $\Yn$, are computed as
  \begin{align*}
    Y_{i} = U_{\ell}\big( 
      X_i, 
      \textstyle\sum_{j\in [n]} \indicator_{\{ \widetilde{X}_{i,j} > 0 \}} M_{\ell}(X_i, X_j, \widetilde{X}_{i,j})
    \big) \;, \quad i \in [n] \;,
  \end{align*}
  for some functions $M_{\ell} : \bbR^3 \to \bbR$ and $U_{\ell} : \bbR^2 \to \bbR$. This is a deterministic special case of \eqref{eq:aug:vtx:feature:rep}. \Citet{Gilmer:etal:2017} note that \citet{Kearnes:etal:2016} also compute different edge features, $\widetilde{Y}_{i,j}$, in each layer. The updated edge features are used in place of $\widetilde{X}_{i,j}$ in the third argument of $M_{\ell}$ in the equation above, and are an example of the two-step implementation of \eqref{eq:aug:edge:rep}-\eqref{eq:aug:vtx:feature:rep}. 
\end{example}

\paragraph{Separate exchangeability with features} 
Separately exchangeable matrices with features, which may be regarded as representing a bipartite graph with vertex features, have a similar representation. Let $\Xnfeati$ and $\Xnfeatj$ denote the row and column features, respectively. Separate exchangeability for such data structures is defined in the obvious way. Furthermore, a similar conditional independence assumption is made for the output vertex features: $\Ynfeati \condind_{(\Xnfeati,\Xnfeatj,\Xntwo,\Yntwo)} \Ynfeatj$ and
\begin{align} \label{eq:sep:cond:ind}
  Y_{i}^{(m)} \condind_{(\Xnfeati,\Xnfeatj,\Xntwo,\Yntwo)} (\bfY_n^{(m)} \setminus Y_{i}^{(m)}) \;, \quad i\in [n_m], m\in\{1,2\} \;.
\end{align}
The representation result is stated here for completeness. To state the result, note that appropriate broadcasting of features requires that $\Xnfeati$ be broadcast over columns $\Xntwoj$, and $\Xnfeatj$ be broadcast over rows $\Xntwoi$. We denote this with the same broadcasting operator as above, $\bcast(\Xnfeati,\Xnfeatj,\Xntwo)$. 
The proof is similar to that of \cref{thm:joint:equivariant:aug:two:arrays}, and is omitted for brevity. 

\begin{theorem} \label{thm:sep:equivariant:aug:two:arrays}
  Suppose $(\Xnfeati,\Xnfeatj,\Xntwo)$ and $(\Ynfeati,\Ynfeatj,\Yntwo)$ are $\calX$- and $\calY$-valued vertex feature-augmented arrays, and that $(\Xnfeati,\Xnfeatj,\Xntwo)$ is separately exchangeable. Assume that the elements of $\Yntwo$ are mutually conditionally independent given $(\Xnfeati,\Xnfeatj,\Xntwo)$, and that $\Ynfeati$ and $\Ynfeatj$ satisfy \eqref{eq:sep:cond:ind}. Then $(\Ynfeati,\Ynfeatj,\Yntwo)$ is conditionally $\bbS_{\bfn_2}$-equivariant given $(\Xnfeati,\Xnfeatj,\Xntwo)$ if and only if there are measurable functions $f_e : [0,1] \times \calX^3 \times \augcspace_{\bcast,\bfn_2}(\calX) \to \calY$ and $f^{(m)}_v : [0,1] \times \calX \times \calX^{n_m} \times \calY^{n_m} \times \augcspace_{\bcast,\bfn_2}(\calX\times\calY) \to \calY$, for $m\in\{1,2\}$, such that
  \begin{align}
    {Y}_{i,j} & \equas f_e \big( \noise_{i,j}, {X}_{i,j}, X^{(1)}_i, X^{(2)}_j, \cboard_{\bcast\big(\stak(\Xnfeati,\Xntwoj),\stak(\Xnfeatj,\Xntwoi),\Xntwo\big)} \big) \;, \quad i \in [n_1], j \in [n_2] \label{eq:aug:sep:edge:rep}  \\ 
    Y_i^{(1)} & \equas f^{(1)}_v \big( \noise^{(1)}_i,X^{(1)}_i, \Xntwoi, \Yntwoi, \cboard_{\bcast\big(\Xnfeati,\stak(\Xnfeatj,\Xntwoi),\stak(\Xntwo,\Yntwo)\big)} \big) \;, \quad i \in [n_1] \label{eq:aug:sep:row:feature} \;, \\
    Y_j^{(2)} & \equas f^{(2)}_v \big( \noise^{(2)}_j,X^{(2)}_j, \Xntwoj, \Yntwoj, \cboard_{\bcast\big(\stak(\Xnfeati,\Xntwoj),\Xnfeatj,\stak(\Xntwo,\Yntwo)\big)} \big) \;, \quad j \in [n_2] \label{eq:aug:sep:col:feature} \;,
  \end{align}
  for \iid uniform random variables $((\noise_{i,j})_{i\in [n_1], j\in[n_2]}) \condind (\Xnfeati,\Xnfeatj,\Xntwo)$ and \\ $((\noise^{(1)}_i)_{i\in [n_1]},(\noise^{(2)}_j)_{j\in [n_2]}) \condind (\Xnfeati,\Xnfeatj,\Xntwo,\Yntwo)$. 
\end{theorem}

\section{Discussion}
\label{sec:discussion}

The probabilistic approach to symmetry has allowed us to draw on tools from an area that is typically outside the purview of the deep learning community. Those tools shed light on the underlying structure of previous work on invariant neural networks, and expand the scope of what is possible with such networks. Moreover, those tools place invariant neural networks in a broader statistical context, making connections to the fundamental concepts of sufficiency and adequacy. 

To conclude, we give some examples of other data structures to which the theory developed in the present work could be applied and describe some questions prompted by this work and by others.

\subsection{Other Exchangeable Structures}
\label{sec:other}

The results in \cref{sec:finite:exch:seq} can be adapted to other exchangeable structures in a straightforward manner. We briefly describe two settings; \citet{Orbanz:Roy:2015} survey a number of other exchangeable structures from statistics and machine learning to which this work could apply. 

\paragraph{Edge-exchangeable graphs} 
\Citet{Cai:etal:2016,Williamson:2016,Crane:Dempsey:2018} specified generative models for network data as an exchangeable sequence of edges, $\mathbf{E}_n = ((u,v)_1,\dotsc,(u,v)_n)$. The sequence theory from \cref{sec:finite:exch:seq} applies, rather than the array theory from \cref{sec:finite:exch:arrays}, which applies to vertex-exchangeable graphs. However, incorporating vertex features into edge-exchangeable models would require some extra work, as a permutation acting on the edge sequence has a different (random) induced action on the vertices; we leave it as an interesting problem for future work.

The model of \citet{Caron:Fox:2017} is finitely edge-exchangeable when conditioned on the random number of edges in the network. Therefore, the sequence theory could be incorporated into inference procedures for that model, with the caveat that the neural network architecture would be required to accept inputs of variable size and the discussion from \cref{sec:var:size} would be relevant.

\paragraph{Markov chains and recurrent processes} 
A Markov chain $\bfZ_n:=(Z_1,Z_2,\dotsc,Z_n)$ on state-space $\calZ$ has exchangeable sub-structure. In particular, define a \emph{$z$-block} as a sub-sequence of $\Zn$ that starts and ends on some state $z \in \calZ$. Clearly, the joint probability
\begin{align*}
  P_{\Zn}(\Zn) = P_{Z_1}(Z_1) \prod_{i=2}^n Q(Z_i \mid Z_{i-1})
\end{align*}
is invariant under all permutations of $Z_1$-blocks (and of any other $z$-blocks for $z\in \Zn$). Denote these blocks by $(B_{Z_1,j})_{j\geq 1}$. \Citet{Diaconis:Freedman:1980a} used this notion of \emph{Markov exchangeability} to show that all recurrent processes on countable state-spaces are mixtures of Markov chains, and \citet{Bacallado:Favaro:Trippa:2013} used similar ideas to analyze reversible Markov chains on uncountable state-spaces.

For a finite Markov chain, the initial state, $Z_1$, and the $Z_1$-blocks are sufficient statistics. Equivalently, $Z_1$ and the empirical measure of the $m \leq n$ $Z_1$-blocks,
\begin{align*}
  \empirical_{\Zn}(\argdot) = \sum_{j=1}^m \delta_{B_{Z_1,j}}(\argdot) \;,
\end{align*}
plays the same role as $\empirical_{\Xn}$ for an exchangeable sequence. (If $\Zn$ is the prefix of an infinite, recurrent process, then $Z_1$ and the empirical measure of transitions are sufficient.) It is clear that the theory from \cref{sec:finite:exch:seq} can be adapted to accommodate Markov chains; indeed, \citet{Wiqvist:etal:2019} used this idea to learn summary statistics from a Markov chain for use in Approximate Bayesian Computation.

\subsection{Open Questions}
\label{sec:open:questions}

\paragraph{More flexible conditional independence assumptions} 
The simplicity of results on functional representations of equivariant conditional distributions relied on conditional independence assumptions like $Y_i \condind_{\Xn} (\Yn \setminus Y_i)$. As discussed in \cref{sec:equivariant:seq}, additional flexibility could be obtained by assuming the existence of a random variable $W$ such that $Y_i \condind_{(\Xn, W)} (\Yn \setminus Y_i)$ holds, and learning $W$ for each layer. $W$ could be interpreted as providing additional shared ``context'' for the input-output relationship, a construct that has been useful in recent work \citep{Edwards:Storkey:2017,Kim:etal:2019}.

\paragraph{Choice of pooling functions}
The results here are quite general, and say nothing about what other properties the function classes under consideration should have. For example, different symmetric pooling functions may lead to very different performance. This seems to be well understood in practice; for example, the pooling operation of element-wise mean of the top decile in \citet{Chan:etal:2018} is somewhat unconventional, but appears to have been chosen based on performance. Recent theoretical work by \citet{Xu:etal:2019} studies different pooling operations in the context of graph discrimination tasks with graph neural networks, a problem also considered from a slightly different perspective by \citet{ShaweTaylor:1993}. \Citet{Wagstaff:2019aa} take up the problem for functions defined on sets.

\paragraph{Learning and generalization} 
Our results leave open questions pertaining to learning and generalization. However, they point to some potential approaches, one of which we sketch here. \Citet{ShaweTaylor:1991,ShaweTaylor:1995} applied PAC theory to derive generalization bounds for feed-forward networks that employ a symmetry-induced weight-sharing scheme; these bounds are improvements on those for standard fully-connected multi-layer perceptrons. Weight-sharing might be viewed as a form of pre-training compression; recent work uses PAC-Bayes theory to demonstrate the benefits of compressibility of trained networks \citep{Arora:etal:2018,Zhou:etal:2019}, and \citet{Lyle:etal:2020} make some initial progress along these lines.

\section*{Acknowledgments}
The authors are grateful to the editor and three anonymous referees for detailed comments that helped improve the paper. The authors also thank Bruno Ribeiro for pointing out an error, and to Hongseok Yang for pointing out typos and an unnecessary condition for \cref{thm:kall:equiv}, in an earlier draft. BBR's and YWT's research leading to these results has received funding from the European Research Council under the European Union's Seventh Framework Programme (FP7/2007-2013) ERC grant agreement no.\ 617071.

\newpage


\crefalias{section}{appsec} 

\begin{appendix}

\section{Proof of Lemma 5}
\label{sec:proof:noise:out:suff}

d-separation is a statement of conditional independence and therefore the proof of \cref{lem:noise:out:suff} is a straightforward application of the following basic result about conditional independence, which appears (with different notation) as Proposition 6.13 in \citet{Kallenberg:2002}.
\begin{lemma}[Conditional independence and randomization] \label{lem:cond:ind}
  Let $X,Y,Z$ be random elements in some measurable spaces $\calX,\calY,\calZ$, respectively, where $\calY$ is Borel. Then $Y\condind_{Z} X$ if and only if $Y \equas f(\noise,Z)$ for some measurable function $f : [0,1] \times \calZ \to \calY$ and some uniform random variable $\noise\condind (X,Z)$.
\end{lemma}

\section{Orbit Laws, Maximal Invariants, and Representative Equivariants: Proofs for Section 4}
\label{sec:technical}

The results in \cref{sec:connecting:symmetries} (\cref{thm:group:invariant:rep,thm:kall:equiv,thm:suff:maximal:invariant}) rely on the disintegration of $\grp$-invariant distributions $P_X$ into a distribution over orbits and a distribution on each orbit generated by the normalized Haar measure (the orbit law). This disintegration, established in \cref{lem:maximal:invariant}, is the mathematically precise version of the intuitive explanation given in \cref{sec:connecting:symmetries}, just before \cref{thm:group:invariant:rep}. We apply the result, along with a key conditional independence property that it implies, in order to obtain \cref{thm:group:invariant:rep,thm:kall:equiv,thm:suff:maximal:invariant}.

We require the following definition, taken from \citet{Eaton:1989}.
\begin{definition}
  A \emph{measurable cross-section} is a set $\calC\subset\calX$ with the following properties:
  \begin{enumerate}[label=(\roman*)]
    \item $\calC$ is measurable.
    \item For each $x\in\calX$, $\calC \cap \grp\cdot x$ consists of exactly one point, say $\calC_x$.
    \item The function $t : \calX \to \calC$ defined by $t(x) = \calC_x$ is $\borel_{\calX}$-measurable when $\calC$ has $\sigma$-algebra $\borel_{\calC} = \{ B \cap \calC | B \in \borel_{\calX} \}$.
  \end{enumerate}
\end{definition}

One may think of a measurable cross-section as consisting of a single representative from each orbit of $\calX$ under $\grp$. For a compact group acting measurably on a Borel space, there always exists a (not necessarily unique) measurable cross-section. For non-compact groups or more general spaces, the existence of a measurable cross-section as defined above is not guaranteed, and more powerful technical tools are required \citep[e.g.][]{Andersson:1982,Schindler:2003,Kallenberg:2017}. 

\paragraph{Orbit laws} 
Intuitively, conditioned on $X$ being on a particular orbit, it should be possible to sample a realization of $X$ by first sampling a random group element $G$, and then applying $G$ to a representative element of the orbit. 
More precisely, let $\normhaar$ be the normalized Haar measure of $\grp$, which is the unique left- and right-invariant measure on $\grp$ such that $\normhaar(\grp)=1$.  
Let ${M : \calX \to \calS}$ be a maximal invariant. For any $m\in \text{range}(M)$, let $c_{m} = M^{-1}(m)$ be the element of $\calC$ for which $M(c_m) = m$; such an element always exists. Note that the set of elements in $\calS$ that do not correspond to an orbit of $\calX$, $\calS_{\emptyset} := \{ m ; M^{-1}(m) = \emptyset  \}$, has measure zero under $P_X\circ M^{-1}$. 

For any $B\in\borel_{\calX}$ and $m\in\calS\setminus\calS_{\emptyset}$, define the \emph{orbit law} as
\begin{align} \label{eq:orbit:law}
  \orbitlaw_{m}(B) = \int_{\grp} \delta_{g\cdot c_{m}}( B ) \normhaar(dg) = \normhaar (\{ g ; g\cdot c_m \in B \} ) \;.
\end{align}
(For $m\in\calS_{\emptyset}$, $\orbitlaw_{m}(B)=0$.) 
Observe that, in agreement with the intuition above, a sample from the orbit law can be generated by sampling a random group element $G\sim\normhaar$, and applying it to $c_{m}$. The orbit law inherits the invariance of $\normhaar$ and acts like the uniform distribution on the elements of the orbit, up to fixed points.\footnote{The orbit law only coincides with the uniform distribution on the orbit if the action of $\grp$ on the orbit is \emph{regular}, which corresponds to the action being transitive and \emph{free} (or fixed-point free). Since by definition the action of $\grp$ is transitive on each orbit, the orbit law is equivalent to the uniform distribution on the orbit if and only if $\grp$ acts freely on each orbit; if the orbit has any fixed points (i.e., $g\cdot x = x$ for some $x$ in the orbit and $g\neq e$), then the fixed points will have higher probability mass under the orbit law.\label{fn:orbit:uniform}} 
For any integrable function $\varphi: \calX \to \bbR$, the expectation with respect to the orbit law is
\begin{align}
  \orbitlaw_{m}[\varphi] = \int_{\calX} \varphi(x) \orbitlaw_{m}(dx) 
    = \int_{\grp} \varphi(g \cdot c_{m}) \normhaar(dg)  \;.
\end{align} 
The orbit law arises in the disintegration of any $\grp$-invariant distribution $P_X$ as the fixed distribution on each orbit.
\begin{lemma} \label{lem:maximal:invariant}
  Let $\calX$ and $\calS$ be Borel spaces, $\grp$ a compact group acting measurably on $\calX$, and $M : \calX \to \calS$ a maximal invariant on $\calX$ under $\grp$. If $X$ is a random element of $\calX$, then its distribution $P_X$ is $\grp$-invariant if and only if
  \begin{align} \label{eq:suff:maximal:invariant}
    P_X(X \in \argdot \mid M(X) = m) = \orbitlaw_{m}(\argdot) = q(\argdot,m) \;,
  \end{align}
  for some Markov kernel $q : \borel_{\calX} \times \calS \to \bbR_+$. 
  If $P_X$ is $\grp$-invariant and $Y$ is any other random variable, then $P_{Y|X}$ is $\grp$-invariant if and only if $Y \condind_{M(X)} X$.
\end{lemma}

\begin{proof} 
  For the first claim, a similar result appears in \citet[][Ch.\ 4]{Eaton:1989}. For completeness, a slightly different proof is given here. 
  Suppose that $P_X$ is $\grp$-invariant. Then $(g\cdot X,M) \equdist (X,M)$ for all $g\in\grp$. By Fubini's theorem, the statement is true even for a random group element $G \condind X$. Let ${G}$ be a random element of $\grp$, sampled from $\normhaar$. Then for any measurable function $\varphi : \calX \to \bbR_+$, and for any set $A\in\sigma(M)$, 
  \begin{align*}
    \bbE[\varphi(X); A] & = \bbE[\varphi({G}\cdot X); A] = \bbE[ \bbE[ \varphi({G}\cdot X) \mid X ] ; A  ]  = \bbE[ \orbitlaw_{M(X)}[\varphi]; A ] \;,
  \end{align*}
  which establishes the first equality of \eqref{eq:suff:maximal:invariant}; the second equality follows from \citet[][Thm.~6.3]{Kallenberg:2002} because $\calS$ is a Borel space, as is $\calX$. 
  Conversely, suppose that \eqref{eq:suff:maximal:invariant} is true. $P_X(X\in\argdot) = \bbE[\orbitlaw_{M}(\argdot)]$, where the expectation is taken with respect to the distribution of $M$. $\orbitlaw_{M}(\argdot)$ is $\grp$-invariant for all $M\in\calS$, and therefore so is the marginal distribution $P_X$.

  For the second claim, suppose $Y$ is such that $P_{Y|X}$ is $\grp$-invariant, which by \cref{prop:joint:symmetry} is equivalent to $(g\cdot X, Y) \equdist (X,Y)$ for all $g\in\grp$. The latter is equivalent to $g\cdot(X,Y)\equdist (X,Y)$, with $g\cdot Y = Y$ almost surely, for all $g\in\grp$. Therefore, $\grp\cdot Y = \{Y\}$ almost surely, and $\widetilde{M}(X,Y) := (M(X), Y)$ is a maximal invariant of $\grp$ acting on $\calX\times\calY$. Therefore, $\orbitlaw_{\widetilde{M}(X,Y)}(A\times B) = \orbitlaw_{M(X)}(A) \delta_Y(B)$, for $A \in \borel_{\calX}$ and $B\in\borel_{\calY}$, and 
  \begin{align*}
    P_X(X \in \argdot \mid M(X) = m, Y) = \orbitlaw_{m}(\argdot) = P_X(X\in\argdot \mid M(X) = m) \;,
  \end{align*}
  which implies $Y \condind_{M(X)} X$. The converse is straightforward to check.
\end{proof}

Note that \eqref{eq:suff:maximal:invariant} is equivalent to the disintegration
\begin{align*}
  P_X( X \in \argdot ) = \int_{\calS} \orbitlaw_m(\argdot) \nu(dm) = \int_{\calC} \orbitlaw_{M(x)}(\argdot) \mu(dx) \;,
\end{align*}
for some maximal invariant $M : \calX \to \calS$ and probability measures $\nu$ on $\calS$ and $\mu$ on the measurable cross-section $\calC$. This type of statement is more commonly encountered \citep[e.g.,][Ch.\ 4-5]{Eaton:1989} than $\eqref{eq:suff:maximal:invariant}$, but the form of \eqref{eq:suff:maximal:invariant} emphasizes the role of the orbit law as the conditional distribution of $X$ given $M(X)$, which is central to the development of ideas in \cref{sec:sufficiency:adequacy}.

\paragraph{Example: orbit law of an exchangeable sequence}
The orbit law of an exchangeable sequence $\Xn$ is also known as the, \emph{urn law}. It is defined as a probability measure on $\calX^n$ corresponding to permuting the elements of $\empirical_{\Xn}$ according to a uniformly sampled random permutation:
\begin{align} \label{eq:urn:law}
  \urnlaw_{\empirical_{\Xn}}(\argdot)  = \frac{1}{n!} \sum_{\perm\in\Symn} \delta_{\perm\cdot \Xn}(\argdot) \;.
\end{align}
The urn law is so called because it computes the probability of sampling any sequence without replacement from an urn with $n$ balls, each labeled by an element of $\Xn$.\footnote{The metaphor is that of an urn with $n$ balls, each labeled by an element of $\empirical_{\Xn}$; a sequence is constructed by repeatedly picking a ball uniformly at random from the urn, without replacement. Variations of such a scheme can be considered, for example sampling with replacement, or replacing a sampled ball with two of the same label. See \cite{Mahmoud:2008} for an overview of the extensive probability literature studying such processes, including the classical P\'olya urn and its generalizations \citep[e.g.,][]{Blackwell:MacQueen:1973}, which play important roles in Bayesian nonparametrics.}  
\Citet{Diaconis:Freedman:1980} and \citet[][Proposition 1.8]{Kallenberg:2005} prove results (in different forms) similar to \cref{thm:finite:exch:seq:invariant} that rely on the urn law.

\subsection{Invariant Conditional Distributions: Proof of \texorpdfstring{\cref{thm:group:invariant:rep}}{Theorem 7}}

\groupinvariantrep*

\begin{proof} 
  With the conditional independence relationship $Y\condind_{M(X)} X$ established in \cref{lem:maximal:invariant}, \cref{thm:group:invariant:rep} follows from \cref{lem:noise:out:suff}.
\end{proof}

\subsection{Representative Equivariants and Proof of \texorpdfstring{\cref{thm:kall:equiv}}{Theorem 9}}
\label{sec:proof:kall:equiv}

Recall that a representative equivariant is an equivariant function $\tau : \calX \to \grp$ satisfying
\begin{align*}
  \tau(g\cdot x) = g\cdot\tau(x) \quad \text{for each } g\in\grp \;.
\end{align*}
Recall also that $\tau_x := \tau(x)$ and that $\tau_x^{-1}$ is the inverse element in $\grp$ of $\tau_x$. The following intermediate result (restated from \cref{sec:connecting:symmetries}) establishes two useful properties of $\tau$, which are used in the proof of \cref{thm:kall:equiv}.

\tauproperties*

\begin{proof} 
  Property (i) is proved in \citet[][Ch.\ 2]{Eaton:1989}, as follows. Observe that for all $g\in\grp$ and $x\in\calX$,
  \begin{align*}
    M_{\tau}(g\cdot x) = \tau_{g\cdot x}^{-1} \cdot (g\cdot x) = (g\cdot \tau_x)^{-1} \cdot (g\cdot x)
      = \tau_x^{-1} \cdot g^{-1} \cdot g \cdot x = \tau_x^{-1} \cdot x = M_{\tau}(x) \;,
  \end{align*}
  so $M_{\tau}(x)$ is invariant. Now suppose $M_{\tau}(x_1) = M_{\tau}(x_2)$ for some $x_1, x_2$, and define $g = \tau_{x_1} \cdot \tau_{x_2}^{-1}$. Then
  \begin{align*}
    \tau_{x_1}^{-1}\cdot x_1 & = \tau_{x_2}^{-1} \cdot x_2 \\
    x_1 & = (\tau_{x_1} \cdot \tau_{x_2}^{-1}) \cdot x_2 = g\cdot x_2 \;.
  \end{align*}
  Therefore, $x_1$ and $x_2$ are in the same orbit and $M_{\tau}$ is a maximal invariant.

  For (ii), consider an arbitrary mapping $b : \calX \to \calY$, and define the function $f$ by $\tau_x\cdot b(\tau_x^{-1}\cdot x)$, as in \eqref{eq:tau:equiv}. Then
  \begin{align*}
    f(g\cdot x) & = \tau_{g\cdot x} \cdot b(\tau_{g\cdot x}^{-1} \cdot g \cdot x) = \tau_{g\cdot x} \cdot b(\tau_x^{-1} \cdot x) = \tau_{g\cdot x} \cdot \tau_{x}^{-1} \cdot f(x) \\
      & = g \cdot \tau_{x} \cdot \tau_{x}^{-1} \cdot f(x) = g\cdot f(x) \;,
  \end{align*}
  where the equivariance of $\tau_x$ is used repeatedly.
\end{proof}

We note that the equivariance requirement on $\tau$ can be relaxed, at the price of an extra condition required for $f$ defined in (ii) above to be equivariant. Specifically, any function $\bar{\tau}: \calX \to \grp$ that satisfies $\bar{\tau}^{-1}_{g\cdot x} \cdot g\cdot x = \bar{\tau}^{-1}_x \cdot x$ (i.e., part (i) above) can be used in the same way as $\tau$ to construct an equivariant $f$, with the condition that $\grp_x \subseteq \grp_{f(x)}$. \citet{Kallenberg:2005}, Lemma 7.10 proves such a result (and the existence of $\bar{\tau}$) in the case of a discrete group. 

\groupequivariantrep*

\begin{proof} 
  The proof of sufficiency closely follows the proof of Lemma 7.11 in \citet{Kallenberg:2005}. It relies on proving that $\tau^{-1}_X \cdot Y \condind_{M_{\tau}} X$, and applying \cref{lem:noise:out:suff,prop:tau:properties}. 

  Assume that $g \cdot (X,Y) \equdist (X,Y)$ for all $g\in\grp$. Let $M_{\tau}(x) = \tau^{-1}_x \cdot x$, and for any $x\in\calX$, let $M'_{\tau,x} : \calY \to \calY$ be defined by $M'_{\tau, x}(y) = \tau^{-1}_x \cdot y$, for $y\in\calY$. 
  As shown in the proof of \cref{prop:tau:properties}, $\tau_x \cdot \tau_{g\cdot x}^{-1} \cdot g = \tau_x^{-1}$ for all $x\in\calX$.  
  Therefore, the random elements $M_{\tau}(X) \in \calX$ and $M'_{\tau,X}(Y)\in\calY$ satisfy
  \begin{align} \label{eq:g:iden}
    (M_{\tau}(X), M'_{\tau,X}(Y)) = \tau_x^{-1} \cdot (X, Y) = \tau_{g\cdot X}^{-1} \cdot g \cdot (X,Y) \;, \quad \text{a.s., } g \in \grp \;.
  \end{align} 
  Now let $G\sim\normhaar$ be a random element of $\grp$ such that $G \condind (X,Y)$. Observe that because $G\sim\normhaar$, $g\cdot G \equdist G\cdot g \equdist G$ for all $g\in\grp$. Furthermore, by assumption $g \cdot (X,Y) \equdist (X,Y)$. Therefore, using Fubini's theorem,
  \begin{align} \label{eq:g:iden:two}
    G\cdot (X,Y) \equdist (X,Y) \;, \quad \text{and} \quad (G\cdot \tau^{-1}_X, X,Y) \equdist (G,X,Y) \;.
  \end{align}
  Using \eqref{eq:g:iden} and \eqref{eq:g:iden:two},
  \begin{align}
    (X, M_{\tau}(X), M'_{\tau,X}(Y)) & = (X, \tau_X^{-1} \cdot X, \tau_X^{-1} \cdot Y) \nonumber \\
      & \equdist (G\cdot X, \tau_{G\cdot X}^{-1} \cdot G \cdot X, \tau_{G\cdot X}^{-1} \cdot G \cdot Y) \nonumber \\
      & = (G\cdot X, \tau^{-1}_X \cdot X, \tau^{-1}_X \cdot Y) \nonumber \\
      & \equdist (G \cdot \tau^{-1}_X \cdot X , \tau^{-1}_X \cdot X, \tau^{-1}_X \cdot Y) \nonumber \\
      & = (G \cdot M_{\tau}(X), M_{\tau}(X), M'_{\tau,X}(Y)) \label{eq:tau:key:iden}
  \end{align}
  That is, jointly with the orbit representative $M_{\tau}(X)$ and the representative equivariant applied to $Y$, $M'_{\tau,X}(Y)$, the distribution of $X$ is the same as if applying a random group element $G\sim\normhaar$ to the orbit representative.

  Trivially, $M_{\tau}(X) \condind_{M_{\tau}(X)} M'_{\tau,X}(Y)$. Therefore, $G\cdot M_{\tau}(X) \condind_{M_{\tau}(X)} M'_{\tau,X}(Y)$ is implied by $G \condind (X,Y)$. This conditional independence is transferred by the distributional equality in \eqref{eq:tau:key:iden} to
  \begin{align*}
    X \condind_{M_{\tau}(X)} M'_{\tau,X}(Y) \;.
  \end{align*} 
  Therefore, by \cref{lem:noise:out:suff}, there exists some measurable $b : [0,1] \times \calX \to \calY$ such that $M'_{\tau,X}(Y) = \tau^{-1}_X \cdot Y \equas b(\noise, M_{\tau}(X))$ for $\noise \sim \Unif[0,1]$ and $\noise \condind X$. Applying $\tau_X$, \cref{prop:tau:properties} implies that 
  \begin{align*}
    Y \equas \tau_X \cdot b(\noise,M_{\tau}(X)) = \tau_X \cdot b(\noise, \tau^{-1}_X \cdot X) =: f(\noise,X)
  \end{align*}
  is $\grp$-equivariant in the second argument. This establishes sufficiency.

  Conversely, for necessity, using \eqref{eq:kall:equiv} and the assumption that $g \cdot X \equdist X$ for all $g\in\grp$,
  \begin{align*}
    g\cdot (X,Y) = (g\cdot X, f(\noise, g\cdot X)) 
    \equdist (X, f(\noise, X)) = (X,Y) \;.
  \end{align*}
\end{proof}

\subsection{Proof of \texorpdfstring{\cref{thm:suff:maximal:invariant}}{Theorem 10}}

\sufficientmaximalinvariant*

\begin{proof} 
  (i) Let $M : \calX \to \calS$ be any maximal invariant on $\calX$ under $\grp$. 
  From \cref{lem:maximal:invariant}, the conditional distribution of $X$ given $M(X) = m$ is equal to the orbit law $\orbitlaw_m$. $\calS$ is assumed to be Borel, so there is a Markov kernel \citep[e.g.,][Thm.\ 6.3]{Kallenberg:2002} $q_M : \borel_{\calX} \times \calS \to \bbR_+$ such that $q_M(\argdot,m) = \orbitlaw_m(\argdot)$, and therefore $M$ is a sufficient statistic for $\model_{X}^{\inv}$. Moreover, also by \cref{lem:maximal:invariant}, $M$ d-separates $X$ and $Y$ under any $P_{X,Y} \in \model_{X,Y}^{\inv}$, and therefore $M$ is also an adequate statistic. \Cref{lem:noise:out:suff} implies the identity \eqref{eq:max:invariant:rep}. 

  (ii) The proof follows a similar argument as in (i). \Cref{prop:tau:properties} established that $M_{\tau}(X) = \tau_X^{-1}\cdot X$ is a representative equivariant, so by \cref{lem:maximal:invariant}, $M_\tau$ is sufficient for $\model_{X}^{\inv}$, a subset of which is obtained by marginalizing over $Y$ in each distribution belonging to $\model_{X,\tau_X^{-1}\cdot Y}$. The proof of \cref{thm:kall:equiv} also established that $\tau_X^{-1}\cdot Y \condind_{M_{\tau}(X)} X$, which, along with sufficiency for the marginal model, is the definition of adequacy. 
\end{proof}

\section{Representations of Exchangeable \texorpdfstring{$d$}{d}-Dimensional Arrays}
\label{sec:proof:arrays}

The functional representations of conditionally $\bbS_{\bfn_2}$-invariant (\cref{thm:invariant:two:arrays}) and -equivariant (\cref{thm:equivariant:two:arrays}) distributions in \cref{sec:finite:exch:arrays} are special cases with $d=2$ of more general results for $d$-dimensional arrays.

Some notation is needed in order to state the results. For a fixed $d \in \bbN$, let $\Xnd$ be a $d$-dimensional $\calX$-valued array (called a $d$-array) with index set $[\bfn_d] := [n_1]\times \dotsb \times [n_d]$, and with $X_{i_1,\dotsc,i_d}$ the element of $\Xnd$ at position $(i_1,\dotsc,i_d)$. The size of each dimension is encoded in the vector of integers $\bfn_d = (n_1,\dotsc,n_d)$. In this section, we consider random $d$-arrays whose distribution is invariant to permutations applied to the index set. In particular let $\perm_k\in\bbS_{n_k}$ be a permutation of the set $[n_k]$. Denote by $\bbS_{\bfn_d}:=\bbS_{n_1}\times\dotsb\times\bbS_{n_d}$ the direct product of each group $\bbS_{n_k}$, $k\in [d]$.

A collection of permutations $\permd:=(\perm_1,\dotsc,\perm_d)\in\bbS_{\bfn_d}$ acts on $\Xnd$ in the natural way, separately on the corresponding input dimension of $\Xnd$:
\begin{align*} 
  [\permd \cdot \Xnd ]_{i_1,\dotsc,i_d} = X_{\perm_1(i_1),\dotsc,\perm_d(i_d)} \;.
\end{align*}
The distribution of $\Xnd$ is {separately exchangeable} if
\begin{align} \label{eq:sep:exch:d}
  (X_{i_1,\dotsc,i_d})_{(i_1,\dotsc,i_d)\in[n_{1,\dotsc,d}]} \equdist (X_{\perm_1(i_1),\dotsc,\perm_d(i_d)})_{(i_1,\dotsc,i_d)\in[n_{1,\dotsc,d}]} \;,
\end{align}
for every collection of permutations $\permd \in \bbS_{\bfn_d}$. 
We say that $\Xnd$ is separately exchangeable if its distribution is. 

For a symmetric array $\Xndsymm$, such that $n_1=n_2=\dotsb=n_d=n$ and $\Xsymm_{i_1,\dotsc,i_d}=\Xsymm_{i_{\permdim(1)},\dotsc,i_{\permdim(d)}}$ for all $\permdim\in\bbS_d$, the distribution of $\Xndsymm$ is \emph{jointly exchangeable} if, for all $\perm\in\bbS_n$
\begin{align} \label{eq:joint:exch:d}
  (\Xsymm_{i_1,\dotsc,i_d})_{(i_1,\dotsc,i_d)\in[n]^d} \equdist (\Xsymm_{\perm(i_1),\dotsc,\perm(i_d)})_{(i_1,\dotsc,i_d)\in[n]^d} \;.
\end{align}

As in \cref{sec:suff:rep:arrays}, a canonical form of $\Xnd$ is required. For a particular canonicalization routine, denote the space of $d$-dimensional canonical forms by $\cspace_{\bfn_d}(\calX)$.

The first result concerns $\bbS_{\bfn_d}$-invariant conditional distributions. An equivalent result holds for jointly exchangeable $d$-arrays and $\bbS_n$-invariant conditional distributions, which we omit for brevity.
\begin{theorem}
  \label{thm:invariant:d:arrays}

  Suppose $\Xnd$ is a separately exchangeable $\calX$-valued array on index set $\bfn_d$, and $Y\in\calY$ is another random variable. Then $P_{Y|\Xnd}$ is $\bbS_{\bfn_d}$-invariant if and only if there is a measurable function $f : [0,1] \times \cspace_{\bfn_d}(\calX) \to \calY$ such that
  \begin{align} \label{eq:invariant:d:arrays}
    (\Xnd,Y) \equas (\Xnd,f(\noise ,\cboard_{\Xnd})) \quad \text{where} \quad \noise\sim\Unif[0,1] \quad \text{and} \quad \noise \condind \Xnd \;.
  \end{align}
\end{theorem}
\begin{proof}
  Clearly, the representation \eqref{eq:invariant:d:arrays} satisfies $(\permd \cdot \Xnd,Y) \equdist (\Xnd,Y)$, for all $\permd \in \bbS_{\bfn_d}$, which by \cref{prop:joint:symmetry} and the assumption that $\Xnd$ is separately exchangeable implies that $Y$ is conditionally $\bbS_{\bfn_d}$-invariant given $\Xnd$. The converse is a easy consequence of the fact that the canonical form is a maximal invariant of $\bbS_{\bfn_d}$ acting on $\calX^{n_1\times\dotsb\times n_d}$ and \cref{thm:suff:maximal:invariant}. 
\end{proof}

\paragraph{$\bbS_{\bfn_d}$-equivariant conditional distributions.} 
Let $\Ynd$ be a $\calY$-valued array indexed by $[\bfn_d]$. \Cref{thm:equivariant:d:arrays} below states that each element $Y_{i_1,\dotsc,i_d}$ can be represented as a function of a uniform random variable $\noise_{i_1,\dotsc,i_d}\condind \Xnd$, and of a sequence of canonical forms: one for each sub-array of $\Xnd$ that contains $X_{i_1,\dots,i_d}$. As was the case for $d=2$, we assume
\begin{align} \label{eq:d:array:cond:ind}
  Y_{i_1,\dotsc,i_d} \condind_{\Xnd} (\Ynd \setminus Y_{i_1,\dotsc,i_d}) \;, \quad \text{for each } (i_1,\dotsc,i_d) \in [\bfn_d] \;.
\end{align}

For convenience, denote by $\bfi:=(i_1,\dotsc,i_d)$ an element of the index set $[\bfn_d]$. For each $k_1 \in [d]$, let $\subarr{\bfX}{\bfi(\{k_1\})}$ be the $(d-1)$-dimensional sub-array of $\Xnd$ obtained by fixing the $k_1$th element of $\bfi$ and letting all other indices vary over their entire range. For example, for an array with $d=3$, $\subarr{\bfX}{\bfi(\{ 1 \})}$ is the matrix extracted from $\Xnd$ by fixing $i_1$, $\bfX_{i_1,\; :,\; :}$.

Iterating, let $\subarr{\bfX}{\bfi(\{k_1,\dotsc,k_p\})}$ be the $(d-p)$-dimensional sub-array of $\Xnd$ obtained by fixing elements $i_{k_1},\dotsc,i_{k_p}$ of $\bfi$ and letting all other indices vary. 
For each $p\in [d]$, denote by $[d]^p$ the collections of subsets of $[d]$ with exactly $p$ elements; let $[d]^{(p)}$ be the collection of subsets of $[d]$ with $p$ or fewer elements. Let the collection of $(d-p)$-dimensional sub-arrays containing $\bfi$ be denoted as
\begin{align*}
  \subarrcoll{\bfX}{d-p}{\bfi} = \big( \subarr{\bfX}{\bfi(s)} \big)_{s\in [d]^{p}}  \;.
\end{align*}

A $d$-dimensional version of the augmented canonical form \eqref{eq:sep:aug:cboard} is needed. To that end, let $\bfj = (j_1,\dotsc,j_d)\in [\bfn_d]$ and define the index sequence $(\bfi(s),\bfj)\in [\bfn_d]$ as
\begin{align} \label{eq:ij:index}
  (\bfi(s),\bfj)_k = 
  \begin{cases}
    i_k & \text{if } k \in s \\
    j_k & \text{if } k \notin s
  \end{cases}
  \;, \quad \text{for } s \in [d]^{(p)} \text{ and } k \in [d] \;.
\end{align}
Then we define the $d$-dimensional $p$-augmented canonical form,
\begin{align} \label{eq:aug:cboard:d}
  [\cboard^{(p)}_{\bfi,\Xnd}]_{\bfj}
  = 
  \big([\cboard_{\Xnd}]_{\bfj}, \big( [\cboard_{\Xnd}]_{(\bfi(s),\bfj)}  \big)_{s\in[d]^{(p)}}  \big) \;.
\end{align}
Denote by $\cspace_{\bfn_d}^{(d,p)}$ the space of all such augmented canonical forms.

The function returns the collection of elements from $\cboard_{\Xnd}$ that correspond to $\bfj$ in each of the $(d-q)$-dimensional sub-arrays containing $\bfi$, $\subarrcoll{\bfX}{d-q}{\bfi}$, for $q=0,\dotsc,p$. Alternatively, observe that the function can be constructed by recursing through the $(d-q)$-dimensional $1$-augmented canonical forms (there are $\binom{d}{q}$ of them for each $q$), for $q=0,\dotsc,p-1$, and returning the first argument of each. This recursive structure captures the action of $\bbS_{\bfn_d}$ on $\Xnd$, and is at the heart of the following result, which gives a functional representation of $\bbS_{\bfn_d}$-equivariant conditional distributions.

\begin{theorem}
  \label{thm:equivariant:d:arrays}

  Suppose $\Xnd$ and $\Ynd$ are $\calX$-valued arrays indexed by $[\bfn_d]$, and that $\Xnd$ is separately exchangeable. Assume that the elements of $\Ynd$ are conditionally independent given $\Xnd$. Then $P_{\Ynd | \Xnd}$ is $\bbS_{\bfn_d}$-equivariant if and only if 
  \begin{align} \label{eq:rep:exch:d:array}
    \big(\Xnd,\Ynd  \big)
    \equas 
    \bigg(\Xnd,
      \big(
        f(
          \noise_{\bfi}, 
          X_{\bfi}, 
          \cboard^{(d)}_{\bfi,\Xnd}
        )_{\bfi\in [\bfn_d]}
      \big)
    \bigg) \;,
  \end{align}
  for some measurable function $f : [0,1] \times \calX \times \cspace_{\bfn_d}^{(d,d)} \to \calY$ and \iid uniform random variables $(\noise_{\bfi})_{\bfi \in [\bfn_d]} \condind \Xnd$.
\end{theorem}

\begin{proof}
First, assume that $\Ynd$ is conditionally $\bbS_{\bfn_d}$-equivariant given $\Xnd$. By assumption, $\Xnd$ is separately exchangeable, so  by \cref{prop:joint:symmetry}, $\permd \cdot (\Xnd,\Ynd) \equdist (\Xnd,\Ynd)$ for all $\permd\in\bbS_{\bfn_d}$. Fix $\bfi \in [\bfn_d]$, and let $\bbS_{\bfn_d}^{(\bfi)} \subset \bbS_{\bfn_d}$ be the stabilizer of $\bfi$. Observe that each $\permd^{(\bfi)} \in \bbS_{\bfn_d}^{(\bfi)}$ fixes $X_{\bfi}$ and $Y_{\bfi}$. 

In analogy to the fixed -row and -column structure for the $d=2$ case, any $\permd^{(\bfi)}\in\bbS_{\bfn_d}^{\bfi}$ results in sub-arrays of $\Xnd$ in which the elements my be rearranged, but the sub-array maintains its position within $\Xnd$. For example, consider $d=3$. Each of the two-dimensional arrays $\bfX_{i_1,:,:}$, $\bfX_{:,i_2,:}$, and $\bfX_{:,:,i_3}$ may have their elements rearranged, but they will remain the two-dimensional sub-arrays that intersect at $\bfi$. Likewise for the one-dimensional sub-arrays $\bfX_{i_1,i_2,:}$, $\bfX_{i_1,:,i_3}$, and $\bfX_{:,i_2,i_3}$.

Recall that $\subarr{\bfX}{\bfi(\{k_1,\dotsc,k_p\})}$ is the $(d-p)$-dimensional sub-array of $\Xnd$ obtained by fixing elements $i_{k_1},\dotsc,i_{k_p}$ of $\bfi$ and letting the other indices vary, and $\subarrcoll{\bfX}{d-p}{\bfi}$ is the collection of these sub-arrays. 
For $p=1$, $\subarrcoll{\bfX}{d-1}{\bfi} = (\subarr{\bfX}{\bfi(\{ k_1 \})})_{k_1 \in [d]}$ is the collection of $(d-1)$-dimensional sub-arrays containing the element $X_{\bfi}$. Denote by $\superarr{\bfX}{d}{\bfi} := \Xnd \setminus \subarrcoll{\bfX}{d-1}{\bfi}$ the $d$-dimensional array that remains after extracting $\subarrcoll{\bfX}{d-1}{\bfi}$ from $\Xnd$. Call $\superarr{\bfX}{d}{\bfi}$ a $d$-dimensional \emph{remainder array}. For example, with $d=3$, the remainder array $\superarr{\bfX}{d}{\bfi}$ consists of $\Xnd$ with the three two-dimensional arrays (matrices) that contain $X_{\bfi}$ removed.

Continue recursively and construct a remainder array $\superarr{\bfX}{d-p}{\bfi(s)}$, $s\in [d]^p$, from each sub-array in the collection $\subarrcoll{\bfX}{d-p}{\bfi} = (\subarr{\bfX}{\bfi(s)})_{s\in [d]^p}$. 
The recursive structure of the collection of remainder arrays $\big( (\superarr{\bfX}{d-p}{\bfi(s)})_{s\in [d]^p} \big)_{0 \leq p \leq d-1}$ is captured by the array $\bfZ^{(i)}$, with entries
\begin{align}
  [\bfZ^{(\bfi)}]_{\bfj} = \big(X_{\bfj} , ( X_{(\bfi(s),\bfj)} )_{s\in [d]^{(d)}} \big) \;,
\end{align}
where $(\bfi(s),\bfj)$ is as in \eqref{eq:ij:index}. 
Define the action of a collection of permutations $\permd' \in \bbS_{\bfn_d - \bfOne_d}$ on $\bfZ^{(\bfi)}$ to be such that, with $\bfj_{\perm} = (\perm_1(j_1),\dotsc,\perm_d(j_d))$,
\begin{align}
  [\permd' \cdot \bfZ^{(\bfi)}]_{\bfj} = \big(X_{\bfj_{\perm}} , ( X_{(\bfi(s),\bfj_{\perm})} )_{s\in [d]^{(d)}} \big) \;.
\end{align}
$\bfZ^{(\bfi)}$ inherits the exchangeability of $\Xnd$, so that marginally for $Y_{\bfi}$, 
\begin{align*}
  ( 
    \permd \cdot \bfZ^{(\bfi)} ,
    (X_{\bfi}, Y_{\bfi}) 
  )
  \equdist
  (
    \bfZ^{(\bfi)} ,
    (X_{\bfi}, Y_{\bfi}) 
  )
  \quad \text{for all } \permd \in \bbS_{\bfn_d - \bfOne_d} \;.
\end{align*}
which implies $(X_{\bfi}, Y_{\bfi})\condind_{\cboard_{\bfZ^{(\bfi)}}} \bfZ^{(i)}$. Conditioning on $X_{\bfi}$ and $\cboard_{\bfZ^{(\bfi)}}$ is the same as conditioning on $X_{\bfi}$ and the $d$-dimensional $d$-augmented canonical form $\cboard^{(d)}_{\bfi,\Xnd}$ defined in \eqref{eq:aug:cboard:d}, implying that
\begin{align*}
  Y_i \condind_{ (X_{\bfi}, \cboard^{(d)}_{\bfi,\Xnd}) } \Xnd \;.
\end{align*}
By \cref{lem:noise:out:suff}, there is a measurable function $f_{\bfi} : [0,1] \times \calX \times \cspace_{\bfn_d}^{(d)} \to \calY$ such that
\begin{align*}
  Y_{\bfi} = f_{\bfi}(\noise_{\bfi}, X_{\bfi}, \cboard^{(d)}_{\bfi,\Xnd} ) \;,
\end{align*}
for a uniform random variable $\noise_{\bfi} \condind \Xnd$. This is true for all $\bfi \in [\bfn_d]$; by equivariance the same $f_{\bfi}=f$ must work for every $\bfi$. Furthermore, by assumption the elements of $\Ynd$ are mutually conditionally independent given $\Xnd$, and therefore by the chain rule for conditional independence \citep[][Prop.\ 6.8]{Kallenberg:2002}, the joint identity \eqref{eq:rep:exch:d:array} holds.

The converse is straightforward to verify.

\end{proof}

\end{appendix}

\bibliography{refs}

\end{document}